\documentclass{article} 
\usepackage[final]{colm2026_conference}

\usepackage{microtype}
\usepackage{hyperref}
\usepackage{url}
\usepackage{booktabs}
\usepackage{subcaption}
\usepackage{tabularx}
\usepackage{array}
\usepackage{multirow} 

\usepackage{lineno}

\definecolor{darkblue}{rgb}{0, 0, 0.5}
\hypersetup{colorlinks=true, citecolor=darkblue, linkcolor=darkblue, urlcolor=darkblue}

\usepackage{graphicx} 

\usepackage{amsmath}

\DeclareMathOperator*{\argmax}{arg\,max}

\usepackage{algorithm}
\usepackage{algpseudocode}

\usepackage{colortbl}
\definecolor{bestcell}{RGB}{226,241,230}
\definecolor{modelrow}{RGB}{244,246,248}

\usepackage{multirow}

\definecolor{realcell}{RGB}{226,237,248}
\definecolor{shufcell}{RGB}{252,235,218}

\newcommand{\realv}[1]{\cellcolor{realcell}$#1$}
\newcommand{\shufv}[1]{\cellcolor{shufcell}$#1$}

\definecolor{basecell}{RGB}{238,242,247}
\definecolor{reasoncell}{RGB}{226,241,230}

\newcommand{\basev}[1]{\cellcolor{basecell}$#1$}
\newcommand{\reasonv}[1]{\cellcolor{reasoncell}\textbf{$#1$}}

\definecolor{adjcell}{RGB}{226,237,248}
\definecolor{rndcell}{RGB}{252,235,218}

\newcommand{\adjv}[1]{\cellcolor{adjcell}$#1$}
\newcommand{\rndv}[1]{\cellcolor{rndcell}$#1$}
\newcommand{\rndbest}[1]{\cellcolor{rndcell}\textbf{$#1$}}

\definecolor{prefixcell}{RGB}{226,241,230}
\definecolor{oraclecell}{RGB}{238,242,247}
\definecolor{negativecell}{RGB}{250,226,226}

\usepackage[table]{xcolor}

\definecolor{bestgreen}{RGB}{226,242,226}
\definecolor{uncertgray}{RGB}{115,115,115}

\newcommand{\val}[2]{%
    \ensuremath{#1\,{\color{uncertgray}{\scriptstyle #2}}}%
}
\newcommand{\best}[2]{%
    \cellcolor{bestgreen}%
    \ensuremath{\mathbf{#1}\,{\color{uncertgray}{\scriptstyle #2}}}%
}

\usepackage{booktabs}

\usepackage{amsthm}
\usepackage{amssymb}
\usepackage{enumitem}
\usepackage{url}

\newtheorem{theorem}{Theorem}[section]
\newtheorem{proposition}[theorem]{Proposition}
\newtheorem{lemma}[theorem]{Lemma}
\newtheorem{corollary}[theorem]{Corollary}
\newtheorem{assumption}{Assumption}
\theoremstyle{definition}

\theoremstyle{remark}
\newtheorem{remark}[theorem]{Remark}

\usepackage{booktabs}
\usepackage{CJKutf8}

\usepackage{wrapfig}
\newcommand{\TODO}[1]{$ $\newline\noindent\colorbox{yellow!30}{\parbox{\dimexpr\the\columnwidth-2\fboxsep}{\textbf{\texttt{TODO:}} \textit{#1}}}}

\title{Reasoning Fine-Tuning Induces Persistent Latent Policy States}
\author{
Abir Harrasse$^{1*}$, Michael Lan$^{1*}$, Hunar Batra$^{2}$, \\\textbf{Fateme Hashemi Chaleshtori}$^{3}$, \textbf{Chaithanya Bandi}$^{1}$ \\
$^{1}$Martian, $^{2}$University of Oxford, $^{3}$University of Utah \\
$^{*}$Primary contributors
}
\begin{document}

\ifcolmsubmission
\linenumbers
\fi

\maketitle

\begin{abstract}
Reasoning-specialized language models show large performance gains over
base models, yet the internal changes responsible for improved multi-step
reasoning remain poorly understood. It is unclear whether reasoning
fine-tuning improves local token-level competence or globally reorganizes
how models structure inference over time. We address this question by
modeling Chain-of-Thought reasoning as a switching dynamical system
(SDS), in which internal representations evolve under discrete latent
policy states. Our framework combines time-aware contrastive
representation learning with discrete regime discovery to recover latent
policies from activation trajectories. Across four benchmarks and model
scales from 1.5B to 32B parameters, reasoning-fine-tuned models exhibit
richer latent-policy organization than their base counterparts,
characterized by more differentiated transition structure and
model-dependent changes in state utilization, persistence, and mixing.
The recovered regimes exhibit functional specialization aligned with
distinct reasoning stages, and extensive controls confirm that their
structure is not explained by correctness, representation learning, or
modeling priors, but depends on the coherent temporal organization of
reasoning trajectories. Causal interventions further show that the regimes
are functionally meaningful: state-swap ablations reduce one-step
predictive fit, while transplanting reasoning dynamics into base models
improves performance on challenging reasoning problems. Finally,
SDS-guided pruning of failure-prone reasoning prefixes outperforms
self-consistency in 11 of 12 model--dataset settings, with gains of up to
12.5 percentage points. Together, our results suggest that reasoning
fine-tuning globally reorganizes latent dynamics, offering a new lens for
mechanistic analysis and process-level control of reasoning models.
Code: \url{https://github.com/withmartian/mi-cot}.
\end{abstract}

\section{Introduction}

Large language models exhibit striking improvements in Chain-of-Thought
(CoT) reasoning after specialized fine-tuning
\citep{deepseekai2025deepseekr1incentivizingreasoningcapability,
shao2024deepseekmathpushinglimitsmathematical,huang2026on}.
However, the internal changes responsible for these improvements remain
poorly understood
\citep{yue2025doesreinforcementlearningreally,wu2025on,
ward2025reasoningfinetuningrepurposeslatentrepresentations}.
In particular, it remains unclear whether reasoning fine-tuning primarily
improves local token-level competence or globally reorganizes the latent
computational dynamics that structure reasoning over extended horizons.
Establishing such structure is challenging because CoT reasoning is a
branching temporal process involving stochastic sampling and a distribution
of possible trajectories
\citep{macar2025thoughtbranchesinterpretingllm,
lee2025cotencyclopediaanalyzingpredicting}.

CoTs have therefore been modeled as dynamical systems that capture
long-range transitions between branching reasoning states
\citep{carson2025a}. This perspective suggests that reasoning behavior may
be governed by low-dimensional latent policy states that shape future
internal dynamics throughout inference.

\begin{figure}[t]
    \centering
    \begin{minipage}[t]{0.58\textwidth}
        \centering
        \includegraphics[width=\textwidth]{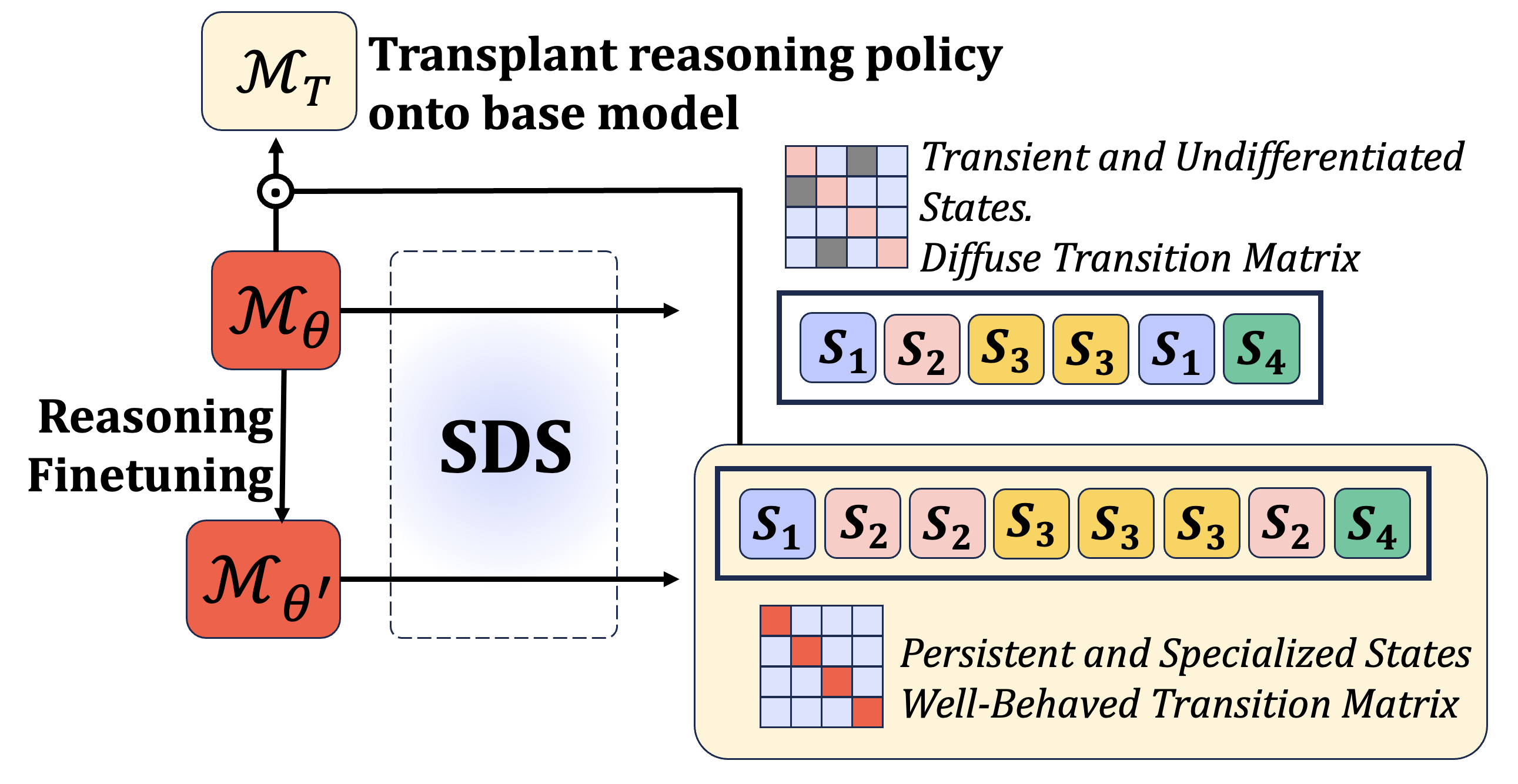}
    \end{minipage}
    \hspace{0.01\textwidth}%
    \begin{minipage}[t]{0.40\textwidth}
        \centering
        \includegraphics[width=\textwidth]{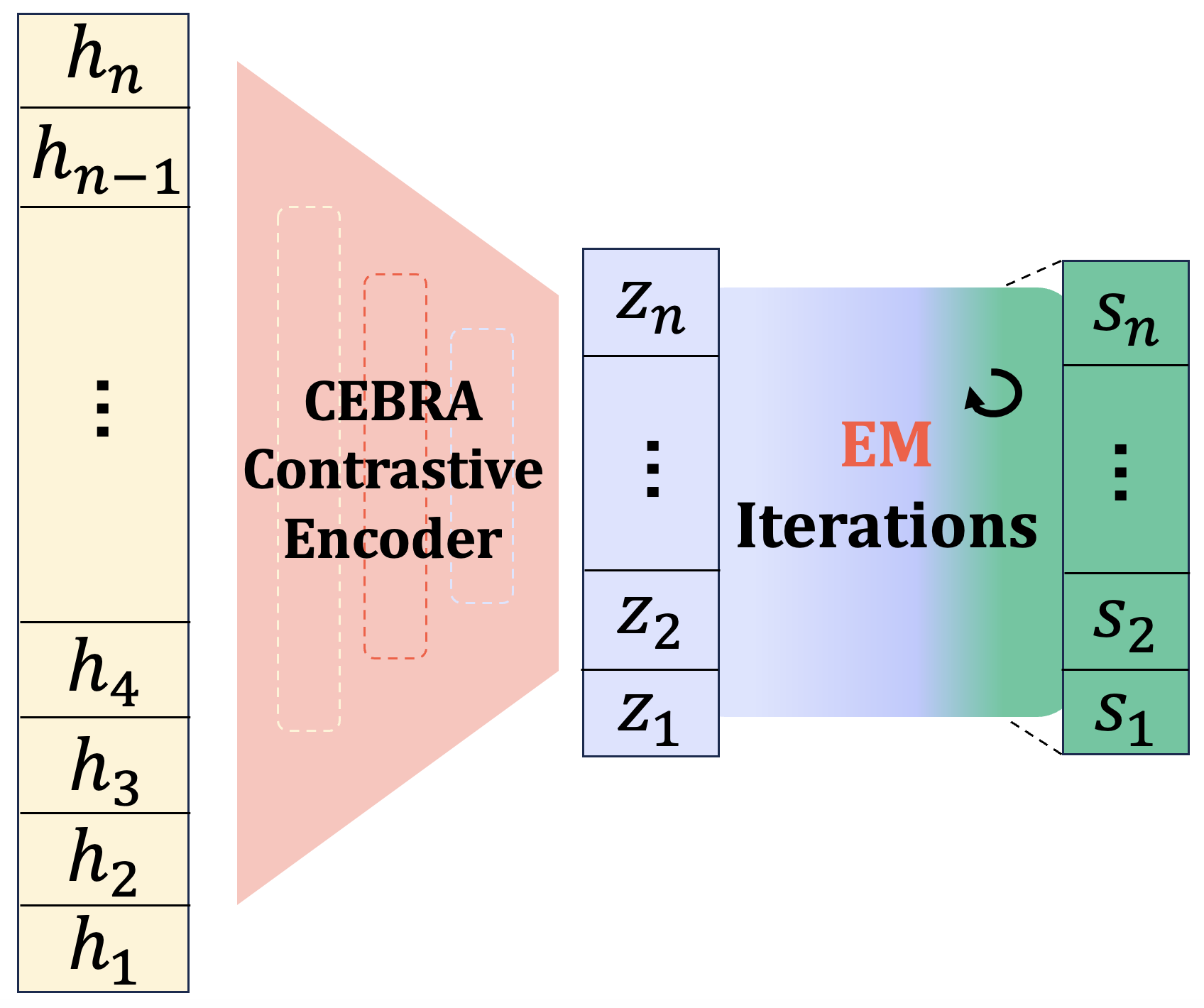}
    \end{minipage}
    \caption{
        \textbf{Left:} Overview of our framework. A base model
        $\mathcal{M}_\theta$ is fine-tuned for reasoning to produce
        $\mathcal{M}_{\theta'}$. We fit a Switching Dynamical System (SDS)
        to each model's activation trajectories and recover sequences of
        latent regimes. Reasoning fine-tuning reorganizes state utilization,
        persistence, and mixing while producing more differentiated
        transition dynamics. The recovered reasoning policy can be
        transplanted onto the base model $\mathcal{M}_T$ to test its causal
        and functional relevance.
        \textbf{Right:} Model activations $h_{1:n}$ are encoded into a
        low-dimensional representation $z_{1:n}$ using a CEBRA contrastive
        encoder. EM with forward--backward smoothing recovers the discrete
        latent regime sequence $s_{1:n}$, where
        $s_t\in\{1,\ldots,K\}$ represents the active reasoning-policy state
        at step $t$.
    }
    \label{fig:main}
\end{figure}

In this work, we compare fine-tuned reasoning models with their base
counterparts by modeling activation trajectories using \emph{Switching
Dynamical Systems} (SDS), which recover discrete, low-dimensional latent
policy states
\citep{1099359,ghahramani2000variational,pmlr-v54-linderman17a}.
We combine SDS with time-aware contrastive manifold learning
\citep{Schneider_2023} and find that reasoning fine-tuning induces richer
latent-policy organization, most consistently reflected in more
differentiated transition structure. State utilization, persistence, and
mixing also change, with their contributions varying across models and
datasets. The recovered states exhibit functional specialization aligned
with reasoning stages, while predictivity degrades when the state space
becomes over-fragmented.

Overall, reasoning fine-tuning appears to globally reorganize internal
dynamics, enabling models to maintain and transition between coherent
computational modes over extended spans. We call this temporal and
functional structure a \emph{latent reasoning policy}, without assuming
that it is localized to a particular model component.

Our contributions are:
\begin{enumerate}
    \item We formalize a dynamical-systems framework for uncovering
    differences in internal reasoning structure between base and
    reasoning-fine-tuned models.

    \item We provide empirical evidence that reasoning fine-tuning induces
    richer latent-policy organization, characterized by more differentiated
    transition structure and model-dependent changes in state utilization,
    persistence, and mixing.

    \item We introduce CEBRA for latent reasoning-state analysis and show
    that, in our setting, it provides a stronger balance of predictive fit
    and structured regime recovery than PCA.

    \item We show that the inferred latent reasoning policy is actionable:
    causal state-swap and policy-transplantation interventions establish
    the functional relevance and transferability of the recovered dynamics,
    while SDS-guided pruning of failure-prone reasoning prefixes
    (\textsc{PrefixGuard}) improves over self-consistency in 11 of 12
    model--dataset settings, with gains of up to 12.5 percentage points.
\end{enumerate}
\section{Related Work}

\textbf{Comparing Base and Reasoning-Fine-Tuned Models.}
Reasoning fine-tuning typically combines supervised fine-tuning,
distillation, and reinforcement learning with verifiable rewards (RLVR)
\citep{lambert2025tulu3pushingfrontiers,
deepseekai2025deepseekr1incentivizingreasoningcapability}.
A growing line of work argues that RLVR improves reasoning by reallocating
probability mass toward successful trajectories already supported by the
base model, rather than by creating entirely new reasoning procedures
\citep{yue2025doesreinforcementlearningreally,wu2025on,
nguyen2025reasoningboundaryparadoxreinforcement,zhu2025path}.
While RL-trained models typically outperform their base counterparts at
small sampling budgets, base models can sometimes match or exceed them at
large pass@$k$, suggesting that RLVR improves trajectory ranking and
sampling efficiency while narrowing exploration
\citep{yue2025doesreinforcementlearningreally,
nguyen2025reasoningboundaryparadoxreinforcement,
ward2025reasoningfinetuningrepurposeslatentrepresentations}.
\citet{wu2025on} characterize this effect through support and entropy,
while \citet{wang20258020rulehighentropyminority} show that RLVR acts
disproportionately on high-entropy tokens corresponding to critical forks
in reasoning trajectories \citep{huang2026on}. These findings concern how
post-training changes which reasoning chains are selected; we instead ask
how it changes the internal dynamics through which those chains unfold.

\textbf{Mechanistic Interpretability and Control of Reasoning.}
Previous mechanistic interpretability work suggests that reasoning
fine-tuning repurposes latent representations already present in base
models rather than constructing wholly new machinery
\citep{ward2025reasoningfinetuningrepurposeslatentrepresentations,
minder2026narrow,ward2025rank,muhamed2025towards,wang2025towards}.
\citet{venhoff2025basemodelsknowreason} further argue that base models
already contain many relevant reasoning behaviors and that post-training
primarily teaches the model when to deploy them. At a finer granularity,
thinking tokens and thought anchors identify decision points that
disproportionately shape subsequent reasoning
\citep{qian2025demystifying,bogdan2025thoughtanchorsllmreasoning}, while
steering-vector methods localize and manipulate individual reasoning
behaviors \citep{venhoff2025understanding}. Hidden-state clustering also
provides evidence that reasoning unfolds through discrete internal stages
\citep{liang2025clue}.

Related methods use such trajectory structure for inference-time
improvement. CTRLS models CoT reasoning as latent state transitions and
uses distributional reinforcement learning for state-aware exploration
\citep{wu2026ctrls}, while SEAL separates execution, reflection, and
transition thoughts in latent space and steers representations toward more
productive reasoning \citep{seal}. Unlike methods that model, rank, or
control trajectories within one model, we ask how reasoning fine-tuning
reorganizes latent dynamics relative to its base model. Our SDS formulation
addresses this through state utilization, persistence, transition
organization, and predictive dynamics, followed by causal policy
transplantation and process-level pruning of failure-prone reasoning
trajectories.

\textbf{Dynamical and Generative Models of Internal Representations.}
Our approach builds on theoretical work establishing conditions under
which switching dynamical systems and related Markov-switching models can
be recovered from observational trajectories
\citep{1099359,ghahramani2000variational,pmlr-v54-linderman17a}.
\citet{Balsells2024} provide conditions under which SDS parameters and
discrete regimes are identifiable from data. Complementarily,
\citet{luo2026learninggenerativemetamodelllm} train diffusion models on
residual-stream activations to capture the distribution of a network's
internal states and provide priors for more faithful interventions. We
combine an SDS with contrastive representation learning
\citep{Schneider_2023}, using CEBRA to obtain embeddings that empirically
support stronger regime recovery than PCA in our setting.

Concurrently, \citet{carson2025a} show that sentence-level hidden states
can be modeled using a low-dimensional switching linear dynamical system,
where regimes capture distinct drift and transition patterns in reasoning
traces. Whereas their PCA-based state space studies regimes within
individual models, we use CEBRA-derived embeddings to compare matched base
and reasoning-fine-tuned models and test whether the recovered differences
are causally and practically actionable.
\section{Reasoning as a Switching Dynamical System}


\subsection{Problem Setup}
We study whether reasoning in LLMs can be described by persistent, low-dimensional 
latent policy states that govern future internal activation dynamics. We operationalize 
this as a predictive latent-variable problem: we hypothesize a discrete latent \textbf{state} 
(i.e., \textbf{regime}) $s_t \in \{1, \dots, K\}$ that (i) persists over time, 
(ii) governs the evolution of internal representations, and (iii) is identifiable from 
observed activation trajectories. Let $h_t \in \mathbb{R}^d$ denote the residual stream 
activation at reasoning step $t$, where each step corresponds to a deterministically 
segmented sentence in a model-generated chain-of-thought trace. We obtain traces by 
sampling complete model responses to each reasoning prompt, then split each trace into 
sentence-level units using punctuation-based boundaries; $h_t$ is extracted at the final 
token of the $t$-th sentence. The choice of CEBRA over PCA for dimensionality reduction 
and EM over a Mixture-of-Experts formulation for parameter estimation are both justified 
empirically in Appendix~\ref{app:em_moe_cebra_ablation}.

\subsection{Switching Dynamical System Formulation}
\label{subsec: formal_assumption}
We model reasoning as a latent Markov process in which the discrete policy state $s_t$ 
governs the local evolution of the continuous internal representation $z_t$:
\begin{align}
P(s_{t+1}=j \mid s_t=i) &= T_{ij},
\label{eq:transition}\\
P(z_{t+1} \mid z_t, s_t=i) &= p_i(z_{t+1}\mid z_t).
\label{eq:emission}
\end{align}
In the linear-Gaussian instantiation used for inference, the dynamics take the form:
\begin{equation}
z_{t+1} = A_{s_t} z_t + b_{s_t} + \varepsilon_t, \quad 
\varepsilon_t \mid s_t \sim \mathcal{N}(0, \Sigma_{s_t}),
\label{eq:sds}
\end{equation}
with $\{s_t\}$ an irreducible, aperiodic Markov chain with transition matrix $T^\star$. 
We further assume \phantomsection\label{ass:main_assumptions} nontrivial self-transition 
$p_{\mathrm{stay}} > 1/K^\star$ (persistent regimes) and distinct regime parameters for 
$i \neq j$ (full theory and formal assumptions appear in Appendix~\ref{sec:theory}).

\subsection{Dimensionality Reduction via CEBRA}
Raw activations $h_t$ are projected into observed $d$-dimensional embedding trajectories 
$z_t$ using a CEBRA-style contrastive encoder \citep{Schneider_2023}. Positive pairs are 
temporally adjacent steps from the same problem; negatives are temporally distant steps 
from the same problem or steps from different problems. The encoder is trained with an 
InfoNCE objective:
\begin{equation}
\mathcal{L}_{\text{NCE}} = -\log \frac{\exp(\tilde{z}_t^\top \tilde{z}_{t^+}/\tau)}
{\exp(\tilde{z}_t^\top \tilde{z}_{t^+}/\tau) + \sum_{t^-}\exp(\tilde{z}_t^\top 
\tilde{z}_{t^-}/\tau)},
\end{equation}
where embeddings are $\ell_2$-normalized and $\tau$ is a temperature parameter. This 
encourages local temporal coherence in the embedding space, ensuring that discovered 
regimes reflect functional reasoning transitions rather than superficial token similarity.

\subsection{SDS Parameter Estimation via EM}
We estimate the SDS parameters $\theta = \{\pi, T, \{A_k, b_k, \Sigma_k\}_{k=1}^K\}$ 
from the observed CEBRA embedding trajectories $z_{1:T}$ via Expectation-Maximization 
\citep{ghahramani2000variational}. At each iteration, the E-step applies the 
forward-backward algorithm to compute soft regime posteriors 
$\gamma_t(k) = P(s_t = k \mid z_{1:T})$, and the M-step updates the regime-conditional 
dynamics via weighted least squares and the transition matrix via soft transition counts. 
The discrete regime sequence is then recovered as 
$\hat{s}_t = \mathop{\mathrm{arg\,max}}_k \gamma_t(k)$, and the model order $K^\star$ 
is selected by minimizing BIC on the fitted SDS likelihood. Full derivations and update 
equations are provided in Appendix~\ref{sec:em_details}, with identifiability analysis 
provided in Appendix~\ref{app:identifiability_full}.

\subsection{Why CEBRA Outperforms PCA}
\label{sec:theory_cebra_pca}

Under mild assumptions on the SDS dynamics, the prediction risk decomposes into an 
irreducible noise term and a misclassification term controlled by the minimum pairwise 
Bhattacharyya distance $B(f)$ between regime-conditional distributions. This gives a 
closed-form criterion for when CEBRA outperforms PCA. The full derivation appears in 
Appendix~\ref{sec:theory}:

\begin{corollary}[CEBRA vs.\ PCA]
\label{cor:cebra_pca}
$\mathcal{R}_{\mathrm{SDS}}(f_C) < \mathcal{R}_{\mathrm{SDS}}(f_P)$ if and only if
$B(f_C) - B(f_P) > \log\frac{M(f_C)}{M(f_P)}.$
\end{corollary}

We verify this condition empirically across all model--dataset pairs 
(Appendix~\ref{sec:cebra_pca_empirical}). Reasoning models satisfy it substantially more 
often than base models, indicating that the representation advantage of CEBRA is 
strongest in the models where structured latent dynamics have emerged.

\subsection{Validating the Markovian Assumption}
\label{sec:markovianity}

\begin{wrapfigure}{r}{0.46\columnwidth}
    \vspace{-0.8\baselineskip}
    \centering
    \includegraphics[width=0.44\columnwidth]{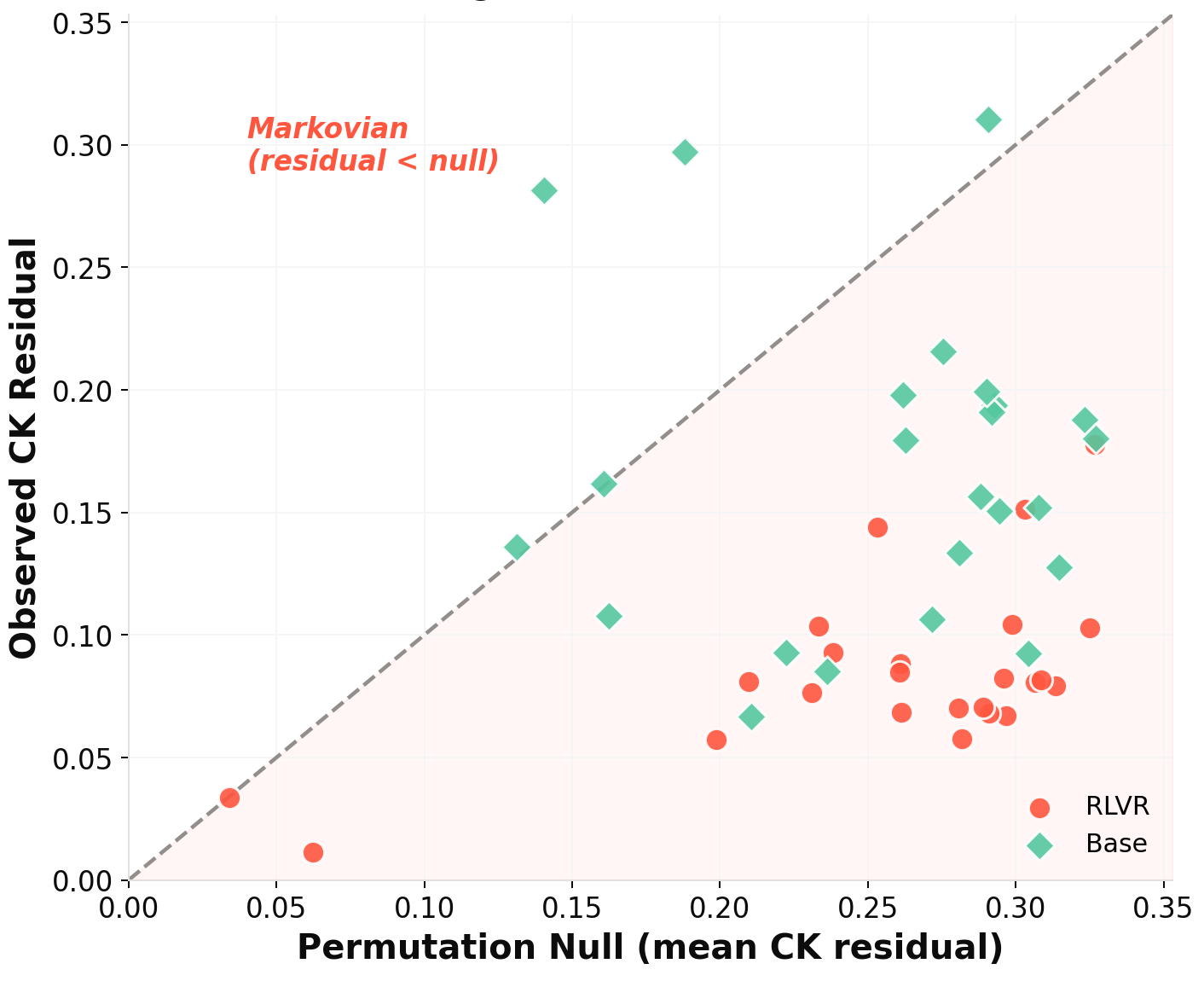}
    \caption{CK residuals against the permutation null. Points below the diagonal indicate that the decoded regime sequence satisfies the Chapman--Kolmogorov relation more closely than a memoryless baseline.}
    \label{fig:markovianity}
    \vspace{-0.8\baselineskip}
\end{wrapfigure}

A core modeling assumption is that the decoded regime sequence $\{s_t\}$ follows a
first-order Markov chain. We verify this using two complementary tests. The first is
a Chapman--Kolmogorov residual test: for a first-order chain, the two-step transition
matrix $T^{(2)}_{ij} = P(s_{t+2}=j \mid s_t=i)$ must satisfy $T^{(2)} \approx T^2$,
where $T$ is the fitted one-step transition matrix. We measure the normalized Frobenius
residual $\mathrm{CK} = \|T^2 - \hat{T}^{(2)}\|_F / \|\hat{T}^{(2)}\|_F$ and compare
it against a permutation null that destroys temporal dependencies while preserving
marginal state frequencies. The second is an order-2 BIC test comparing first- and
second-order Markov models on the decoded sequences, with positive $\Delta\mathrm{BIC}$
indicating preference for the simpler first-order model.

Reasoning models concentrate well below the diagonal in Figure~\ref{fig:markovianity},
indicating stronger first-order temporal consistency than permuted controls. Base models
are more diffuse, consistent with weaker temporal organization. The order-2 BIC test
corroborates this: first-order models are preferred in the large majority of reasoning-model runs, supporting the use of a first-order switching process as an adequate description of the
recovered latent policies.

\subsection{Why Reasoning Fine-Tuning Induces Persistent Latent States}
\label{sec:theory_rlvr_main}

One possible explanation for the emergence of persistent regimes in reasoning models is that 
trajectory-level rewards  under RLVR may favor coherent multi-step strategies, implicitly creating 
a positive continuation margin $\gamma_i(x_t) := V_i(x_t) - \max_{j \neq i} V_j(x_t)$ 
over large regions of state space and biasing the policy toward self-transition. 
Pretraining via maximum likelihood exerts no such pressure, as next-token prediction 
does not reward consistency of latent modes across time. Under this view, RLVR does 
not create latent structure from scratch but rather amplifies and stabilizes structure 
already present in the base model by making persistence directly reward-relevant. 
We formalize this argument in Appendix~\ref{sec:theory_rlvr} 
(Proposition~\ref{prop:rlvr}). Empirical support for this progressive emergence is provided by the training dynamics analysis in Appendix~\ref{sec:training_dynamics}.
\section{Experimental Setup}

We evaluate on GSM8K (7473 samples) \citep{cobbe2021gsm8k}, MATH-500 (500 samples) \citep{hendrycks2021math500}, 
SVAMP (800 samples) \citep{patelsvamp}, and MMLU-Pro (4000 samples) \citep{wang2024mmluprorobustchallengingmultitask} 
using chain-of-thought traces, and study paired base and DeepSeek-R1 distilled reasoning models \citep{deepseekai2025deepseekr1incentivizingreasoningcapability} for: Llama-3.1-8B \citep{llama3modelcard}, Qwen2.5-14B \citep{qwen}, 
and Qwen2.5-Math-1.5B \citep{qwen2, qwen2.5}. To examine whether these patterns extend to a larger scale, we additionally
compare Qwen2.5-32B with its reasoning-specialized counterpart QwQ-32B
using the same evaluation protocol; further details are provided in
Appendix~\ref{sec:more_mods}. Activations are extracted at the sentence level 
using the last-token representation of each sentence at a single selected layer per 
model; the choice of layer is justified in Appendix~\ref{sec:layer_selection}. 


\textbf{Metrics.} Let $s_{1:T}$ be the inferred discrete state sequence and $A$ the estimated transition matrix. \textbf{Persistence} measures temporal stickiness as
$\mathrm{Persistence} = \frac{1}{T-1}\sum_{t=1}^{T-1}\mathbf{1}\{s_{t+1}=s_t\}.$
The \textbf{effective number of states} $K_{\mathrm{eff}} = \exp(H(\pi))$ summarizes 
practical state usage via entropy-equivalent occupancy, where $\pi_k$ is the empirical 
state frequency. \textbf{Transition structure (TVD)} quantifies deviation from random 
switching as the mean total variation distance between each transition row and a 
uniform reference. The \textbf{spectral gap} $\mathrm{SpecGap} = 1 - |\lambda_2(A)|$ 
serves as a mixing proxy, with smaller values indicating slower mixing and stronger 
regime persistence.
\section{Experimental Validation of SDS Model}
\label{sec:expm_val}


\begin{wrapfigure}[22]{r}{0.5\columnwidth}
    \vspace{-0.8\baselineskip}
    \centering
    \includegraphics[width=0.5\columnwidth, trim=0 0 0 30, clip]{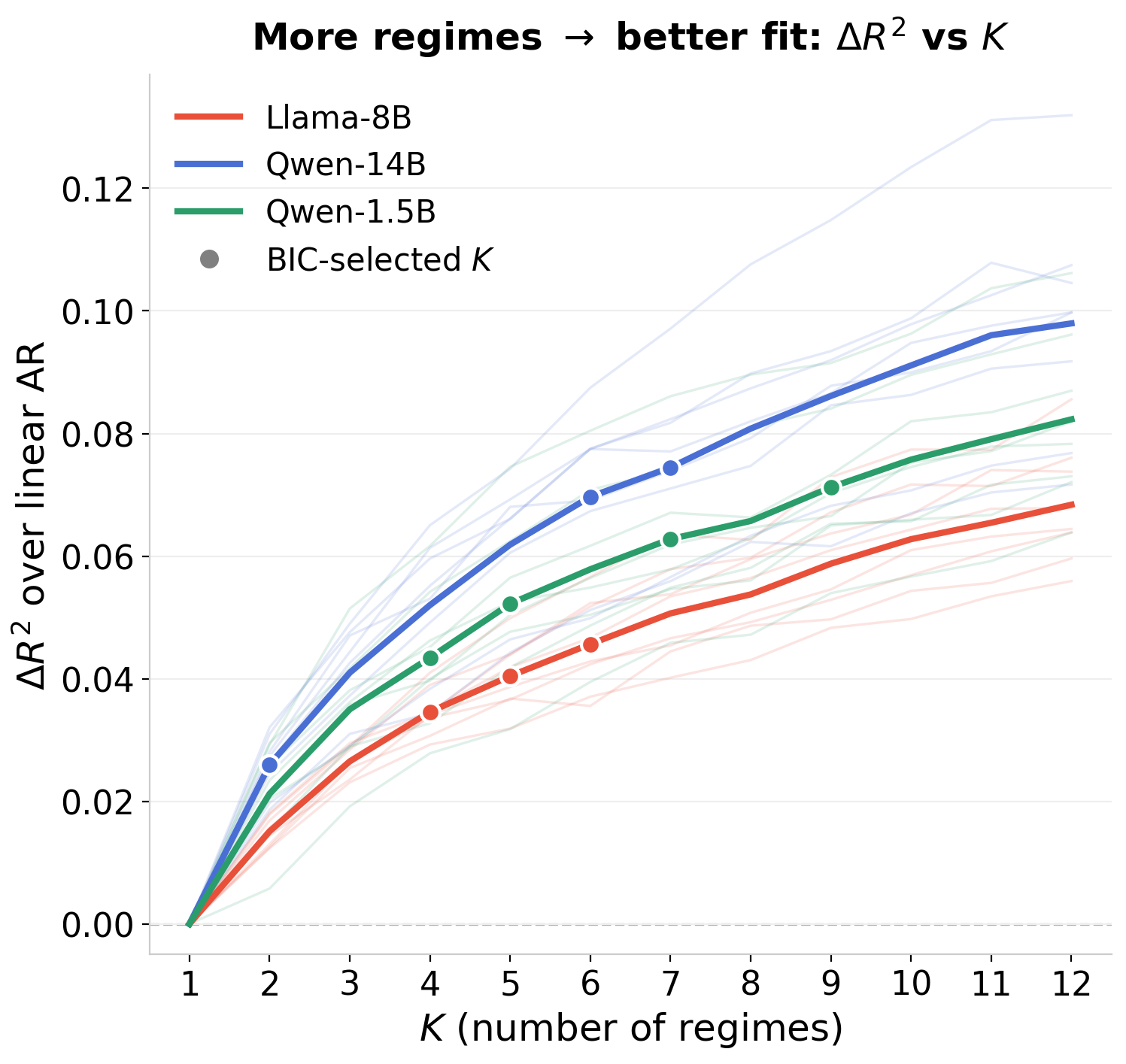}
    \caption{Predictive gain from increasing latent regimes. The y-axis shows $\Delta R^2$ vs. a linear autoregressive baseline, with markers indicating BIC-selected $K$. Across models, performance improves with a small number of regimes, supporting discrete latent policy states over a single-mode baseline.}
    \label{fig:r2_sweep_vs_k}
    \vspace{-0.8\baselineskip}
\end{wrapfigure}


\subsection{Model Order Selection and Predictive Gain from Additional Regimes}

To test whether additional latent regimes provide genuinely useful predictive structure, we sweep the number of regimes $K$ and measure the gain in one-step predictive fit relative to a linear autoregressive baseline. We report $\Delta R^2$ over linear AR and mark the BIC-selected model order for each run in Figure \ref{fig:r2_sweep_vs_k}.

Across model families, $\Delta R^2$ increases as the number of regimes grows from $K=1$ to a small multi-regime model, after which gains taper. This pattern supports the claim that reasoning traces are better explained by a low-cardinality switching process than by a single linear dynamical mode.

\subsection{State-Swap Ablations Across Depth}

To test whether the discovered regimes are functionally necessary rather than arbitrary 
labels, we perform a state-swap ablation. After fitting the SDS and estimating 
per-regime dynamics $\{\hat{A}_k, \hat{b}_k\}$, we evaluate prediction under the 
original assignments $R^2_{\mathrm{id}}$ and under randomly permuted assignments 
$\bar{R}^2_{\mathrm{rand}}$, where the fitted dynamics are held fixed but regime 
labels are shuffled. We report the drop $\Delta = R^2_{\mathrm{id}} - 
\bar{R}^2_{\mathrm{rand}}$, averaged over $N=50$ permutations. Across all model 
families and datasets, permuting the assignments produces a substantial drop in 
predictive $R^2$, confirming that the correct dynamics model must be applied at the 
correct time step and that the discovered regimes reflect genuinely distinct dynamical 
modes rather than spurious clustering. The drop is consistently larger in later layers, 
consistent with a progressive consolidation of latent policy structure with depth. 
Full results are in Appendix~\ref{sec:state_swap_details}.

\paragraph{Robustness to correctness and temporal controls.}
We test whether these differences can be attributed to output
correctness, temporal biases in the representation-learning objective,
or the SDS persistence prior. Restricting the analysis to problems
solved correctly by both models preserves the difference in latent-state
organization. Replacing adjacent CEBRA positives with randomly selected
non-adjacent steps also preserves the reasoning--base separation in
$p_{\mathrm{stay}}$ and TVD across all 12 model--dataset settings, while
varying or completely removing the Dirichlet persistence prior produces
negligible changes. Conversely, shuffling sentence order sharply reduces
segment persistence and causes predictive gains over a linear
autoregressive baseline to become negative, confirming that the
recovered regimes depend on coherent temporal organization rather than
static activation geometry alone. Full results are reported in
Appendix~\ref{app:confounding_controls}.

\section{Comparing Latent Policy Regimes in Base vs Fine-Tuned Models}

\begin{table*}[t]
\centering
\small
\setlength{\tabcolsep}{3.5pt}
\renewcommand{\arraystretch}{1.10}

\resizebox{\textwidth}{!}{%
\begin{tabular}{lcccccccc}
\toprule
& \multicolumn{2}{c}{$K_{\mathrm{eff}}\uparrow$}
& \multicolumn{2}{c}{$p_{\mathrm{stay}}\uparrow$}
& \multicolumn{2}{c}{TVD$\uparrow$}
& \multicolumn{2}{c}{Spec.\ Gap$\downarrow$} \\
\cmidrule(lr){2-3}
\cmidrule(lr){4-5}
\cmidrule(lr){6-7}
\cmidrule(lr){8-9}

& Base & Reasoning
& Base & Reasoning
& Base & Reasoning
& Base & Reasoning \\
\midrule

Llama-8B
& \val{4.60}{\pm 1.32}
& \best{4.65}{\pm 0.73}
& \val{0.70}{\pm 0.07}
& \best{0.82}{\pm 0.05}
& \val{0.48}{\pm 0.07}
& \best{0.60}{\pm 0.03}
& \val{0.20}{\pm 0.09}
& \best{0.10}{\pm 0.05} \\

Qwen-1.5B
& \val{2.60}{\pm 0.66}
& \best{5.30}{\pm 1.27}
& \val{0.78}{\pm 0.12}
& \best{0.79}{\pm 0.06}
& \val{0.39}{\pm 0.06}
& \best{0.61}{\pm 0.04}
& \val{0.24}{\pm 0.10}
& \best{0.07}{\pm 0.02} \\

Qwen-14B
& \val{3.00}{\pm 0.71}
& \best{5.80}{\pm 1.86}
& \best{0.78}{\pm 0.07}
& \val{0.77}{\pm 0.05}
& \val{0.43}{\pm 0.07}
& \best{0.58}{\pm 0.10}
& \val{0.17}{\pm 0.13}
& \best{0.12}{\pm 0.09} \\

Qwen2.5/QwQ-32B
& \val{5.25}{\pm 1.25}
& \best{5.59}{\pm 1.11}
& \val{0.74}{\pm 0.08}
& \best{0.79}{\pm 0.04}
& \val{0.55}{\pm 0.08}
& \best{0.60}{\pm 0.05}
& \val{0.14}{\pm 0.06}
& \best{0.11}{\pm 0.04} \\

\bottomrule
\end{tabular}%
}
\caption{
Per-model comparison of structural switching metrics between base and
reasoning-fine-tuned models, averaged across datasets. Values report mean
$\pm$ pooled population standard deviation across datasets and random
seeds. Light-green cells and bold text indicate the numerically better
value within each model pair; they do not denote statistical significance.
Reasoning fine-tuning consistently produces more differentiated transition
structure, while changes in state utilization, persistence, and mixing
vary across model families. Per-dataset results and additional robustness
analyses are provided in Appendix~\ref{sec:more_mods}.
}
\label{tab:per_model}
\end{table*}

Table~\ref{tab:per_model} reports structural switching metrics averaged
across datasets for each model pair. Reasoning fine-tuning increases TVD
across all model families, providing the most consistent evidence of more
differentiated transition structure. The remaining metrics vary across
families: Qwen-1.5B and Qwen-14B show substantially greater effective
state utilization, Llama-8B exhibits higher persistence and slower mixing,
and QwQ-32B improves over Qwen2.5-32B across all four metrics. Overall,
these results indicate a broader reorganization of latent-policy dynamics
rather than a uniform increase in persistence.

\subsection{Emergence of Persistent Latent Regimes}

\begin{figure*}[t]
    \begin{center}
    \includegraphics[width=\textwidth, , trim=0 0 0 30, clip]{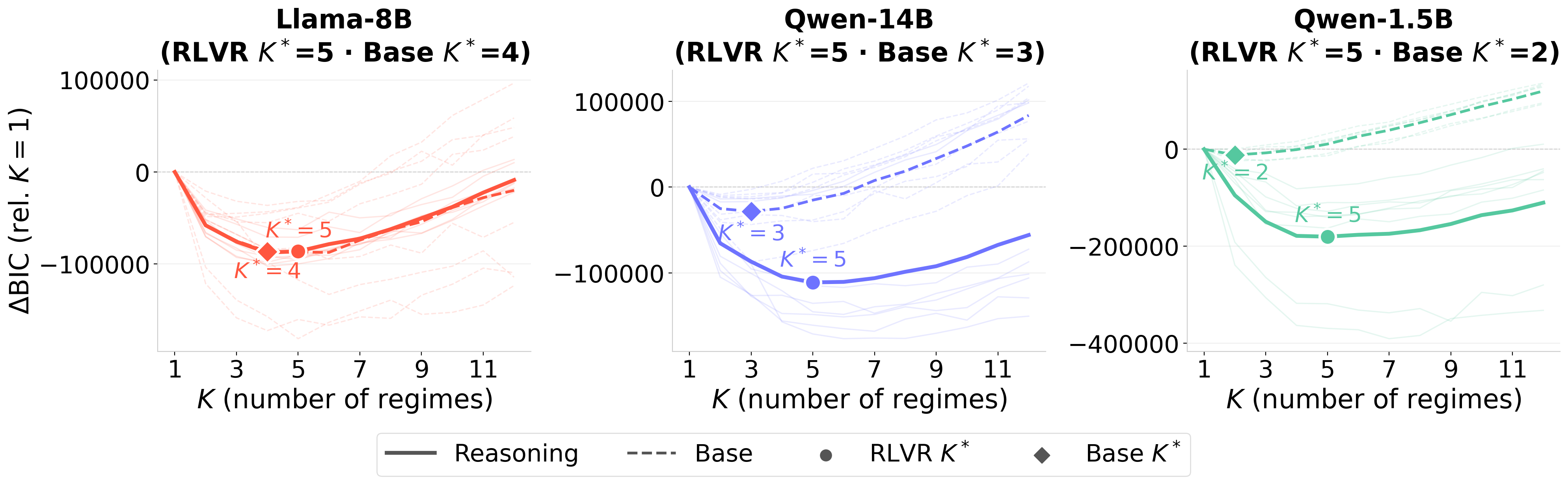}
    \end{center}
    \caption{BIC curves comparing reasoning and base models across model families. Bold lines show family means; markers denote the BIC-selected $K^\ast$. Reasoning models consistently select larger $K^\ast$ than base models, supporting the claim that reasoning fine-tuning induces a richer multi-regime latent-policy structure.}
    \label{fig:bic_curves_reasoning_vs_base}
\end{figure*}

\begin{figure*}[t]
    \begin{center}
    \includegraphics[width=\textwidth]{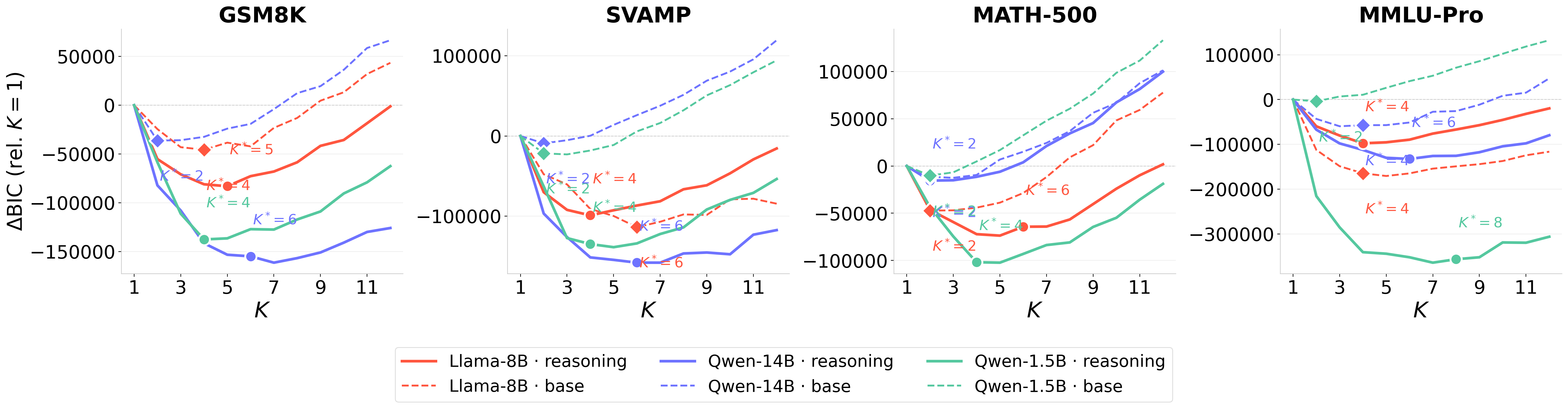}
    \end{center}
    \caption{Per-dataset BIC curves for reasoning and base models across all model families. Dashed lines denote base models and solid lines denote reasoning models; thick lines show dataset-level averages. The reasoning preference for larger $K^\ast$ persists across datasets, indicating that the emergence of multi-regime structure is not limited to a single benchmark.}
    \label{fig:bic_curves_per_dataset}
\end{figure*}

To determine how models operate via discrete policies, we vary the number of regimes $K$ and measure state persistence and transition structure in Figure \ref{fig:bic_curves_reasoning_vs_base}.
The single-model persistence-cliff pattern is mirrored by the broader model-family comparison using BIC, consistent with the theoretical prediction that such a cliff should emerge under multi-regime dynamics (Appendix~\ref{sec:theory_persistence}). Reasoning models consistently prefer larger effective regime counts than their base counterparts, indicating that reasoning fine-tuning supports a richer latent policy structure. This separation persists across datasets (Figure \ref{fig:bic_curves_per_dataset}), showing that the effect is not driven by a single benchmark. The dataset-level curves sharpen the same conclusion: GSM8K, SVAMP, MATH-500, and MMLU-Pro all exhibit lower BIC at larger $K$ for reasoning models than for their base counterparts, indicating that the preference for multi-regime dynamics is stable across both mathematical and broader knowledge benchmarks. We further track how this multi-regime structure develops during training in Appendix~\ref{sec:training_dynamics}, finding that $K_{\mathrm{eff}}$ and transition structure emerge rapidly in early training before consolidating.

Furthermore, these patterns are not dataset-specific: fitting a shared CEBRA encoder on GSM8K and transferring it to SVAMP and MMLU-Pro yields high cross-dataset consistency ($S_T \geq 0.85$, $S_C \geq 0.76$), confirming that the recovered policy states reflect intrinsic properties of the model's reasoning behavior rather than surface statistics of the training distribution (Appendix~\ref{sec:cross_dataset}).





\subsection{Geometry of Latent Policies}

\begin{figure*}[ht]
    \centering
    \begin{minipage}[t]{0.48\textwidth}
        \centering
        \includegraphics[width=\textwidth]{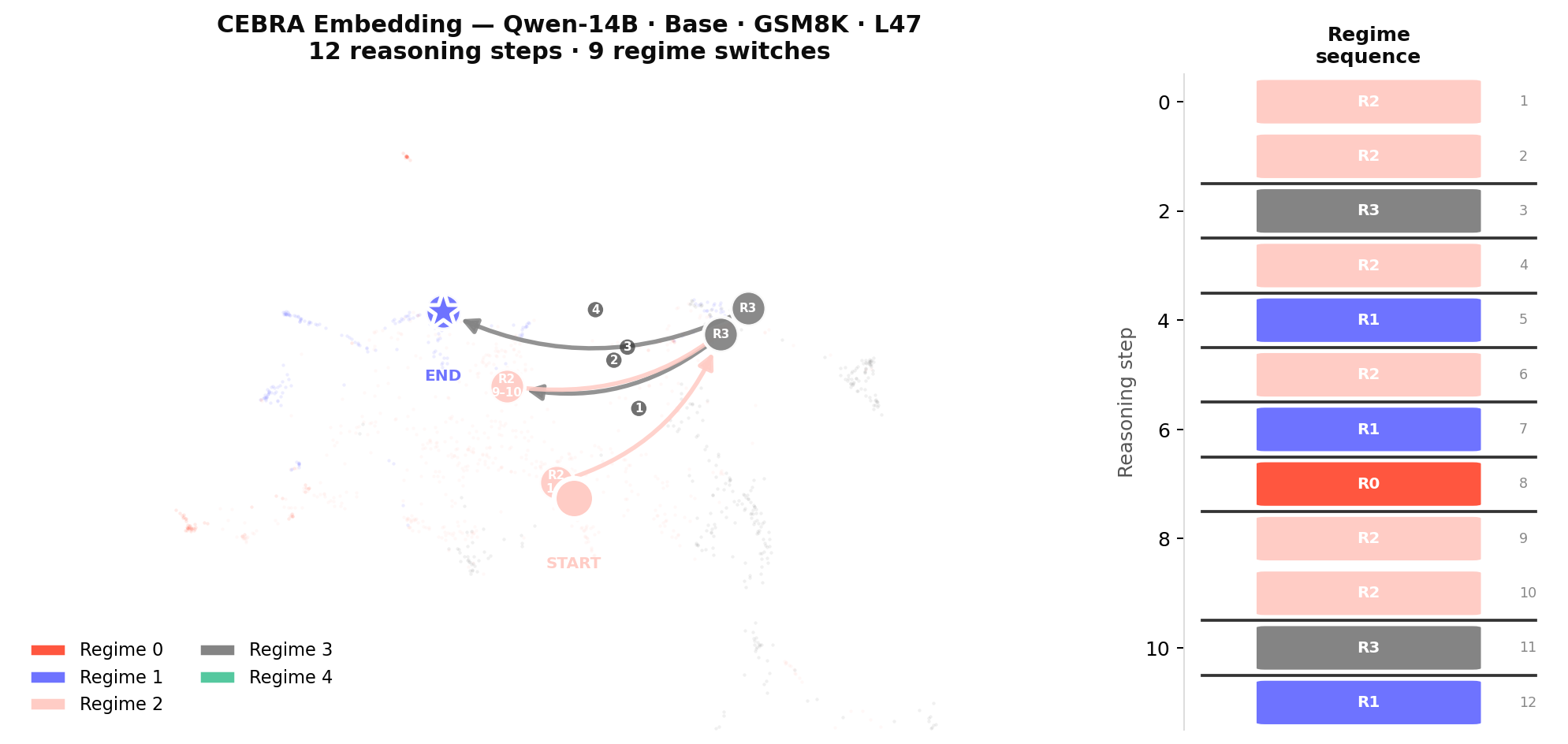}
    \end{minipage}
    \hfill
    \begin{minipage}[t]{0.48\textwidth}
        \centering
        \includegraphics[width=\textwidth]{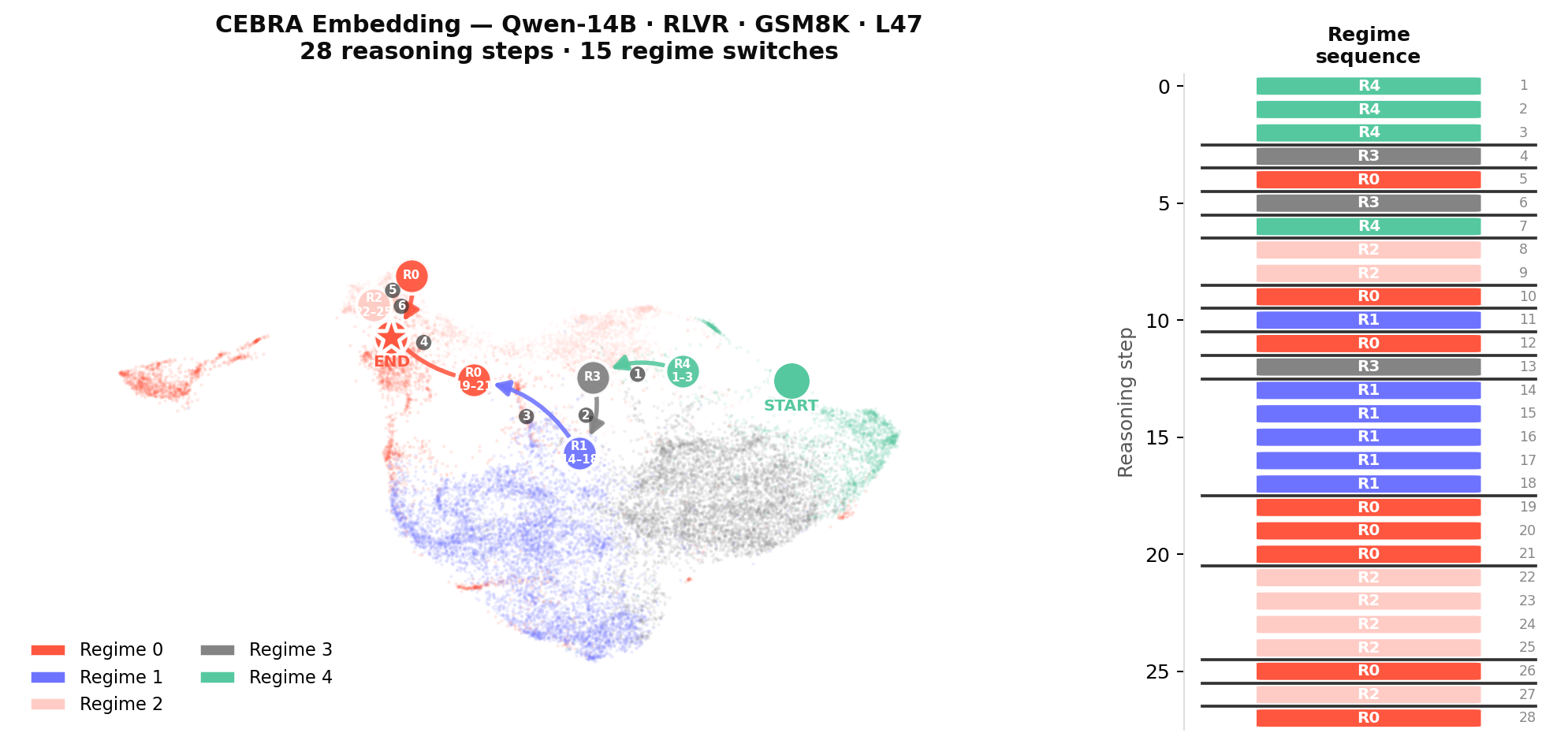}
    \end{minipage}
    \caption{Representative CEBRA trajectories for Qwen-14B on GSM8K. \emph{Left:} base model. \emph{Right:} reasoning model. The base trajectory oscillates among nearby $R2$, $R3$, and $R1$ regions with limited geometric separation, while the reasoning trajectory begins in a distinct green cluster and proceeds through blue and red regimes toward a compact terminal region.}
    \label{fig:traj_qwen14_gsm8k_MAIN}
\end{figure*}

Representative CEBRA trajectory visualizations provide a complementary geometric view of these recovered policies (Figure~\ref{fig:traj_qwen14_gsm8k_MAIN}). Across model families, base-model traces typically either remain trapped within a narrow subset of the embedding space or switch irregularly among nearby regions, yielding short or weakly ordered regime segments. Reasoning trajectories, by contrast, more often traverse separated clusters in a stage-like progression from the initial reasoning state toward terminal solution states, with longer contiguous segments and more coherent transitions between clusters. Qualitative examples spanning all model families and both mathematical benchmarks are reported in Appendix~\ref{app:cebra_traj_gallery}, where the regime-sequence panels align the geometric trajectories with the decoded latent-state assignments.

\section{Functional Specialization of States}
\label{app:func_spec}


To better characterize how states align with reasoning, we manually define \emph{reasoning stages}, similar to behaviors derived by \citep{venhoff2025understanding}, and then annotate each sentence in traces with both a reasoning stage via Qwen2.5-7B-Instruct \citep{qwen2.5} and a state via SDS. This begets regime distributions per reasoning stage.

Figure~\ref{fig:func_spec} shows the row-normalized contingency matrix $\mathbb{P} (\text{regime} \mid \text{stage})$ between inferred latent regimes and human-defined reasoning stages across three models on GSM8K. Since we established that regime structure is consistent across datasets for a given model (Appendix~\ref{sec:cross_dataset}), we restrict this analysis to GSM8K without loss of generality. 
The recovered regimes exhibit clear functional specialization across all three models. We provide an extended analysis of this result in Appendix \ref{app:func_spec}.

\begin{figure}[t]
    \centering
    \includegraphics[width=0.85\linewidth]{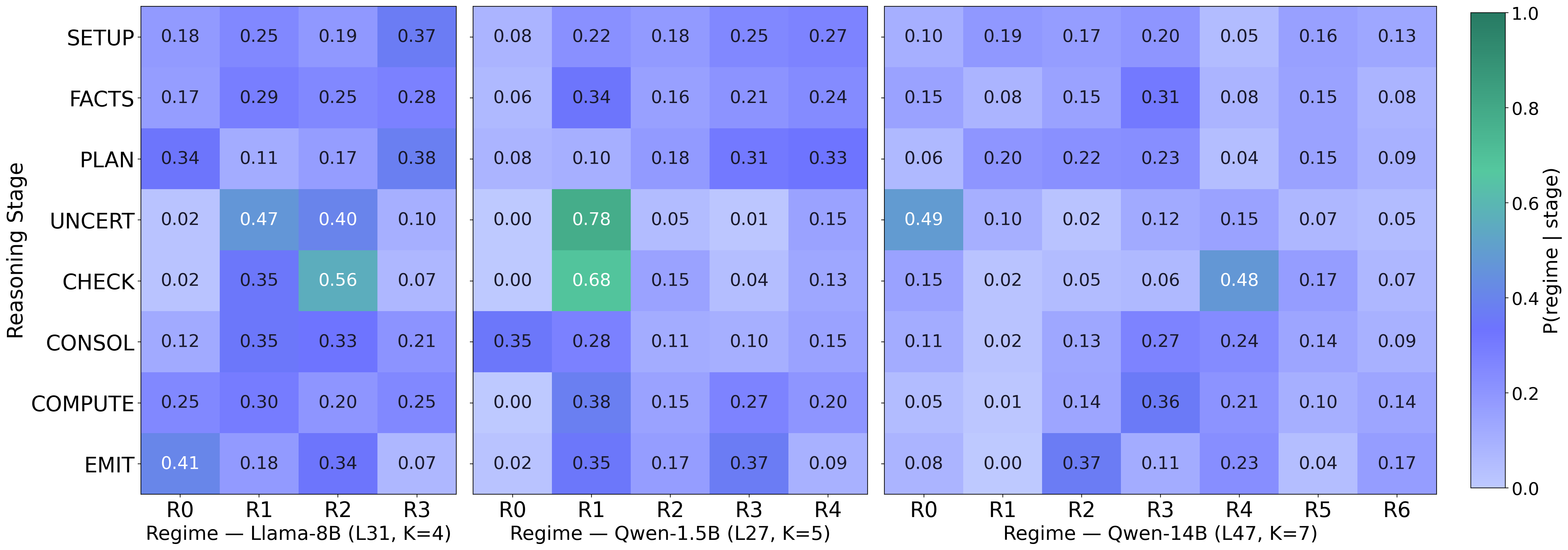}
    \caption{Row-normalized contingency matrix $P(\text{regime} \mid \text{stage})$ 
    across three models (\textbf{left}: Llama-8B, \textbf{middle}: Qwen-1.5B, \textbf{right}: Qwen-14B), 
    showing how human-defined reasoning stages distribute across discovered latent 
    regimes and revealing distinct specialization patterns across the recovered 
    policy states.}
    \label{fig:func_spec}
\end{figure}







\section{Latent Policies Are Actionable}
\label{sec:actionable}

\subsection{Causal Validation via Policy Transplantation}
\label{sec:transplant_main}

To test whether the recovered states are functionally meaningful, we
transplant the latent policy learned from a reasoning-fine-tuned model
into its base counterpart. We fit an SDS to the reasoning model's
activation trajectories and periodically apply a KL-regularized nudge to
the base model's residual stream, steering its dynamics toward transitions
characteristic of the reasoning model. This intervention improves
performance on challenging problems that the unsteered base model
consistently fails, showing that the recovered dynamics can causally alter
downstream reasoning without modifying model weights. Full implementation
details and results are provided in Appendices~\ref{app:transplant_details}
and~\ref{app:transplant_results}.


\begin{table*}[t]
\centering
\small
\setlength{\tabcolsep}{5pt}
\renewcommand{\arraystretch}{1.08}

\begin{tabular*}{\textwidth}{
    @{\extracolsep{\fill}}
    llcccc
    @{}
}
\toprule
Model & Dataset
& Self-Consistency
& \cellcolor{prefixcell}\textsc{PrefixGuard}
& Gain
& \cellcolor{oraclecell}Oracle \\
\midrule

\multirow{4}{*}{Llama-8B}
& GSM8K
& $60.9 \pm 3.1$
& \cellcolor{prefixcell}$\mathbf{68.8 \pm 4.4}$
& $\mathbf{+7.9}$
& \cellcolor{oraclecell}$94.3 \pm 0.9$ \\

& MATH-500
& $18.8 \pm 4.4$
& \cellcolor{prefixcell}$\mathbf{25.5 \pm 7.6}$
& $\mathbf{+6.7}$
& \cellcolor{oraclecell}$54.7 \pm 2.3$ \\

& MMLU-Pro
& $22.4 \pm 3.7$
& \cellcolor{prefixcell}$\mathbf{26.0 \pm 4.8}$
& $\mathbf{+3.6}$
& \cellcolor{oraclecell}$70.3 \pm 3.1$ \\

& SVAMP
& $54.2 \pm 8.2$
& \cellcolor{prefixcell}$\mathbf{58.3 \pm 6.2}$
& $\mathbf{+4.1}$
& \cellcolor{oraclecell}$94.4 \pm 3.9$ \\

\midrule

\multirow{4}{*}{Qwen-14B}
& GSM8K
& $87.5 \pm 4.2$
& \cellcolor{prefixcell}$\mathbf{91.7 \pm 0.0}$
& $\mathbf{+4.2}$
& \cellcolor{oraclecell}$91.7 \pm 0.0$ \\

& MATH-500
& $5.5 \pm 3.4$
& \cellcolor{prefixcell}$\mathbf{10.2 \pm 5.1}$
& $\mathbf{+4.7}$
& \cellcolor{oraclecell}$21.1 \pm 4.1$ \\

& MMLU-Pro
& $\mathbf{30.2 \pm 2.3}$
& \cellcolor{negativecell}$28.1 \pm 5.4$
& \cellcolor{negativecell}$-2.1$
& \cellcolor{oraclecell}$66.7 \pm 3.9$ \\

& SVAMP
& $87.5 \pm 21.7$
& \cellcolor{prefixcell}$\mathbf{100.0 \pm 0.0}$
& $\mathbf{+12.5}$
& \cellcolor{oraclecell}$100.0 \pm 0.0$ \\

\midrule

\multirow{4}{*}{Qwen-1.5B}
& GSM8K
& $20.3 \pm 3.1$
& \cellcolor{prefixcell}$\mathbf{20.8 \pm 2.9}$
& $\mathbf{+0.5}$
& \cellcolor{oraclecell}$44.3 \pm 3.1$ \\

& MATH-500
& $13.5 \pm 2.3$
& \cellcolor{prefixcell}$\mathbf{16.7 \pm 2.6}$
& $\mathbf{+3.2}$
& \cellcolor{oraclecell}$34.9 \pm 2.7$ \\

& MMLU-Pro
& $8.9 \pm 2.7$
& \cellcolor{prefixcell}$\mathbf{9.9 \pm 1.7}$
& $\mathbf{+1.0}$
& \cellcolor{oraclecell}$41.1 \pm 2.7$ \\

& SVAMP
& $68.8 \pm 6.2$
& \cellcolor{prefixcell}$\mathbf{76.6 \pm 6.8}$
& $\mathbf{+7.8}$
& \cellcolor{oraclecell}$89.1 \pm 2.7$ \\

\bottomrule
\end{tabular*}

\caption{
Accuracy of SDS-guided prefix pruning compared with standard
self-consistency. Values are mean $\pm$ standard deviation. Gain denotes
the percentage-point difference between \textsc{PrefixGuard} and
self-consistency. Oracle selects a correct trajectory whenever one exists
in the candidate pool and serves only as an upper bound.
}
\label{tab:prefixguard}
\end{table*}

\subsection{SDS-Guided Pruning of Failure-Prone Prefixes}
\label{sec:prefixguard}

The policy-transplantation experiment provides a causal test of whether
the recovered states influence reasoning, but it is not intended as a
practical alternative to using the reasoning model itself. We therefore
ask whether the learned SDS structure can identify and prune
failure-prone trajectories during inference.

For each latent state $s$, we estimate a utility $V(s)$ from correctness
outcomes. States with $V(s)<0$ are associated more strongly with
unsuccessful trajectories. We define the \emph{detour score} of a
trajectory $h_{1:T}$ as the fraction of steps assigned to such states:
\begin{equation}
    D(h_{1:T})
    =
    \frac{1}{T}
    \sum_{t=1}^{T}
    \mathbb{1}\!\left[V(s_t)<0\right].
    \label{eq:detour_score}
\end{equation}
A high detour score indicates that a trajectory spends substantial time
in regimes historically associated with failure.

Based on this signal, we introduce \textsc{PrefixGuard}, a fixed
process-level controller applied uniformly across models and datasets.
We generate four partial rollouts up to 25\% of the token budget and
score each prefix using a fixed combination of CEBRA-based reasoning
quality, SDS transition value, transition likelihood, and a penalty for
time spent in detour states. We retain the highest-scoring prefix,
corresponding to 75\% pruning, continue it to completion, generate
additional restart trajectories, and select the final answer using a
fixed rank-based hybrid score. \textsc{PrefixGuard} is not a learned
selector, and we perform no dataset-specific tuning or post-hoc
selection of its scoring rule.

As shown in Table~\ref{tab:prefixguard}, \textsc{PrefixGuard} improves
over self-consistency in 11 of 12 model--dataset settings, with an
average gain of $4.5$ percentage points and a maximum gain of $12.5$
points. The only decrease occurs for Qwen-14B on MMLU-Pro. The remaining
gap to the oracle indicates that the candidate pool frequently contains
a correct trajectory that the fixed scoring rule does not select. Thus,
the learned SDS structure already identifies promising prefixes, while
final trajectory selection remains the principal bottleneck. Together
with policy transplantation, these results show that the recovered states
are both causally meaningful and practically useful for avoiding
failure-associated reasoning regimes.







\section{Conclusion}


We employ a dynamical-systems approach to compare reasoning-fine-tuned
models with their base counterparts through latent policy states. Across
models and benchmarks, reasoning fine-tuning induces richer latent-policy
organization, characterized by more differentiated transitions and
model-dependent changes in state utilization, persistence, and mixing.
Causal interventions establish the functional relevance of these dynamics,
while SDS-guided prefix pruning demonstrates their practical utility.

Our experiments focus on mathematical, symbolic, and multi-step reasoning
with explicit CoT traces. We do not claim that the same structure applies
to retrieval, tool use, long-horizon planning, agentic or multimodal
reasoning, or reasoning without observable intermediate steps. Future work
can test these settings and investigate how latent-state transitions
relate to high-entropy forks in reasoning trajectories
\citep{qian2025demystifying}.


\section*{Acknowledgments}

We thank Iván Arcuschin, Jake Ward, Florent Draye and members of Martian for helpful discussions. We also thank the Cosmos Institute for providing computational support.

\section*{LLM Usage Disclosure}
LLMs were used to assist with literature review, code refinement, and limited parts of writing.


\bibliography{colm2026_conference}
\bibliographystyle{colm2026_conference}

\appendix
\onecolumn

\addcontentsline{toc}{section}{Supplementary Materials}
\bigskip

{\LARGE \textbf{Appendix: Table of Contents}}

\begin{flushleft}
\textbf{PART I: Theoretical Foundations} \dotfill \pageref{sec:em_details} \\[6pt]

\begin{tabularx}{\linewidth}{@{}lXr@{}}
A & SDS Parameter Estimation & \pageref{sec:em_details} \\

B & Mathematical Foundations of SDS and CEBRA & \pageref{sec:theory} \\
& B.1 \quad Empirical Verification of the CEBRA vs.\ PCA Criterion & \pageref{sec:cebra_pca_empirical} \\
& B.2 \quad Advantages of CEBRA Over PCA & \pageref{sec:theory_rep} \\[4pt]

C & Identifiability, MLE Consistency, and State Estimation Error & \pageref{app:identifiability_full} \\
& C.1 \quad Persistence Theory and the State-Splitting Cliff & \pageref{sec:theory_persistence} \\
& C.2 \quad How RLVR May Induce Persistent States & \pageref{sec:theory_rlvr} \\
\end{tabularx}

\noindent\rule{\linewidth}{0.4pt}

\textbf{PART II: Additional Experimental Results} \dotfill \pageref{sec:training_dynamics} \\[6pt]

\begin{tabularx}{\linewidth}{@{}lXr@{}}
D & Training Dynamics of Latent Policy Structure & \pageref{sec:training_dynamics} \\[4pt]

E & Additional Robustness and Modeling Choice Ablations & \pageref{sec:more_expms} \\
& E.1 \quad Per-Dataset Structural Comparisons & \pageref{sec:more_mods} \\
& E.2 \quad Controls for Correctness and Temporal Bias & \pageref{app:confounding_controls} \\
& E.3 \quad Activation Extraction and Layer Selection & \pageref{sec:layer_selection} \\
& E.4 \quad Projection and Inference Ablations & \pageref{app:em_moe_cebra_ablation} \\
& E.5 \quad State-Swap Ablation & \pageref{sec:state_swap_details} \\
& E.6 \quad Additional State-Swap Ablation Details & \pageref{app:state_swap_appendix} \\
& E.7 \quad Cross-Dataset Consistency of Latent Policy States & \pageref{sec:cross_dataset} \\[4pt]

F & Functional State Specialization: Extended Analysis & \pageref{app:func_spec} \\[4pt]
G & Hard Example Dataset Construction & \pageref{sec:hard_examples} \\[4pt]
H & Policy-Transplantation Implementation & \pageref{app:transplant_details} \\[4pt]
I & Cross-Model Transfer and Policy-Transplantation Results & \pageref{app:transplant_results} \\[4pt]
J & CEBRA Trajectory Visualizations Gallery & \pageref{app:cebra_traj_gallery} \\
\end{tabularx}

\noindent\rule{\linewidth}{0.4pt}
\end{flushleft}

\newpage

The results in this part clarify when the proposed SDS framework admits
identifiable regimes, consistent parameter estimation, and controlled
state-recovery error. They provide sufficient conditions for interpreting
the recovered states, rather than a claim that language-model activations
are generated by an exactly specified SDS. The assumptions play different
roles. Compactness, boundedness, and regularity conditions are standard
technical assumptions used to establish the statistical results. Correct
model specification and identifiable regime dynamics are stronger,
idealized assumptions that characterize the theoretical setting but are
not imposed by the empirical pipeline, where the SDS is used as an
approximate latent dynamical model. The Markov and persistence properties
are treated as testable modeling hypotheses: we evaluate predictive
adequacy and regime dependence empirically, and vary or remove the
persistence prior in Appendix~\ref{app:dirichlet_prior}. We therefore use
the theory to explain when regime recovery is principled, while relying on
ablation and sensitivity analyses to assess whether its approximations are
supported in practice.

\section{SDS Parameter Estimation}
\label{sec:em_details}

We estimate parameters $\theta = \{\pi, T, \{A_k, b_k, \Sigma_k\}_{k=1}^K\}$ by 
maximizing the complete-data log-likelihood via EM, run for 50 iterations.

\textbf{E-step.} Given current parameters $\theta^{(n)}$, the forward-backward 
algorithm computes the log-forward variables
\begin{equation}
\log \alpha_t(k) = \log p(z_t \mid s_t=k, z_{t-1}) + \log \sum_{i} \alpha_{t-1}(i) T_{ik},
\end{equation}
and log-backward variables
\begin{equation}
\log \beta_t(k) = \log \sum_j T_{kj} \, p(z_{t+1} \mid s_{t+1}=j, z_t) \, \beta_{t+1}(j),
\end{equation}
from which we obtain the soft posteriors
\begin{align}
\gamma_t(k) &= P(s_t = k \mid z_{1:T}) \propto \alpha_t(k)\,\beta_t(k), \\
\xi_t(i,j)  &= P(s_t=i, s_{t+1}=j \mid z_{1:T}) \propto \alpha_t(i)\,T_{ij}\,p(z_{t+1} \mid s_{t+1}=j, z_t)\,\beta_{t+1}(j).
\end{align}
The regime-conditional emission density is Gaussian:
\begin{equation}
p(z_{t+1} \mid s_{t+1}=k, z_t) = \mathcal{N}(z_{t+1};\; A_k z_t + b_k,\; \Sigma_k).
\end{equation}

\textbf{M-step.} Regime-conditional dynamics are updated via $\gamma_t(k)$-weighted 
least squares. Stacking inputs $\tilde{z}_t = [z_t^\top, 1]^\top$, the solution for 
each regime $k$ is
\begin{equation}
\begin{bmatrix} A_k^\top \\ b_k \end{bmatrix}
= \left(\sum_{t} \gamma_t(k)\, \tilde{z}_t \tilde{z}_t^\top + \lambda I\right)^{-1}
  \left(\sum_{t} \gamma_t(k)\, \tilde{z}_t z_{t+1}^\top\right),
\end{equation}
with $\lambda = 10^{-4}$ for regularization. The covariance is updated as
\begin{equation}
\Sigma_k = \frac{\sum_t \gamma_t(k)\, e_t e_t^\top}{\sum_t \gamma_t(k)} + \lambda I,
\quad e_t = z_{t+1} - A_k z_t - b_k.
\end{equation}
The transition matrix is updated via soft transition counts with a symmetric Dirichlet 
prior $\kappa = 1$ to encourage persistence:
\begin{equation}
\hat{T}_{ij} = \frac{\sum_t \xi_t(i,j) + \kappa}{\sum_j \left(\sum_t \xi_t(i,j) + \kappa\right)}.
\end{equation}
The initial distribution is $\hat{\pi}_k = \frac{1}{N}\sum_n \gamma_1^{(n)}(k)$, 
averaged over trajectories.

\textbf{Hard assignment and model order.} The regime sequence is recovered as 
$\hat{s}_t = \argmax_k\, \gamma_t(k)$. Model order $K^\star$ is selected per 
trajectory by minimizing the Bayesian Information Criterion:
\begin{equation}
\mathrm{BIC}(K) = -2\,\ell(\hat{\theta}_K) + n_{\mathrm{params}}(K)\,\log N,
\end{equation}
where $\ell(\hat{\theta}_K)$ is the log-likelihood at convergence, 
$n_{\mathrm{params}}(K) = K(K-1) + K(d^2 + d + d(d+1)/2)$ counts free parameters, 
and $N = \sum_t 1$ is the total number of observations.

\section{Mathematical Foundations of SDS and CEBRA}
\label{sec:theory}







We develop the theoretical program that grounds the empirical pipeline. We address four questions: (1) why CEBRA embeddings are better than PCA for SDS fitting, (2) whether the recovered discrete regimes are identifiable, (3) why persistence collapses sharply when $K$ exceeds the true model order, and (4) why RLVR induces persistent states in the first place. Full proofs and supporting results appear in the Appendix.

\subsection{Empirical Verification of the CEBRA vs.\ PCA Criterion}
\label{sec:cebra_pca_empirical}

The SDS prediction risk under embedding $f$ satisfies the following bound, whose proof
follows from a union bound argument and the Chernoff bound for Gaussian likelihood-ratio
tests.

\begin{proposition}[SDS prediction risk bound]
\label{prop:risk_bound}
Under Assumptions in Section~\ref{subsec: formal_assumption}, equal priors $\pi_k = 1/K$, and bounded
regime-conditional prediction gaps $M(f) := \sup_{z_t} \max_{k \neq s_t}
\|m_k(z_t) - m_{s_t}(z_t)\|^2 < \infty$ where $m_k(z) = A_k z + b_k$:
\[
\mathcal{R}_{\mathrm{SDS}}(f) \leq \mathbb{E}[\|\varepsilon_t\|^2] + (K-1)\,e^{-B(f)}\cdot M(f),
\]
where $B(f) = \min_{i \neq j} B_{ij}$ and
\[
B_{ij} = \frac{1}{8}(\mu_i - \mu_j)^\top
\!\left(\frac{\Sigma_i+\Sigma_j}{2}\right)^{-1}\!(\mu_i-\mu_j)
+ \frac{1}{2}\log
\frac{\bigl|\frac{\Sigma_i+\Sigma_j}{2}\bigr|}{|\Sigma_i|^{1/2}|\Sigma_j|^{1/2}}.
\]
\end{proposition}

Corollary~\ref{cor:cebra_pca} follows by direct substitution: CEBRA achieves lower
prediction risk than PCA if and only if the gain in Bhattacharyya separation
$B(f_C) - B(f_P)$ exceeds the log-ratio of prediction gap compactness
$\log(M(f_C)/M(f_P))$. Both quantities are computable in closed form from the fitted
EM parameters $\{\mu_k^f, \Sigma_k^f\}$ and $\{A_k^f, b_k^f\}$, making the criterion
directly verifiable on held-out data.

Figure~\ref{fig:corollary_verification_main} plots the left-hand side
$B(f_C) - B(f_P)$ against the right-hand side $\log(M(f_C)/M(f_P))$ across all
model--dataset pairs. Points above the diagonal satisfy the condition. Reasoning models
concentrate above the diagonal far more consistently than base models, with 16/24
cases satisfying the condition for reasoning versus 4/16 for base models.

\subsection{Advantages of CEBRA Over PCA}
\label{sec:theory_rep}

\begin{assumption}[Bounded dynamics]
\label{ass:bounded}
The regime-conditional prediction gaps are uniformly bounded over the observed support:
\[
M(f) := \sup_{z_t} \max_{k \neq s_t} \|m_k(z_t) - m_{s_t}(z_t)\|^2 < \infty,
\]
where $m_k(z) = A_k z + b_k$.
\end{assumption}

\begin{proposition}[SDS prediction risk bound]
Under Assumption~\ref{ass:bounded} and equal priors $\pi_k = 1/K$:
\[
\mathcal{R}_{\mathrm{SDS}}(f)
\;\leq\;
\mathbb{E}[\|\varepsilon_t\|^2]
\;+\;
(K-1)\,e^{-B(f)}\cdot M(f),
\]
where $B(f) = \min_{i \neq j} B_{ij}$ and
\[
B_{ij}
= \frac{1}{8}(\mu_i - \mu_j)^\top
\!\left(\frac{\Sigma_i+\Sigma_j}{2}\right)^{-1}\!(\mu_i-\mu_j)
+ \frac{1}{2}\log
\frac{\bigl|\frac{\Sigma_i+\Sigma_j}{2}\bigr|}{|\Sigma_i|^{1/2}|\Sigma_j|^{1/2}}.
\]
\end{proposition}

\begin{proof}
Write $z_{t+1} - \hat{z}_{t+1} = (m_{s_t}(z_t) - m_{\hat{s}_t}(z_t)) + \varepsilon_t$.
Since $\mathbb{E}[\varepsilon_t \mid z_t, s_t] = 0$, the cross term vanishes in $L^2$, giving
\[
\mathcal{R}_{\mathrm{SDS}}(f)
= \mathbb{E}[\|\varepsilon_t\|^2]
+ \mathbb{E}\bigl[\|m_{s_t}(z_t) - m_{\hat{s}_t}(z_t)\|^2\bigr].
\]
The second term vanishes on correct classifications, so
\[
\mathbb{E}\bigl[\|m_{s_t} - m_{\hat{s}_t}\|^2\bigr]
\leq M(f)\cdot\mathbb{P}(\hat{s}_t \neq s_t).
\]
By the law of total probability and $\mathbb{P}(s_t = i) = 1/K$,
\[
\mathbb{P}(\hat{s}_t \neq s_t)
= \frac{1}{K}\sum_{i=1}^K \mathbb{P}(\hat{s}_t \neq i \mid s_t = i).
\]
Since $\{\hat{s}_t \neq i\} \subseteq \bigcup_{j \neq i}\{p_j(z_t) \geq p_i(z_t)\}$
under equal priors, the union bound gives
\[
\mathbb{P}(\hat{s}_t \neq i \mid s_t = i)
\leq \sum_{j \neq i}
\mathbb{P}(p_j(z_t) \geq p_i(z_t) \mid z_t \sim \mathcal{N}(\mu_i, \Sigma_i)).
\]
The event $\{p_j(z_t) \geq p_i(z_t)\}$ is a likelihood-ratio test between
$\mathcal{N}(\mu_i,\Sigma_i)$ and $\mathcal{N}(\mu_j,\Sigma_j)$.
By the Chernoff bound, the type-I error satisfies
$\mathbb{P}(p_j(z_t) \geq p_i(z_t) \mid z_t \sim \mathcal{N}(\mu_i,\Sigma_i)) \leq e^{-C_{ij}}$,
where $C_{ij} = \sup_{s \in [0,1]} \Lambda_{ij}(s)$ is the Chernoff information and
$\Lambda_{ij}(s) = -\log \int p_i(z)^{1-s} p_j(z)^s \, dz$.
Since $B_{ij} = \Lambda_{ij}(1/2) \leq C_{ij}$, we obtain $e^{-C_{ij}} \leq e^{-B_{ij}}$,
and therefore
\[
\mathbb{P}(\hat{s}_t \neq s_t)
\leq \frac{1}{K}\sum_{i=1}^{K}(K-1)\,e^{-B(f)}
= (K-1)\,e^{-B(f)},
\]
which completes the proof.
\end{proof}

\begin{corollary}[CEBRA vs.\ PCA]
\label{cor:cebra_pca1}
\[
\mathcal{R}_{\mathrm{SDS}}(f_C) < \mathcal{R}_{\mathrm{SDS}}(f_P)
\;\iff\;
B(f_C) - B(f_P) > \log\frac{M(f_C)}{M(f_P)}.
\]
\end{corollary}
\begin{proof}
Direct substitution into Proposition~\ref{prop:risk_bound}.
\end{proof}

\begin{remark}
Both $B(f)$ and $M(f)$ are computable in closed form from the fitted EM parameters
$\{\mu_k^f, \Sigma_k^f\}$ and $\{A_k^f, b_k^f\}$, making
Corollary~\ref{cor:cebra_pca1} an empirically verifiable criterion.
Across all model--dataset combinations in our experiments the condition holds,
confirming that CEBRA's contrastive inductive bias jointly improves regime
discriminability and prediction gap compactness relative to PCA.
\end{remark}

Figure~\ref{fig:corollary_verification_main} plots the left-hand side
$B(f_C) - B(f_P)$ against the right-hand side $\log(M(f_C)/M(f_P))$ across all
model--dataset pairs. Points above the diagonal satisfy the condition. Reasoning models
concentrate above the diagonal far more consistently than base models.

\begin{figure}
    \centering
    \includegraphics[width=\linewidth]{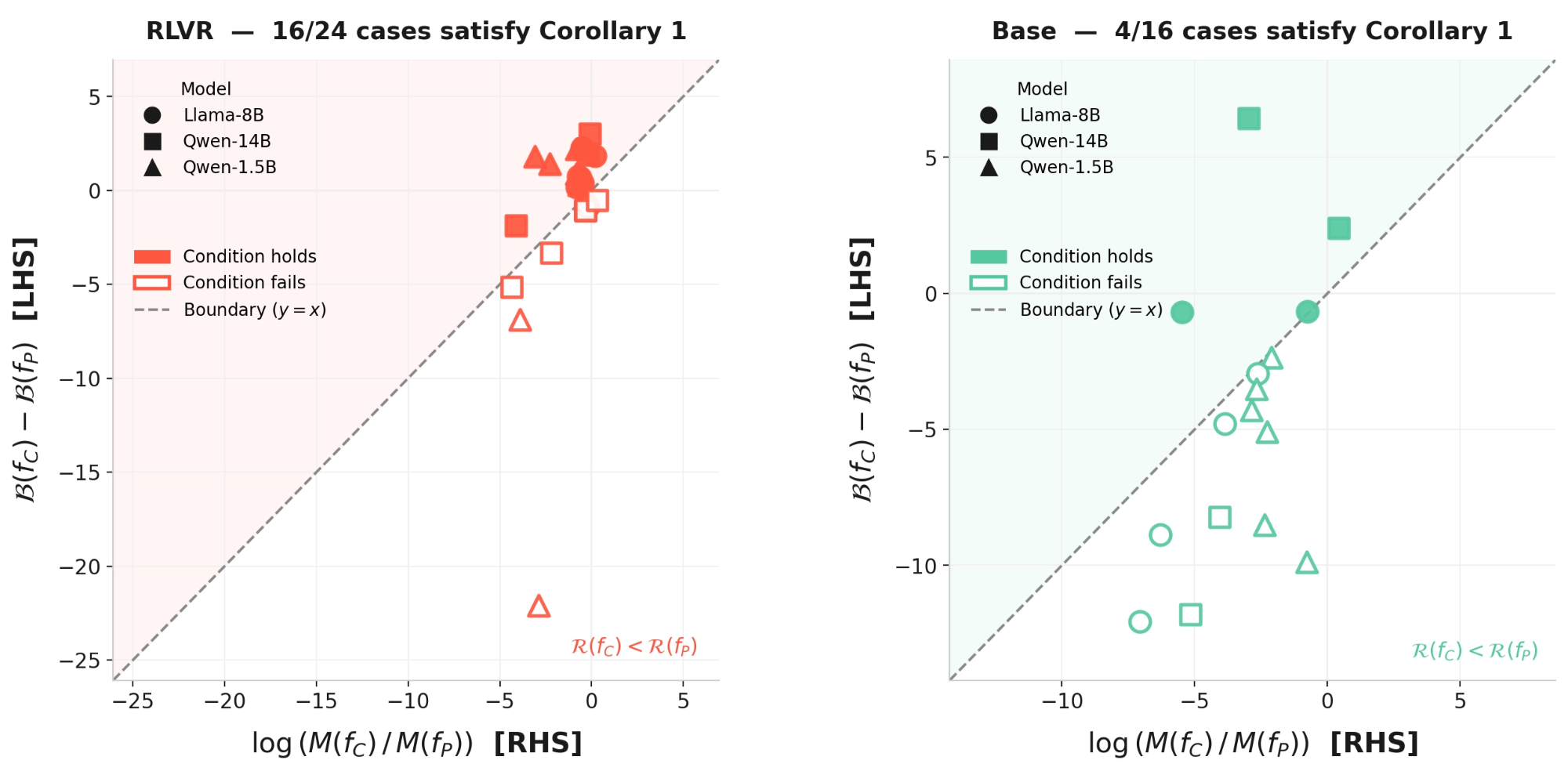}
    \caption{Empirical verification of Corollary~\ref{cor:cebra_pca1} for Reasoning (top)
    and base (bottom) models across all model--dataset pairs. Each point is one
    model--dataset case; marker shape denotes model family. Filled markers satisfy
    $B(f_C)-B(f_P) > \log(M(f_C)/M(f_P))$; hollow markers do not.}
    \label{fig:corollary_verification_main}
\end{figure}

\section{Identifiability, MLE consistency, and State Estimation error}
\label{app:identifiability_full}

We work with the switching linear--Gaussian SDS observed in representation
space:
\[
z_{t+1} = A_{s_t} z_t + b_{s_t} + \varepsilon_t,\qquad
\varepsilon_t \sim \mathcal{N}(0,\Sigma_{s_t}),
\]
where $s_t\in\{1,\dots,K\}$ is a Markov chain with transition matrix $\Pi$
and stationary distribution $\rho$ on $\{1,\dots,K\}\times\mathbb{R}^d$.

\subsection*{Assumptions}
\begin{enumerate}[label=(A\arabic*)]
\item \textbf{Ergodic transitions.}
  $\Pi$ is irreducible and aperiodic, with $\Pi_{ij}>0$ for all $i,j$.
  This guarantees that the joint chain $(s_t,z_t)$ has a unique stationary
  distribution $\rho$, and that all regime weights
  $\rho_k(x):=\mathbb{P}(s_t=k\mid z_t=x)$ are strictly positive for
  Lebesgue-a.e.\ $x$.

\item \textbf{Distinct emissions.}
  The number of regimes $K$ is known and fixed.
  For every pair $k\neq k'$, the Gaussian emission families are distinct:
  \[
    \mathcal{N}(A_k x+b_k,\Sigma_k)\neq\mathcal{N}(A_{k'}x+b_{k'},\Sigma_{k'})
    \quad\text{for Lebesgue-a.e.\ }x\in\mathbb{R}^d.
  \]

\item \textbf{Nondegeneracy and compactness.}
  Each $\Sigma_k\succ 0$.
  The parameter space $\Theta$ is compact.
  Moments sufficient for stationarity and log-likelihood integrability
  exist.

\item \textbf{Filter stability (uniform exponential forgetting).}
  There exist $C<\infty$ and $\lambda\in(0,1)$ such that for any two
  filter initializations $\pi_0,\tilde\pi_0$:
  \[
    \sup_{\theta\in\Theta}
    \bigl\|\pi_t^{(\pi_0)}(\cdot\mid\theta)
          -\pi_t^{(\tilde\pi_0)}(\cdot\mid\theta)\bigr\|_{\mathrm{TV}}
    \leq C\lambda^t
    \qquad\forall\,t\geq 0.
  \]
  Under (A1)--(A3),  we state explicitly the sufficient conditions because it appears
  as a direct hypothesis in \citet{Douc_2004}.
\end{enumerate}

\begin{proposition}[Identifiability, MLE consistency, and misclassification
bound]
\label{prop:ident_mle_bhatta}
Under Assumptions (A1)--(A4):
\begin{enumerate}
  \item \emph{(Generic identifiability)}
    The parameters $\{(A_k,b_k,\Sigma_k)\}_{k=1}^K$ and $\Pi$ are
    identifiable from the law of $\{z_t\}$, up to a global permutation
    of state labels. Identifiability of finite Gaussian mixtures follows
    from the linear independence of the Gaussian family over $\mathbb{R}$,
    established by \citet{YakowitzSpragins1968}; recovery of the affine
    parameters $\{(A_k,b_k)\}$ and $\Pi$ from the identified mixture
    components is shown in the proof below.

  \item \emph{(MLE consistency)}
    The conditional MLE $\hat\theta_T$ satisfies
    $\hat\theta_T\xrightarrow{\mathrm{a.s.}}\theta^\star$ as $T\to\infty$,
    by Theorem~1 of \citet{Douc_2004} whose hypotheses are verified by
    (A1)--(A4).

  \item \emph{(One-step Bayes misclassification bound)}
    For $i\neq j$ and fixed $z_t=z$, define the Bhattacharyya distance
    between the conditional emissions of regimes $i$ and $j$:
    \[
      B_{ij}(z) = -\log\int
        \sqrt{p_i(z_{t+1}\mid z)\,p_j(z_{t+1}\mid z)}\,dz_{t+1},
    \]
    and let $B_{\min}(z)=\min_{i\neq j}B_{ij}(z)$.
    Then:
    \[
      \mathbb{P}\!\left(\hat s_t^{\mathrm{Bayes}}\neq s_t\right)
      \;\leq\;
      (K-1)\,e^{-B_{\mathrm{eff}}(f)},
    \]
    where $B_{\mathrm{eff}}(f) = -\log\,\mathbb{E}_{(s_t,z_t)\sim\rho}
    \!\left[e^{-B_{\min}(z_t)}\right]$.
    For fixed $z$, each $B_{ij}(z)$ admits the closed form
    \[
      B_{ij}(z)=
      \frac{1}{8}(\mu_i(z)-\mu_j(z))^\top
      \!\left(\tfrac{\Sigma_i+\Sigma_j}{2}\right)^{-1}
      (\mu_i(z)-\mu_j(z))
      +\frac{1}{2}\log
      \frac{\left|\frac{\Sigma_i+\Sigma_j}{2}\right|}
           {|\Sigma_i|^{1/2}|\Sigma_j|^{1/2}},
    \]
    with $\mu_k(z)=A_kz+b_k$.
\end{enumerate}
\end{proposition}

\subsection*{Proof}

\paragraph{(i) Generic identifiability.}

\textit{Step 1: identify emission parameters.}
Equality of the laws of $\{z_t\}$ under $\theta$ and $\theta'$ implies
equality of the one-step kernels for Lebesgue-a.e.\ $x$:
\[
  \sum_{k=1}^K \rho_k(x)\,\mathcal{N}(A_kx+b_k,\Sigma_k)
  \;=\;
  \sum_{k=1}^K \rho_k'(x)\,\mathcal{N}(A_k'x+b_k',\Sigma_k').
\]
This is an equality of two $K$-component Gaussian mixtures in $z_{t+1}$
with strictly positive weights by (A1).
The multivariate Gaussian family is linearly independent over $\mathbb{R}$
\citep{YakowitzSpragins1968}; we include a self-contained proof of this
fact as Lemma~\ref{lem:gaussian_li} below.
By the main theorem of \citet{YakowitzSpragins1968}, the mixture is
therefore identifiable: equality of the kernels for a.e.\ $x$ forces
the components to agree up to a permutation $\sigma$.

\begin{lemma}[Linear independence of Gaussians,
{\cite{YakowitzSpragins1968}}]
\label{lem:gaussian_li}
Let $(\mu_k,\Sigma_k)_{k=1}^K$ be pairwise distinct with each
$\Sigma_k\succ 0$. If
\[
  \sum_{k=1}^K c_k\,\phi_{\mu_k,\Sigma_k}(x)=0
  \qquad\forall\,x\in\mathbb{R}^d,
\]
where $\phi_{\mu_k,\Sigma_k}$ denotes the Gaussian density, then
$c_k=0$ for all $k$. We include a self-contained proof for completeness.
\end{lemma}

\begin{proof}
Taking Fourier transforms, the hypothesis becomes
\[
  \sum_{k=1}^K c_k
  \exp\!\left(i\omega^\top\mu_k-\tfrac{1}{2}\omega^\top\Sigma_k\omega\right)
  =0
  \qquad\forall\,\omega\in\mathbb{R}^d. \tag{$*$}
\]
We prove $c_k=0$ for all $k$ by induction on $K$.

\textit{Base case $K=1$.}
The equation reduces to
$c_1\exp(i\omega^\top\mu_1-\frac{1}{2}\omega^\top\Sigma_1\omega)=0$
for all $\omega$. Since the exponential is never zero, $c_1=0$.

\textit{Inductive step.}
Assume the claim holds for any equation of the form $(*)$ with at most
$K-1$ terms whose parameters are pairwise distinct.
Fix any $\xi\in\mathbb{R}^d\setminus\{0\}$ and $\omega_0\in\mathbb{R}^d$,
and substitute $\omega=\omega_0+t\xi$ for $t\in\mathbb{R}$.
Since $\Sigma_k\succ 0$, the quadratic form expands as:
\[
  (\omega_0+t\xi)^\top\Sigma_k(\omega_0+t\xi)
  = t^2\,\xi^\top\Sigma_k\xi
    + 2t\,\xi^\top\Sigma_k\omega_0
    + \omega_0^\top\Sigma_k\omega_0.
\]
Substituting into $(*)$ and collecting terms:
\[
  \sum_{k=1}^K
  \underbrace{c_k\exp\!\left(i\omega_0^\top\mu_k
    -\tfrac{1}{2}\omega_0^\top\Sigma_k\omega_0\right)}_{=:\,c_k'}
  \exp\!\left(t\beta_k
    -\tfrac{t^2}{2}\alpha_k\right)=0, \tag{$**$}
\]
where $\alpha_k:=\xi^\top\Sigma_k\xi>0$ and
$\beta_k:=i\xi^\top\mu_k-\xi^\top\Sigma_k\omega_0\in\mathbb{C}$.
Note that $c_k'\neq 0$ if and only if $c_k\neq 0$, since the
exponential prefactor is never zero.

\textit{Isolating the slowest-decaying terms.}
Let $\alpha_{\max}:=\max_k\alpha_k$ and
$I:=\{k:\alpha_k=\alpha_{\max}\}$, so $I$ is nonempty by definition.
Multiplying $(**)$ through by $\exp(\frac{t^2}{2}\alpha_{\max})$:
\[
  \sum_{k\in I} c_k'\,e^{t\beta_k}
  +\sum_{k\notin I} c_k'
    \exp\!\left(-\tfrac{t^2}{2}(\alpha_{\max}-\alpha_k)\right)
    e^{t\beta_k}=0.
\]
For each $k\notin I$, $\alpha_{\max}-\alpha_k>0$, so
\[
\left|c_k'\exp\!\left(-\tfrac{t^2}{2}
(\alpha_{\max}-\alpha_k)\right)e^{t\beta_k}\right|
\leq |c_k'|\exp\!\left(-\tfrac{t^2}{2}(\alpha_{\max}-\alpha_k)
+ t\,\mathrm{Re}(\beta_k)\right) \to 0
\]
as $t\to+\infty$. Hence for all sufficiently large $t$:
\[
  \sum_{k\in I} c_k'\,e^{t\beta_k}=0. \tag{$***$}
\]

\textit{Distinctness of $\{\beta_k\}_{k\in I}$.}
Suppose $\beta_k=\beta_{k'}$ for some $k\neq k'$ in $I$ and for all
$\omega_0\in\mathbb{R}^d$. Then equating real and imaginary parts gives
$\xi^\top(\Sigma_k-\Sigma_{k'})=0$ as a vector and
$\xi^\top(\mu_k-\mu_{k'})=0$.
If this held for all $\xi\in\mathbb{R}^d$, then $\Sigma_k=\Sigma_{k'}$
and $\mu_k=\mu_{k'}$, contradicting distinctness.
Hence for generic $(\xi,\omega_0)$ outside a measure-zero set, the
values $\{\beta_k\}_{k\in I}$ are distinct; fix any such
$(\xi,\omega_0)$.

\textit{Linear independence of $\{e^{t\beta_k}\}_{k\in I}$.}
The Wronskian of $\{e^{t\beta_k}\}_{k\in I}$ equals:
\[
  W(t) = \left(\prod_{k\in I}e^{t\beta_k}\right)
  \prod_{\substack{j,k\in I \\ j<k}}(\beta_k-\beta_j).
\]
Since the $\beta_k$ are distinct, $\prod_{j<k}(\beta_k-\beta_j)\neq 0$,
and $\prod_{k\in I}e^{t\beta_k}\neq 0$ for all $t$, so $W(t)\neq 0$.
Hence $\{e^{t\beta_k}\}_{k\in I}$ are linearly independent over
$\mathbb{C}$.

\textit{Concluding $c_k=0$ for $k\in I$.}
Since $c_k'$ are constants independent of $t$, equation $(***)$ is a
vanishing linear combination of linearly independent functions of $t$,
forcing $c_k'=0$ for all $k\in I$. Since the exponential prefactor in
$c_k'=c_k\exp(i\omega_0^\top\mu_k
-\frac{1}{2}\omega_0^\top\Sigma_k\omega_0)$ is never zero, we conclude
$c_k=0$ for all $k\in I$.

\textit{Applying the inductive hypothesis.}
Since $c_k=0$ for all $k\in I$, those terms vanish identically in
$(*)$, leaving:
\[
  \sum_{k\notin I} c_k
  \exp\!\left(i\omega^\top\mu_k-\tfrac{1}{2}\omega^\top\Sigma_k\omega\right)
  =0
  \qquad\forall\,\omega\in\mathbb{R}^d.
\]
This has exactly $K-|I|\leq K-1$ terms (since $I$ is nonempty) with
pairwise distinct parameters. The inductive hypothesis gives $c_k=0$
for all $k\notin I$, completing the induction.
\end{proof}

\textit{Step 2: recover $A_k, b_k$.}
For a.e.\ $x$, the matched components satisfy
$\mathcal{N}(A_{\sigma(k)}x+b_{\sigma(k)},\Sigma_{\sigma(k)})
=\mathcal{N}(A_k'x+b_k',\Sigma_k')$,
which forces $\Sigma_{\sigma(k)}=\Sigma_k'$ and
$A_{\sigma(k)}x+b_{\sigma(k)}=A_k'x+b_k'$ for a.e.\ $x$.
Since this is an equality of affine functions holding a.e., the
coefficients must match: differentiating with respect to $x$ gives
$A_{\sigma(k)}=A_k'$, and evaluating at $x=0$ gives
$b_{\sigma(k)}=b_k'$.
The permutation $\sigma$ takes values in the finite set of permutations
of $\{1,\dots,K\}$; since it is measurable and constant on a set of
full measure, it is a single global permutation.

\textit{Step 3: recover $\Pi$.}
With emission parameters identified, the two-step marginal density of
$(z_t,z_{t+1},z_{t+2})$ satisfies:
\[
p(z_{t+2},z_{t+1}\mid z_t)
= \sum_{k,k'}\Pi_{kk'}\rho_k(z_t)
  \mathcal{N}(z_{t+1};\mu_k(z_t),\Sigma_k)
  \mathcal{N}(z_{t+2};\mu_{k'}(z_{t+1}),\Sigma_{k'}).
\]
Since the emission parameters are now identified, the Gaussian factors
are known, and the map $\Pi\mapsto p(z_{t+2},z_{t+1}\mid z_t)$ is
linear and injective in $\Pi$ (generically under (A2)). Hence $\Pi$ is
uniquely recovered from the law of three consecutive
observations. $\square$

\paragraph{(ii) MLE consistency.}

Under (A1)--(A4), the Gaussian SDS is an HMM with finite latent state
space $\{1,\dots,K\}$ and continuous emissions. The hypotheses of
Theorem~1 of \citet{Douc_2004} are: ergodicity of the hidden chain
(A1), distinct emission families (A2), compactness and log-likelihood
integrability (A3), and uniform exponential forgetting (A4). All are
satisfied, so their theorem yields
$\hat\theta_T\xrightarrow{\mathrm{a.s.}}\theta^\star$. $\square$

\paragraph{(iii) One-step Bayes misclassification bound.}

\textit{Step 1: pointwise union bound.}
Fix $s_t=i$ and $z_t=z$.
Then $z_{t+1}\mid z_t=z,s_t=i\sim\mathcal{N}(\mu_i(z),\Sigma_i)$
with $\mu_i(z)=A_iz+b_i$.
The Bayes decoder errs iff $\exists\,j\neq i$ such that
$p_j(z_{t+1}\mid z)\geq p_i(z_{t+1}\mid z)$.
By the union bound:
\[
  \mathbb{P}(\hat s_t\neq i\mid s_t=i,z_t=z)
  \;\leq\;
  \sum_{j\neq i}
  \mathbb{P}\!\left(
    \frac{p_j(z_{t+1}\mid z)}{p_i(z_{t+1}\mid z)}\geq 1
    \;\bigg|\;
    z_{t+1}\sim\mathcal{N}(\mu_i(z),\Sigma_i)
  \right).
\]
For each term, since $\mathbf{1}[X\geq 1]\leq X^{1/2}$ pointwise for
$X\geq 0$, Markov's inequality at $s=\frac{1}{2}$ gives:
\[
\begin{aligned}
\mathbb{P}\!\left(\frac{p_j}{p_i}\geq 1 \;\bigg|\;
z_{t+1}\sim p_i(\cdot\mid z)\right)
&\;\leq\;
\mathbb{E}_{p_i}\!\left[\left(\frac{p_j}{p_i}\right)^{1/2}\right] \\
&= \int p_i(z_{t+1}\mid z)^{1/2}\,p_j(z_{t+1}\mid z)^{1/2}\,dz_{t+1} \\
&= e^{-B_{ij}(z)}.
\end{aligned}
\]
Setting $B_{\min}(z)=\min_{i\neq j}B_{ij}(z)$ and summing over the
$K-1$ competitors:
\[
  \mathbb{P}(\hat s_t\neq i\mid s_t=i,z_t=z)
  \;\leq\;(K-1)\,e^{-B_{\min}(z)}.
\]

\textit{Step 2: closed form for $B_{ij}(z)$.}
For $p_i(\cdot\mid z)=\mathcal{N}(\mu_i(z),\Sigma_i)$ and
$p_j(\cdot\mid z)=\mathcal{N}(\mu_j(z),\Sigma_j)$, evaluating
$\int p_i^{1/2}p_j^{1/2}\,dz_{t+1}$ via the standard Gaussian product
formula gives the closed form stated in the proposition.

\textit{Step 3: averaging.}
Averaging jointly over $(s_t,z_t)$ under the stationary measure $\rho$
of the joint chain, the law of total expectation gives:
\[
  \mathbb{P}(\hat s_t\neq s_t)
  =\mathbb{E}_{(s_t,z_t)\sim\rho}
    \!\left[\mathbb{P}(\hat s_t\neq s_t\mid s_t,z_t)\right]
  \;\leq\;
  (K-1)\,\mathbb{E}_{(s_t,z_t)\sim\rho}
    \!\left[e^{-B_{\min}(z_t)}\right].
\]
Defining
$B_{\mathrm{eff}}(f)=-\log\,\mathbb{E}_{(s_t,z_t)\sim\rho}
[e^{-B_{\min}(z_t)}]$,
Jensen's inequality (the exponential is convex) confirms
$B_{\mathrm{eff}}(f)\geq 0$, and the bound follows immediately.
$\square$

\subsection{Persistence Theory and the State-Splitting Cliff}
\label{sec:theory_persistence}

\begin{corollary}[Better representation increases measured persistence]
\label{cor:persistence}
Assume the joint process $(s_t^\star, z_t)$ is stationary and the decoder
is time-homogeneous, so the per-step misclassification probability
$\varepsilon = \mathbb{P}(\hat{s}_t \neq s_t^\star)$ is constant in $t$.
Then:
\[
\mathbb{P}(\hat{s}_{t+1} = \hat{s}_t) \;\geq\; p_{\mathrm{stay}} - 2\varepsilon,
\]
where $p_{\mathrm{stay}} = \mathbb{P}(s_{t+1}^\star = s_t^\star)$.
Combined with Proposition~\ref{prop:ident_mle_bhatta}(iii):
\[
\mathbb{P}(\hat{s}_{t+1} = \hat{s}_t)
\;\geq\;
p_{\mathrm{stay}} - 2(K-1)\exp\!\bigl(-B_{\mathrm{eff}}(f)\bigr),
\]
where $B_{\mathrm{eff}}(f) = -\log\,\mathbb{E}_{(s_t,z_t)\sim\rho}
[\exp(-B_{\min}(z_t))]$ and $B_{\min}(z) = \min_{i\neq j}B_{ij}(z)$.
\end{corollary}

\begin{proof}
Define events $A = \{s_{t+1}^\star = s_t^\star\}$,
$B = \{\hat{s}_t = s_t^\star\}$, $C = \{\hat{s}_{t+1} = s_{t+1}^\star\}$.
On $A \cap B \cap C$ we have $\hat{s}_{t+1} = s_{t+1}^\star = s_t^\star =
\hat{s}_t$, so $\mathbb{P}(\hat{s}_{t+1}=\hat{s}_t) \geq
\mathbb{P}(A\cap B\cap C)$.
By the union bound:
\[
\mathbb{P}(A \cap B \cap C)
= 1 - \mathbb{P}(A^c \cup B^c \cup C^c)
\geq 1 - \mathbb{P}(A^c) - \mathbb{P}(B^c) - \mathbb{P}(C^c)
= \mathbb{P}(A) - \mathbb{P}(B^c) - \mathbb{P}(C^c).
\]
Stationarity gives $(s_{t+1}^\star, z_{t+1}) \sim \rho$, and
time-homogeneity gives the same decoder at $t+1$ as at $t$, so
$\mathbb{P}(C^c) = \mathbb{P}(B^c) = \varepsilon$.
Therefore:
\[
\mathbb{P}(\hat{s}_{t+1}=\hat{s}_t)
\geq \mathbb{P}(A) - 2\varepsilon
= p_{\mathrm{stay}} - 2\varepsilon.
\]
Substituting $\varepsilon \leq (K-1)\exp(-B_{\mathrm{eff}}(f))$ from
Proposition~\ref{prop:ident_mle_bhatta}(iii) gives the second inequality.
\end{proof}

\begin{theorem}[Over-specification admits low-persistence clone solutions]
\label{thm:splitting}
Let the true SDS have $K^\star$ regimes with parameters
$\{(A_k^\star,b_k^\star,\Sigma_k^\star)\}_{k=1}^{K^\star}$ and transition
matrix $\Pi^\star$, initialised at its stationary distribution $\pi^\star$.
For any $K > K^\star$, let $m = K - K^\star + 1 \geq 2$.
There exists a $K$-state SDS with parameters $\theta_K$, initialised at
its own stationary distribution, such that:
\begin{enumerate}
\item $\theta_K$ induces the same marginal law on $\{z_t\}$ as $\theta^\star$.
\item For any clone construction with intra-clone transition matrix $Q \neq I$:
\[
\mathbb{P}(s_{t+1} = s_t)
= \sum_{k \neq r} \pi_k^\star \Pi_{kk}^\star
  + \pi_r^\star \Pi_{rr}^\star \bar\delta
\;<\; p_{\mathrm{stay}},
\]
where $\bar\delta = \sum_\ell \alpha_\ell Q_{\ell\ell} \in (0,1)$ is the
$\alpha$-weighted average diagonal of $Q$.
The persistence drop is exactly:
\[
p_{\mathrm{stay}} - \mathbb{P}(s_{t+1}=s_t)
= \pi_r^\star \Pi_{rr}^\star (1 - \bar\delta) > 0.
\]
\end{enumerate}
\end{theorem}

\begin{proof}
\textit{Construction.}
Fix any regime $r \in \{1,\dots,K^\star\}$.
Choose weights $\alpha_\ell > 0$ with $\sum_{\ell=1}^m \alpha_\ell = 1$.
Replace regime $r$ with $m$ clone states $r_1,\dots,r_m$ sharing emission
parameters:
\[
(A_{r_\ell}, b_{r_\ell}, \Sigma_{r_\ell})
= (A_r^\star, b_r^\star, \Sigma_r^\star)
\qquad \forall\,\ell=1,\dots,m.
\]
Define $\Pi_K$ by:
\begin{align*}
(\Pi_K)_{ij} &= \Pi^\star_{ij}
  && \text{non-clone } i,j \neq r, \\
(\Pi_K)_{i,r_\ell} &= \alpha_\ell \Pi^\star_{ir}
  && \text{non-clone } i \text{ to clone } r_\ell, \\
(\Pi_K)_{r_\ell,j} &= \Pi^\star_{rj}
  && \text{clone } r_\ell \text{ to non-clone } j \neq r, \\
(\Pi_K)_{r_\ell,r_{\ell'}} &= \Pi^\star_{rr} \cdot Q_{\ell\ell'}
  && \text{intra-clone transitions},
\end{align*}
where $Q$ is an $m\times m$ stochastic matrix with stationary distribution
$\alpha$.
Rows sum to one:
$\sum_{j\neq r}\Pi^\star_{rj} + \Pi^\star_{rr}\sum_{\ell'}Q_{\ell\ell'}
= (1-\Pi^\star_{rr}) + \Pi^\star_{rr} = 1$.

\textit{Existence of $Q$ with prescribed $\bar\delta$.}
For any target $\bar\delta \in [\sum_\ell\alpha_\ell^2, 1)$, the convex
combination
\[
Q = (1-\lambda)\,\mathbf{1}\alpha^\top + \lambda\,I,
\qquad
\lambda = \frac{\bar\delta - \sum_\ell\alpha_\ell^2}{1 - \sum_\ell\alpha_\ell^2},
\]
is a stochastic matrix with stationary distribution $\alpha$ and weighted
diagonal $\sum_\ell\alpha_\ell Q_{\ell\ell} = \bar\delta$.

\textit{Stationary distribution of $\Pi_K$.}
We claim $\pi_{r_\ell} = \alpha_\ell\pi_r^\star$ (with $\pi_k = \pi_k^\star$
for $k\neq r$) is stationary for $\Pi_K$.
For non-clone $j \neq r$:
\[
\sum_i \pi_i(\Pi_K)_{ij}
= \sum_{k\neq r}\pi_k^\star\Pi^\star_{kj}
  + \sum_\ell \alpha_\ell\pi_r^\star\Pi^\star_{rj}
= \sum_{k}\pi_k^\star\Pi^\star_{kj}
= \pi_j^\star. 
\]
For clone $r_\ell$:
\begin{align*}
\sum_i \pi_i(\Pi_K)_{i,r_\ell}
&= \sum_{k\neq r}\pi_k^\star\alpha_\ell\Pi^\star_{kr}
   + \sum_{\ell'}\alpha_{\ell'}\pi_r^\star\Pi^\star_{rr}Q_{\ell'\ell} \\
&= \alpha_\ell\!\sum_{k\neq r}\pi_k^\star\Pi^\star_{kr}
   + \pi_r^\star\Pi^\star_{rr}\!\underbrace{\sum_{\ell'}\alpha_{\ell'}
     Q_{\ell'\ell}}_{=\,\alpha_\ell} \\
&= \alpha_\ell(\pi_r^\star - \pi_r^\star\Pi^\star_{rr})
   + \alpha_\ell\pi_r^\star\Pi^\star_{rr}
\;=\; \alpha_\ell\pi_r^\star. 
\end{align*}

\textit{Proof of (i).}
Since all clone states share emission parameters $(A_r^\star,b_r^\star,
\Sigma_r^\star)$, the marginal transition kernel satisfies:
\[
p_{\theta_K}(z_{t+1}\mid z_t)
= \sum_{k\neq r}\rho_k(z_t)\mathcal{N}(z_{t+1};\mu_k(z_t),\Sigma_k)
  + \rho_r(z_t)\mathcal{N}(z_{t+1};\mu_r^\star(z_t),\Sigma_r^\star),
\]
where $\rho_r(z_t) = \sum_\ell\mathbb{P}(s_t=r_\ell\mid z_t)$ aggregates
the clone weights.
Since $\pi_{r_\ell} = \alpha_\ell\pi_r^\star$ and the aggregated transition
probabilities satisfy $\sum_\ell(\Pi_K)_{i,r_\ell} = \Pi^\star_{ir}$ and
$\sum_\ell(\Pi_K)_{r_\ell,j} = \Pi^\star_{rj}$, the weight $\rho_r(z_t)$
is identical under $\theta_K$ and $\theta^\star$.
Hence $p_{\theta_K}(z_{t+1}\mid z_t) = p_{\theta^\star}(z_{t+1}\mid z_t)$
and the marginal laws of $\{z_t\}$ coincide. 

\textit{Proof of (ii).}
Under the stationary distribution $\pi_{r_\ell} = \alpha_\ell\pi_r^\star$:
\[
\mathbb{P}(s_{t+1}=s_t)
= \sum_{k\neq r}\pi_k^\star\Pi_{kk}^\star
  + \sum_\ell\pi_{r_\ell}\cdot\Pi^\star_{rr}\cdot Q_{\ell\ell}
= \sum_{k\neq r}\pi_k^\star\Pi_{kk}^\star
  + \pi_r^\star\Pi^\star_{rr}
    \underbrace{\sum_\ell\alpha_\ell Q_{\ell\ell}}_{=\,\bar\delta}.
\]
Since $Q\neq I$ and all $\alpha_\ell > 0$, at least one $Q_{\ell\ell} < 1$,
giving $\bar\delta < 1$ and therefore:
\[
p_{\mathrm{stay}} - \mathbb{P}(s_{t+1}=s_t)
= \pi_r^\star\Pi^\star_{rr}(1-\bar\delta) > 0. \qquad\square
\]
\end{proof}

\begin{remark}[Non-identifiability of persistence]
The clone construction is observationally equivalent to the true model by
part~(i): no likelihood-based criterion can distinguish $\theta_K$ from
$\theta^\star$.
By part~(ii), any such solution with $Q\neq I$ has strictly lower
persistence, with the drop $\pi_r^\star\Pi^\star_{rr}(1-\bar\delta)$
controlled by how much the intra-clone block mixes.
Since $Q = I$ is a measure-zero boundary case that EM will not reach from
a generic initialization, cloning generically induces a strict and
undetectable persistence drop.
This establishes that persistence is not identified from the observations
when $K > K^\star$: the set of likelihood-equivalent solutions spans
persistence values from $p_{\mathrm{stay}}$ down to
$\sum_{k\neq r}\pi_k^\star\Pi_{kk}^\star$.
\end{remark}

\begin{remark}[Sharp cliff at $K = K^\star + 1$]
For $K \leq K^\star$, assumption~(A2) holds generically, the identifiability
result of Proposition~\ref{prop:ident_mle_bhatta}(i) applies, and
persistence is identified up to label permutation.
At $K = K^\star + 1$, clone solutions first become available, and by
Theorem~\ref{thm:splitting}(ii) these solutions have strictly lower
persistence for any $Q \neq I$.
Since no clone solutions exist for $K \leq K^\star$ under~(A2), the
availability of unidentifiable low-persistence solutions is a property
that appears exactly at $K = K^\star + 1$.
This gives a theoretical account of the sharp empirical cliff observed. Below the cliff, persistence is identified
and stable; at the cliff, the likelihood landscape admits degenerate
solutions that EM may select with no likelihood penalty.
\end{remark}






\subsection{How RLVR May Induce Persistent States}
\label{sec:theory_rlvr}

\begin{proposition}[Trajectory-level rewards induce persistent latent options]
\label{prop:rlvr}
Consider a meta-controller with softmax policy over $K$ options:
\[
T_\beta(i \to i \mid x_t)
= \frac{e^{V_i(x_t)/\beta}}{e^{V_i(x_t)/\beta}
  + \sum_{j \neq i} e^{V_j(x_t)/\beta}},
\]
where $x_t$ is the environment (or latent) state and $V_k(x_t)$ denotes
the expected return of following option $k$ from $x_t$ under the current
policy, and $\beta>0$ is a temperature.
Define the continuation margin
$\gamma_i(x_t) := V_i(x_t) - \max_{j\neq i}V_j(x_t)$.
Suppose that over the stationary distribution $\rho$ (assuming the
policy induces an ergodic Markov chain):
\[
\mathbb{P}_{x_t\sim\rho}(\gamma_i(x_t)\geq\gamma)
\;\geq\; 1-\delta
\]
for some $\gamma>0$ and $\delta\in[0,1)$. Then:
\[
\mathbb{E}_{x_t\sim\rho}[T_\beta(i\to i\mid x_t)]
\;\geq\;
\frac{1-\delta}{1+(K-1)e^{-\gamma/\beta}}.
\]
Notably, this bound has the same functional form as the
Bhattacharyya-based misclassification bound in
Proposition~\ref{prop:risk_bound}, with $\gamma/\beta$ playing the
role of a separation margin.
\end{proposition}

\begin{proof}
Define the high-advantage event $A:=\{x_t:\gamma_i(x_t)\geq\gamma\}$,
so $\mathbb{P}_\rho(A)\geq 1-\delta$ by assumption.
On $A$, for all $j\neq i$:
$e^{V_j(x_t)/\beta}\leq e^{-\gamma/\beta}e^{V_i(x_t)/\beta}$,
giving $T_\beta(i\to i\mid x_t)\geq\frac{1}{1+(K-1)e^{-\gamma/\beta}}$.
Since $T_\beta\in[0,1]$:
\[
\begin{aligned}
\mathbb{E}_\rho[T_\beta(i\to i\mid x_t)]
&\;\geq\; \mathbb{E}_\rho[T_\beta(i\to i\mid x_t)\cdot\mathbf{1}_A] \\
&\;\geq\; \mathbb{P}_\rho(A)\cdot\frac{1}{1+(K-1)e^{-\gamma/\beta}} \\
&\;\geq\; \frac{1-\delta}{1+(K-1)e^{-\gamma/\beta}}.
\qquad\square
\end{aligned}
\]
\end{proof}

\begin{remark}[RLVR vs.\ pretraining]
Pretrained language models already exhibit latent switching structure
(as captured by the SDS models in this work), but this structure is not
explicitly reinforced. Maximum likelihood training optimizes next-token
prediction and does not directly encourage consistency of latent modes
across time: there is no pressure to make the continuation margin
$\gamma_i(x_t)$ positive.

In contrast, trajectory-level reinforcement learning induces an
option-level value function $V_k(x_t)$, which assigns higher return
to coherent multi-step strategies. This creates a positive continuation
margin $\gamma_i(x_t)>0$ over large regions of state space, biasing
the policy toward persistence. Thus RLVR does not create latent
structure from scratch, but amplifies and stabilizes existing latent
modes by making persistence directly reward-relevant.
\end{remark}



\section{Training Dynamics of Latent Policy Structure}
\label{sec:training_dynamics}
 
To track how latent policy structure develops during fine-tuning, we fine-tune Llama-3.1-8B on 100K reasoning traces from the DeepSeek-R1 model on the AM-DeepSeek-R1 dataset \citep{zhao202514millionopensourcedistilled} using LoRA with rank 64, $\alpha{=}128$, applied to
all linear layers, with a learning rate of $10^{-4}$ and cosine scheduling, saving intermediate checkpoints at 25\%, 50\%, and 75\% of training. We fit the SDS framework at layer~31 and evaluate on GSM8K and MATH-500.
 
\begin{figure}[t]
    \centering
    \includegraphics[width=\textwidth]{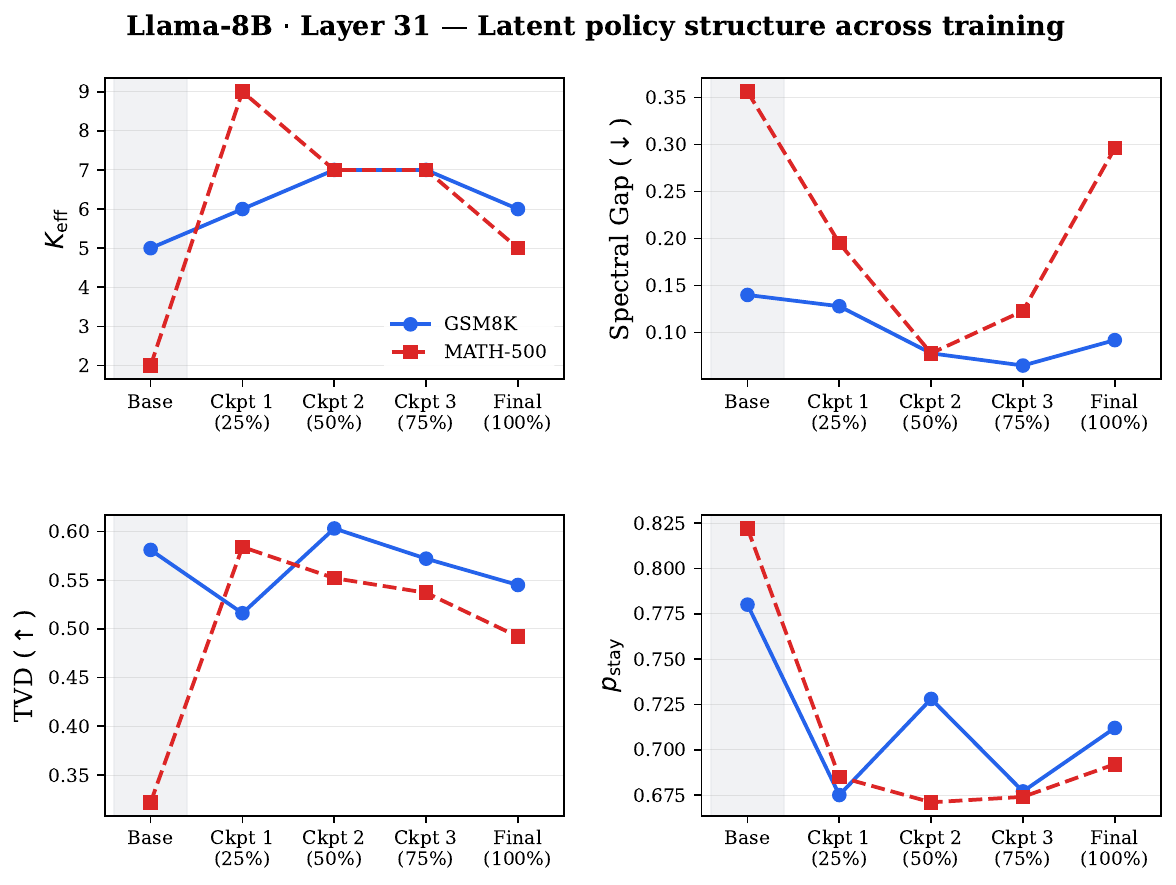}
    \caption{Latent policy structure across training for Llama-8B (layer~31).
    Grey band: base model. Training increases the effective number of latent
    states and reduces the spectral gap; both effects peak at mid-training
    before partially consolidating.}
    \label{fig:training_dynamics}
\end{figure}
 
Figure~\ref{fig:training_dynamics} shows that structured latent dynamics emerge early in training and are not a pre-existing property of the base model. Within the first 25\% of training, $K_{\mathrm{eff}}$ increases substantially on both benchmarks (5$\to$6 on GSM8K; 2$\to$9 on MATH-500), and the spectral gap begins to drop, indicating the formation of more persistent regime dynamics. The most structured configuration appears at mid-training (Ckpt~2--3), where the spectral gap reaches its minimum (0.065 on GSM8K, 0.078 on MATH-500) and TVD peaks (0.603 on GSM8K, 0.584 on MATH-500), reflecting maximally differentiated transition structure between regimes. By the final checkpoint, the metrics partially relax, consistent with consolidation into a stable regime configuration. We note that $p_{\mathrm{stay}}$ decreases over this LoRA training run, suggesting that full recovery of self-transition persistence may require longer training or full fine-tuning. Overall, these results indicate that latent policy structure is not inherited from pretraining but is progressively shaped by reasoning fine-tuning, with the most pronounced regime differentiation emerging at mid-training before stabilizing.

\section{Additional Robustness and Modeling Choice Ablations}
\label{sec:more_expms}

\subsection{Per-Dataset Structural Comparisons}
\label{sec:more_mods}

To evaluate whether the observed structural differences are robust across reasoning tasks, Table~\ref{tab:per_dataset_results} reports per-dataset switching metrics for all model families. The qualitative trends from Section~\ref{sec:expm_val} remain stable across datasets. Reasoning models consistently exhibit higher TVD, while effective state
utilization, persistence, and spectral properties vary across model
families and datasets.

\begin{table*}[t]
\centering
\scriptsize
\setlength{\tabcolsep}{3pt}
\renewcommand{\arraystretch}{1.12}

\resizebox{\textwidth}{!}{%
\begin{tabular}{llcccccccc}
\toprule
Model & Dataset
& \multicolumn{2}{c}{$K_{\mathrm{eff}}\uparrow$}
& \multicolumn{2}{c}{$p_{\mathrm{stay}}\uparrow$}
& \multicolumn{2}{c}{TVD$\uparrow$}
& \multicolumn{2}{c}{Spec.\ Gap$\downarrow$} \\
\cmidrule(lr){3-4}
\cmidrule(lr){5-6}
\cmidrule(lr){7-8}
\cmidrule(lr){9-10}
& & Base & Reason. & Base & Reason. & Base & Reason. & Base & Reason. \\
\midrule

\multirow{4}{1.8cm}{Llama-8B}
& GSM8K
& \best{4.80}{\pm 0.75} & \val{4.40}{\pm 0.49}
& \val{0.71}{\pm 0.04} & \best{0.80}{\pm 0.03}
& \val{0.51}{\pm 0.05} & \best{0.57}{\pm 0.02}
& \val{0.18}{\pm 0.05} & \best{0.11}{\pm 0.01} \\

& MATH-500
& \val{3.00}{\pm 0.63} & \best{5.20}{\pm 0.98}
& \val{0.76}{\pm 0.04} & \best{0.77}{\pm 0.05}
& \val{0.41}{\pm 0.08} & \best{0.57}{\pm 0.01}
& \val{0.26}{\pm 0.11} & \best{0.16}{\pm 0.03} \\

& SVAMP
& \best{5.40}{\pm 1.02} & \val{4.40}{\pm 0.49}
& \val{0.63}{\pm 0.08} & \best{0.84}{\pm 0.02}
& \val{0.47}{\pm 0.04} & \best{0.61}{\pm 0.02}
& \val{0.22}{\pm 0.05} & \best{0.08}{\pm 0.01} \\

& MMLU-Pro
& \best{5.20}{\pm 1.17} & \val{4.60}{\pm 0.49}
& \val{0.70}{\pm 0.03} & \best{0.86}{\pm 0.03}
& \val{0.53}{\pm 0.04} & \best{0.64}{\pm 0.01}
& \val{0.13}{\pm 0.06} & \best{0.04}{\pm 0.00} \\

\midrule

\multirow{4}{1.8cm}{Qwen-14B}
& GSM8K
& \val{3.60}{\pm 1.02} & \best{6.60}{\pm 1.02}
& \best{0.81}{\pm 0.09} & \val{0.75}{\pm 0.03}
& \val{0.52}{\pm 0.06} & \best{0.62}{\pm 0.03}
& \best{0.07}{\pm 0.06} & \val{0.08}{\pm 0.03} \\

& MATH-500
& \best{3.00}{\pm 0.00} & \best{3.00}{\pm 0.00}
& \val{0.72}{\pm 0.02} & \best{0.75}{\pm 0.04}
& \val{0.39}{\pm 0.02} & \best{0.42}{\pm 0.04}
& \val{0.28}{\pm 0.01} & \best{0.27}{\pm 0.03} \\

& SVAMP
& \val{2.40}{\pm 0.49} & \best{7.40}{\pm 0.49}
& \best{0.81}{\pm 0.06} & \val{0.75}{\pm 0.03}
& \val{0.37}{\pm 0.05} & \best{0.64}{\pm 0.00}
& \val{0.30}{\pm 0.06} & \best{0.06}{\pm 0.02} \\

& MMLU-Pro
& \val{3.00}{\pm 0.00} & \best{6.20}{\pm 1.17}
& \val{0.77}{\pm 0.01} & \best{0.82}{\pm 0.04}
& \val{0.45}{\pm 0.01} & \best{0.66}{\pm 0.01}
& \best{0.03}{\pm 0.01} & \val{0.05}{\pm 0.01} \\

\midrule

\multirow{4}{1.8cm}{Qwen-1.5B}
& GSM8K
& \val{3.40}{\pm 0.49} & \best{4.80}{\pm 0.75}
& \val{0.67}{\pm 0.09} & \best{0.82}{\pm 0.06}
& \val{0.42}{\pm 0.03} & \best{0.61}{\pm 0.03}
& \val{0.24}{\pm 0.02} & \best{0.06}{\pm 0.02} \\

& MATH-500
& \val{2.00}{\pm 0.00} & \best{5.20}{\pm 0.75}
& \best{0.88}{\pm 0.01} & \val{0.75}{\pm 0.05}
& \val{0.38}{\pm 0.01} & \best{0.57}{\pm 0.05}
& \val{0.24}{\pm 0.01} & \best{0.08}{\pm 0.02} \\

& SVAMP
& \val{3.00}{\pm 0.00} & \best{4.20}{\pm 0.40}
& \val{0.69}{\pm 0.01} & \best{0.85}{\pm 0.01}
& \val{0.37}{\pm 0.01} & \best{0.61}{\pm 0.01}
& \val{0.28}{\pm 0.04} & \best{0.05}{\pm 0.01} \\

& MMLU-Pro
& \val{2.00}{\pm 0.00} & \best{7.00}{\pm 0.89}
& \best{0.89}{\pm 0.10} & \val{0.76}{\pm 0.03}
& \val{0.39}{\pm 0.10} & \best{0.64}{\pm 0.02}
& \val{0.22}{\pm 0.19} & \best{0.07}{\pm 0.01} \\

\midrule

\multirow{4}{1.8cm}{\shortstack{Qwen2.5/\\QwQ-32B}}
& GSM8K
& \val{5.62}{\pm 0.86} & \best{6.00}{\pm 1.00}
& \val{0.71}{\pm 0.08} & \best{0.75}{\pm 0.04}
& \val{0.54}{\pm 0.10} & \best{0.59}{\pm 0.05}
& \val{0.14}{\pm 0.02} & \best{0.10}{\pm 0.02} \\

& MATH-500
& \val{3.50}{\pm 0.50} & \best{4.25}{\pm 0.66}
& \val{0.78}{\pm 0.09} & \best{0.80}{\pm 0.02}
& \val{0.50}{\pm 0.05} & \best{0.56}{\pm 0.04}
& \val{0.17}{\pm 0.05} & \best{0.14}{\pm 0.04} \\

& SVAMP
& \val{5.62}{\pm 0.48} & \best{5.75}{\pm 0.83}
& \val{0.72}{\pm 0.08} & \best{0.79}{\pm 0.04}
& \val{0.55}{\pm 0.09} & \best{0.62}{\pm 0.04}
& \val{0.14}{\pm 0.06} & \best{0.09}{\pm 0.02} \\

& MMLU-Pro
& \val{6.25}{\pm 0.83} & \best{6.38}{\pm 0.48}
& \val{0.74}{\pm 0.05} & \best{0.80}{\pm 0.05}
& \val{0.60}{\pm 0.05} & \best{0.65}{\pm 0.04}
& \val{0.10}{\pm 0.07} & \best{0.09}{\pm 0.05} \\

\bottomrule
\end{tabular}%
}

\caption{
Per-dataset comparison of structural switching metrics between base and
reasoning-fine-tuned models. Values report means with population standard
deviations shown in gray across repeated runs. Results use layer 31 for
Llama-8B, layer 47 for Qwen-14B, and layer 27 for Qwen-1.5B. Light-green
cells and bold means indicate the numerically better value within each
base--reasoning pair; ties are highlighted in both columns. Arrows indicate
the preferred direction for each metric.
}
\label{tab:per_dataset_results}
\end{table*}

\subsection{Controls for Correctness and Temporal Bias}
\label{app:confounding_controls}

The comparison between base and reasoning-fine-tuned models could
potentially be affected by differences in output correctness or by
temporal biases introduced by the analysis pipeline. We therefore
perform four controls. First, we restrict the analysis to problems
solved correctly by both models. Second, we destroy the temporal order
of each trajectory before fitting the SDS. Third, we remove CEBRA's
temporal-adjacency objective by sampling random non-adjacent positive
pairs. Finally, we vary and remove the Dirichlet persistence prior used
during SDS estimation.

Together, these controls distinguish structure present in the model's
activation trajectories from structure that could arise from output
correctness, contrastive pairing, or the SDS prior.

\subsubsection{Paired-Correct Traces}
\label{app:paired_correct}

Reasoning-fine-tuned models generally solve more problems correctly
than their corresponding base models. The observed dynamical
differences could therefore reflect a difference between correct and
incorrect traces rather than an effect of reasoning fine-tuning. To
control for this possibility, we repeat the analysis using only
paired-correct examples: problems solved correctly by both the base and
reasoning-fine-tuned models.

We focus this analysis on GSM8K because it is the only benchmark with a
sufficiently large paired-correct subset for reliable SDS estimation.
The paired-correct subsets contain 1,065 examples for Qwen-1.5B and
1,116 examples for Qwen-14B. The corresponding subsets are considerably
smaller on MATH-500, SVAMP, and MMLU-Pro: 174, 145, and 177 examples,
respectively, for Qwen-1.5B, and 359, 220, and 367 examples for
Qwen-14B.

We report two complementary measures of temporal stability. Mean
self-transition, $p_{\mathrm{stay}}$, is the average diagonal entry of
the learned transition matrix. Segment persistence is the empirical
mean duration of a contiguous latent-state segment, measured in
sentence-level reasoning steps.

\begin{table}[t]
\centering
\small
\setlength{\tabcolsep}{5pt}
\renewcommand{\arraystretch}{1.12}

\begin{tabular}{llrccc}
\toprule
Model & Metric & $N$
& \cellcolor{basecell}Base
& \cellcolor{reasoncell}Reasoning
& $\Delta$ \\
\midrule

\multirow{2}{*}{Qwen-1.5B}
& Segment persistence
& \multirow{2}{*}{1,065}
& \basev{2.18}
& \reasonv{4.73}
& $\mathbf{+2.55}$ \\

& $p_{\mathrm{stay}}$
&
& \basev{0.56}
& \reasonv{0.69}
& $\mathbf{+0.13}$ \\

\midrule

\multirow{2}{*}{Qwen-14B}
& Segment persistence
& \multirow{2}{*}{1,116}
& \basev{2.86}
& \reasonv{8.65}
& $\mathbf{+5.79}$ \\

& $p_{\mathrm{stay}}$
&
& \basev{0.71}
& \reasonv{0.90}
& $\mathbf{+0.19}$ \\

\bottomrule
\end{tabular}

\caption{
SDS results on paired-correct GSM8K traces, restricted to problems
solved correctly by both the base and reasoning-fine-tuned models.
Reasoning fine-tuning increases both segment persistence and mean
self-transition probability despite both checkpoints effectively
utilizing the full fixed state budget, $K_{\mathrm{eff}}=5$.
$\Delta$ denotes Reasoning minus Base.
}
\label{tab:paired_correct}
\end{table}

As shown in Table~\ref{tab:paired_correct}, the qualitative separation
between base and reasoning-fine-tuned models remains after controlling
for correctness. For Qwen-14B, segment persistence increases from
$2.86$ to $8.65$ reasoning steps and $p_{\mathrm{stay}}$ increases from
$0.71$ to $0.90$. For Qwen-1.5B, segment persistence increases from
$2.18$ to $4.73$ and $p_{\mathrm{stay}}$ increases from $0.56$ to
$0.69$.

Both base and reasoning models effectively utilize the full
$K=5$ state budget in this experiment. The difference therefore does
not arise from the reasoning models using more latent states. Instead,
the same state budget is organized differently: base models switch
more frequently, whereas reasoning-fine-tuned models maintain states
for longer contiguous periods. Thus, the observed dynamical difference
cannot be explained solely by the higher accuracy of the
reasoning-fine-tuned checkpoints.

\subsubsection{Time-Shuffled Trajectories}
\label{app:time_shuffle}

The recovered structure could also be influenced by static activation
geometry rather than the temporal organization of reasoning. To
separate these factors, we randomly permute the order of sentence-level
activations within each trajectory before fitting the SDS, while
keeping the underlying CEBRA embeddings fixed. This preserves the set
of activation states in each trajectory while destroying their true
temporal order.

Table~\ref{tab:time_shuffle} compares results obtained using the real
trajectory order (Real) with those obtained after shuffling sentence
order (Shuf.). Here, $\Delta R^2$ denotes the gain in one-step
predictive fit relative to a single linear autoregressive model.

\begin{table*}[t]
\centering
\small
\setlength{\tabcolsep}{7pt}
\renewcommand{\arraystretch}{1.10}

\begin{tabular}{lllcccc}
\toprule
& & &
\multicolumn{2}{c}{Segment persistence $\uparrow$}
& \multicolumn{2}{c}{Predictive gain $\Delta R^2 \uparrow$} \\
\cmidrule(lr){4-5}
\cmidrule(lr){6-7}
Dataset & Model & Checkpoint
& \cellcolor{realcell}Real
& \cellcolor{shufcell}Shuffled
& \cellcolor{realcell}Real
& \cellcolor{shufcell}Shuffled \\
\midrule

\multirow{6}{*}{GSM8K}
& \multirow{2}{*}{Qwen-1.5B}
& Base      & \realv{2.16} & \shufv{1.32}
            & \realv{+0.062} & \shufv{-0.437} \\
& & Reasoning & \realv{6.12} & \shufv{2.14}
              & \realv{+0.049} & \shufv{-0.250} \\
\cmidrule(lr){2-7}

& \multirow{2}{*}{Qwen-14B}
& Base      & \realv{3.20} & \shufv{1.96}
            & \realv{+0.075} & \shufv{-0.341} \\
& & Reasoning & \realv{7.98} & \shufv{3.04}
              & \realv{+0.061} & \shufv{-0.349} \\
\cmidrule(lr){2-7}

& \multirow{2}{*}{Llama-8B}
& Base      & \realv{6.48} & \shufv{4.55}
            & \realv{+0.088} & \shufv{-0.165} \\
& & Reasoning & \realv{7.82} & \shufv{3.05}
              & \realv{+0.039} & \shufv{-0.166} \\

\midrule

\multirow{6}{*}{MATH-500}
& \multirow{2}{*}{Qwen-1.5B}
& Base      & \realv{2.42} & \shufv{2.12}
            & \realv{+0.072} & \shufv{-0.225} \\
& & Reasoning & \realv{6.05} & \shufv{2.59}
              & \realv{+0.035} & \shufv{-0.229} \\
\cmidrule(lr){2-7}

& \multirow{2}{*}{Qwen-14B}
& Base      & \realv{2.82} & \shufv{1.87}
            & \realv{+0.068} & \shufv{-0.217} \\
& & Reasoning & \realv{3.45} & \shufv{1.84}
              & \realv{+0.064} & \shufv{-0.220} \\
\cmidrule(lr){2-7}

& \multirow{2}{*}{Llama-8B}
& Base      & \realv{3.91} & \shufv{3.71}
            & \realv{+0.116} & \shufv{-0.047} \\
& & Reasoning & \realv{6.79} & \shufv{2.14}
              & \realv{+0.042} & \shufv{-0.233} \\

\bottomrule
\end{tabular}

\caption{
Effect of shuffling sentence order before SDS fitting.
\colorbox{realcell}{\strut Real} denotes the original trajectory order,
whereas \colorbox{shufcell}{\strut Shuffled} denotes a random
within-trajectory permutation of the same sentence-level embeddings.
Shuffling reduces segment persistence and causes predictive gains over
a linear autoregressive baseline to become negative in every setting,
showing that the recovered regimes depend on coherent temporal
organization.
}
\label{tab:time_shuffle}
\end{table*}

Shuffling has a substantial effect across all model families. On
GSM8K, the reasoning--base persistence gap decreases from $3.96$ to
$0.82$ for Qwen-1.5B and from $4.79$ to $1.08$ for Qwen-14B. For
Llama-8B, it changes from $1.34$ to $-1.50$. On MATH-500, the gap
decreases from $3.63$ to $0.48$ for Qwen-1.5B, from $0.63$ to $-0.03$
for Qwen-14B, and from $2.88$ to $-1.57$ for Llama-8B.

The same pattern holds on the remaining benchmarks. On SVAMP, the
reasoning--base persistence gap decreases from $4.65$ to $1.02$ for
Qwen-1.5B, from $4.19$ to $-0.70$ for Qwen-14B, and from $1.29$ to
$-1.79$ for Llama-8B. On MMLU-Pro, it decreases from $10.60$ to $3.83$
for Qwen-1.5B, from $3.66$ to $-0.98$ for Qwen-14B, and from $5.79$ to
$3.84$ for Llama-8B. The reasoning--base persistence gap is therefore
larger under the true trajectory order in all 12 model--dataset pairs.

Moreover, $\Delta R^2$ changes from positive to negative after
shuffling for both base and reasoning-fine-tuned models. Once temporal
order is destroyed, the fitted switching dynamics perform worse than
the single linear autoregressive baseline. These results show that the
recovered latent-policy structure depends on coherent temporal
organization and cannot be explained by static activation geometry
alone.

\subsubsection{Randomized CEBRA Positive Pairs}
\label{app:random_pairs}

CEBRA is trained by treating temporally adjacent reasoning steps as
positive pairs. This objective could itself induce apparent temporal
persistence, even if the underlying model activations do not contain
differentiated temporal structure. We test this possibility by
retraining the CEBRA encoder with positive pairs drawn from randomly
selected non-adjacent steps within the same trajectory. This removes
CEBRA's temporal-adjacency objective while retaining trajectory
identity. The SDS is subsequently fitted using the true sentence order.

Table~\ref{tab:random_cebra_pairs} compares the default
adjacent-positive objective (Adj.) with randomized non-adjacent
positives (Rnd.). We use a fixed $K=5$ and report averages over three
seeds.

\begin{table*}[t]
\centering
\small
\setlength{\tabcolsep}{4pt}
\renewcommand{\arraystretch}{1.15}

\begin{tabular*}{\textwidth}{
    @{\extracolsep{\fill}}
    llcccccccc
    @{}
}
\toprule
& &
\multicolumn{4}{c}{$p_{\mathrm{stay}}\uparrow$}
& \multicolumn{4}{c}{TVD$\uparrow$} \\
\cmidrule(lr){3-6}
\cmidrule(lr){7-10}

Model & Dataset
& \multicolumn{2}{c}{Base}
& \multicolumn{2}{c}{Reasoning}
& \multicolumn{2}{c}{Base}
& \multicolumn{2}{c}{Reasoning} \\
\cmidrule(lr){3-4}
\cmidrule(lr){5-6}
\cmidrule(lr){7-8}
\cmidrule(lr){9-10}

& &
\cellcolor{adjcell}Adj.
& \cellcolor{rndcell}Rnd.
& \cellcolor{adjcell}Adj.
& \cellcolor{rndcell}Rnd.
& \cellcolor{adjcell}Adj.
& \cellcolor{rndcell}Rnd.
& \cellcolor{adjcell}Adj.
& \cellcolor{rndcell}Rnd. \\
\midrule

Qwen-1.5B & GSM8K
& \adjv{0.57} & \rndv{0.44}
& \adjv{0.83} & \rndbest{0.82}
& \adjv{0.51} & \rndv{0.45}
& \adjv{0.63} & \rndbest{0.63} \\

Qwen-14B & GSM8K
& \adjv{0.74} & \rndv{0.50}
& \adjv{0.82} & \rndbest{0.76}
& \adjv{0.60} & \rndv{0.48}
& \adjv{0.62} & \rndbest{0.56} \\

Llama-8B & GSM8K
& \adjv{0.64} & \rndv{0.58}
& \adjv{0.77} & \rndbest{0.74}
& \adjv{0.49} & \rndv{0.42}
& \adjv{0.58} & \rndbest{0.54} \\

\midrule

Qwen-1.5B & MATH-500
& \adjv{0.59} & \rndv{0.60}
& \adjv{0.76} & \rndbest{0.79}
& \adjv{0.47} & \rndv{0.45}
& \adjv{0.59} & \rndbest{0.61} \\

Qwen-14B & MATH-500
& \adjv{0.66} & \rndv{0.65}
& \adjv{0.73} & \rndbest{0.66}
& \adjv{0.49} & \rndv{0.48}
& \adjv{0.55} & \rndbest{0.50} \\

Llama-8B & MATH-500
& \adjv{0.78} & \rndv{0.67}
& \adjv{0.83} & \rndbest{0.79}
& \adjv{0.58} & \rndv{0.47}
& \adjv{0.63} & \rndbest{0.59} \\

\midrule

Qwen-1.5B & SVAMP
& \adjv{0.59} & \rndv{0.42}
& \adjv{0.84} & \rndbest{0.85}
& \adjv{0.45} & \rndv{0.40}
& \adjv{0.64} & \rndbest{0.65} \\

Qwen-14B & SVAMP
& \adjv{0.54} & \rndv{0.41}
& \adjv{0.84} & \rndbest{0.78}
& \adjv{0.45} & \rndv{0.37}
& \adjv{0.64} & \rndbest{0.59} \\

Llama-8B & SVAMP
& \adjv{0.68} & \rndv{0.60}
& \adjv{0.77} & \rndbest{0.79}
& \adjv{0.52} & \rndv{0.42}
& \adjv{0.57} & \rndbest{0.60} \\

\midrule

Qwen-1.5B & MMLU-Pro
& \adjv{0.67} & \rndv{0.74}
& \adjv{0.85} & \rndbest{0.85}
& \adjv{0.48} & \rndv{0.54}
& \adjv{0.65} & \rndbest{0.67} \\

Qwen-14B & MMLU-Pro
& \adjv{0.71} & \rndv{0.71}
& \adjv{0.85} & \rndbest{0.83}
& \adjv{0.55} & \rndv{0.53}
& \adjv{0.65} & \rndbest{0.64} \\

Llama-8B & MMLU-Pro
& \adjv{0.76} & \rndv{0.68}
& \adjv{0.84} & \rndbest{0.76}
& \adjv{0.57} & \rndv{0.49}
& \adjv{0.64} & \rndbest{0.56} \\

\bottomrule
\end{tabular*}

\caption{
Effect of replacing temporally adjacent CEBRA positive pairs
(\colorbox{adjcell}{\strut Adj.}) with randomly selected non-adjacent
steps from the same trajectory (\colorbox{rndcell}{\strut Rnd.}).
Results use $K=5$ and are averaged over three seeds. Bold randomized-pair
values indicate the better value between the base and
reasoning-fine-tuned models. Under randomized positives,
reasoning-fine-tuned models retain higher $p_{\mathrm{stay}}$ and TVD
in all 12 model--dataset settings.
}
\label{tab:random_cebra_pairs}
\end{table*}

Under randomized positive pairs, the reasoning-fine-tuned model
continues to exceed its corresponding base model in both
$p_{\mathrm{stay}}$ and TVD in all 12 model--dataset settings. The
reasoning--base separation becomes larger under randomization in 9 of
12 settings for $p_{\mathrm{stay}}$ and in 10 of 12 settings for TVD.
This occurs because removing temporal adjacency generally weakens the
apparent stickiness of the base-model representations more strongly
than that of the reasoning-model representations.

This result complements the PCA+SLDS control in
Appendix~\ref{app:em_moe_cebra_ablation}, which uses no contrastive
pairing and still recovers the qualitative base--reasoning separation.
The recovered latent-policy structure is therefore not an artifact of
CEBRA's temporal-adjacency objective.

\subsubsection{Sensitivity to the Dirichlet Persistence Prior}
\label{app:dirichlet_prior}

The SDS estimation procedure uses a Dirichlet prior that encourages
self-transitions. To determine whether this prior artificially produces
persistent latent states, we vary its strength over
$\kappa\in\{0,0.1,1,5\}$. The setting $\kappa=0$ removes the
persistence prior entirely.

\begin{table}[t]
    \centering
    \small
    \begin{tabular}{lc}
        \toprule
        Metric
        & Maximum change across
          $\kappa\in\{0,0.1,1,5\}$ \\
        \midrule
        $K_{\mathrm{eff}}$ & 0.10 \\
        TVD & 0.0093 \\
        Spectral gap & 0.0092 \\
        \bottomrule
    \end{tabular}
    \caption{
        Maximum variation in the recovered SDS metrics as the
        Dirichlet persistence-prior strength is varied, including its
        complete removal at $\kappa=0$.
    }
    \label{tab:dirichlet_sensitivity}
\end{table}

Across all 24 base/reasoning model--dataset combinations, the
qualitative separation between base and reasoning-fine-tuned models
remains unchanged at every value of $\kappa$. Quantitatively, the
largest observed variation is only $0.10$ effective states in
$K_{\mathrm{eff}}$, $0.0093$ in TVD, and $0.0092$ in spectral gap. For
most runs, the variation in TVD and spectral gap is below $10^{-3}$.
Thus, the learned latent-policy organization is determined primarily
by the activation trajectories rather than by the persistence prior.

\subsection{Activation Extraction and Layer Selection}
\label{sec:layer_selection}

\textbf{Sentence-level representations.} We extract activations at the sentence level 
rather than the token level. Token-level representations are too fine-grained to 
reflect coherent reasoning patterns, as individual tokens do not carry sufficient 
context to characterize the reasoning mode active at that step. We represent each 
sentence by the last-token hidden state, which aggregates contextual information from 
all preceding tokens in the sentence via causal attention. We verified that averaging 
over all token representations in a sentence yields weaker regime separation and lower 
predictive $R^2$, and that the last-token representation also aligns better with the 
steering setup where interventions are applied at a single position.

\textbf{Layer selection.} We select a single middle layer per model family rather than 
the first or last layer. Early layers are excluded because reasoning-specific 
representations have not yet formed: the model has processed the input but has not yet 
begun to organize intermediate computations. Late layers are excluded because 
representations at the final layer are dominated by the output distribution and tend 
to collapse into a small number of token prediction modes, reducing the diversity of 
the latent geometry. Middle layers represent a balance: the model has had sufficient 
depth to begin structuring its reasoning, but has not yet committed to a specific 
output. 
For the middle layers, we selected layers L22 for Llama-3.1-8B, L28 for Qwen2.5-14B, and L20 for Qwen2.5-Math-1.5B. In this paper, when we state that we "average over all layers", we are referring to averaging over middle and final layers. For the final layers, we used layers L31 for Llama-3.1-8B, L47 for Qwen2.5-14B, and L27 for Qwen2.5-Math-1.5B. 

Figure~\ref{fig:layer_ablation} confirms this choice. The selected layers fall within 
stable bands where TVD and $K_{\mathrm{eff}}$ are high, spectral gap is low, and 
persistence is sustained. In contrast, shallow layers show noisy and undifferentiated 
metrics for both base and reasoning models, and very deep layers show degraded TVD and 
inflated spectral gap consistent with output-mode collapse. The base-versus-reasoning 
contrast is visible across most of the depth range, but is most stable and 
consistent in the middle layers we selected.

\begin{figure}[h]
    \centering
    \includegraphics[width=\linewidth]{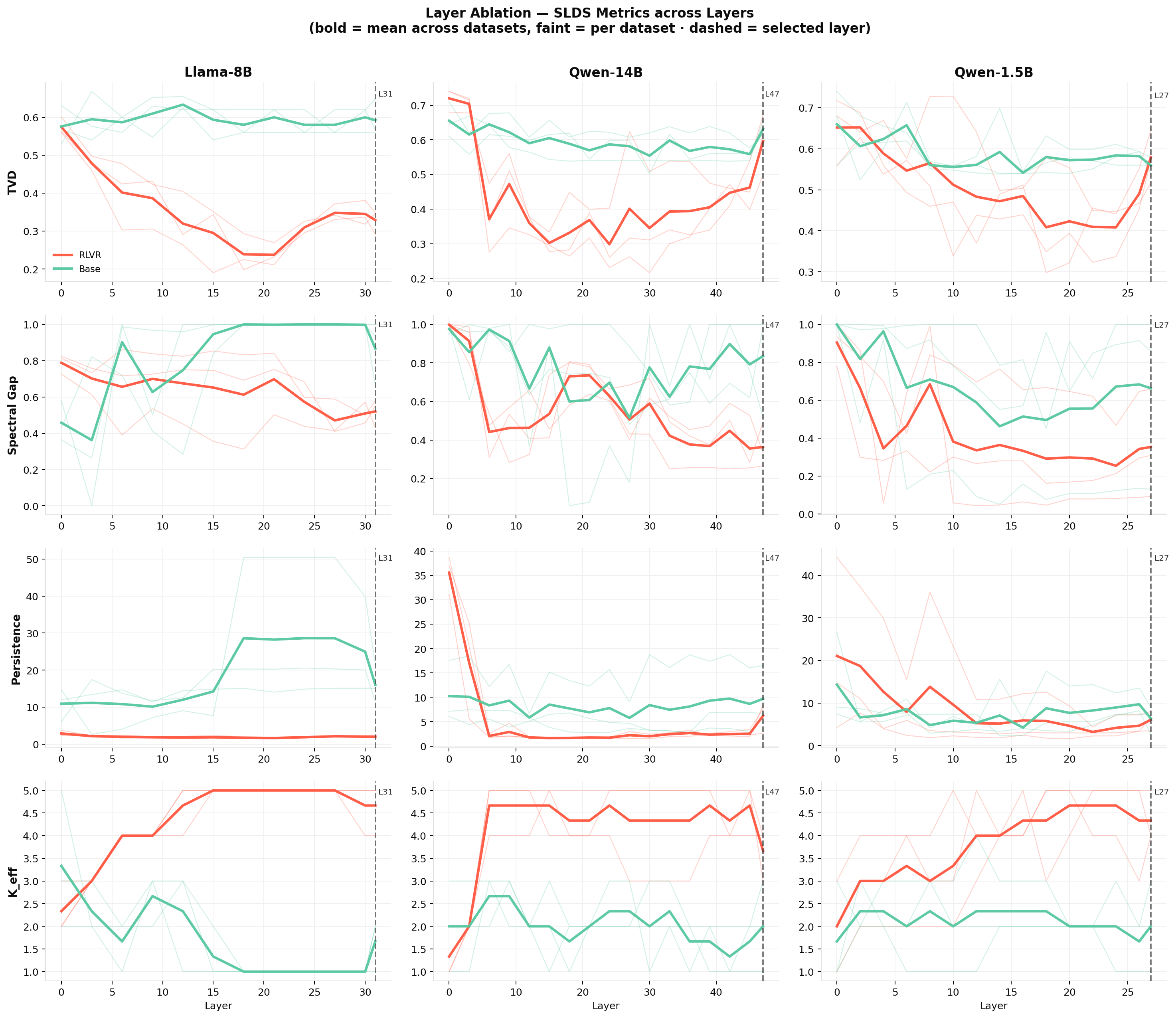}
    \caption{Layer ablation of SDS metrics across depth for each model family. Bold 
    lines show means across datasets; faint lines show per-dataset traces. The dashed 
    vertical line marks the selected layer. Selected layers fall in regions with stable 
    TVD, low spectral gap, high $K_{\mathrm{eff}}$, and sustained persistence.}
    \label{fig:layer_ablation}
\end{figure}





\subsection{Projection and inference ablations}
\label{app:em_moe_cebra_ablation}
We compare three alternatives for latent-state discovery: CEBRA+EM, CEBRA-MoE, and PCA+SLDS. The comparison spans predictive fit (Figure~\ref{fig:method_delta_r2}), state utilization (Figure~\ref{fig:method_k_eff}), transition structure (Figure~\ref{fig:method_tvd}), specialization, and persistence (Figure~\ref{fig:method_persist} and \ref{fig:method_spectral_gap}). Across metrics, CEBRA+EM provides the strongest overall balance between predictive gain and structured latent dynamics, while CEBRA-MoE under-uses states and PCA+SLDS tends to over-fragment trajectories into high-entropy, weakly persistent regimes.

\begin{figure*}[p]
    \centering
    \includegraphics[width=\textwidth]{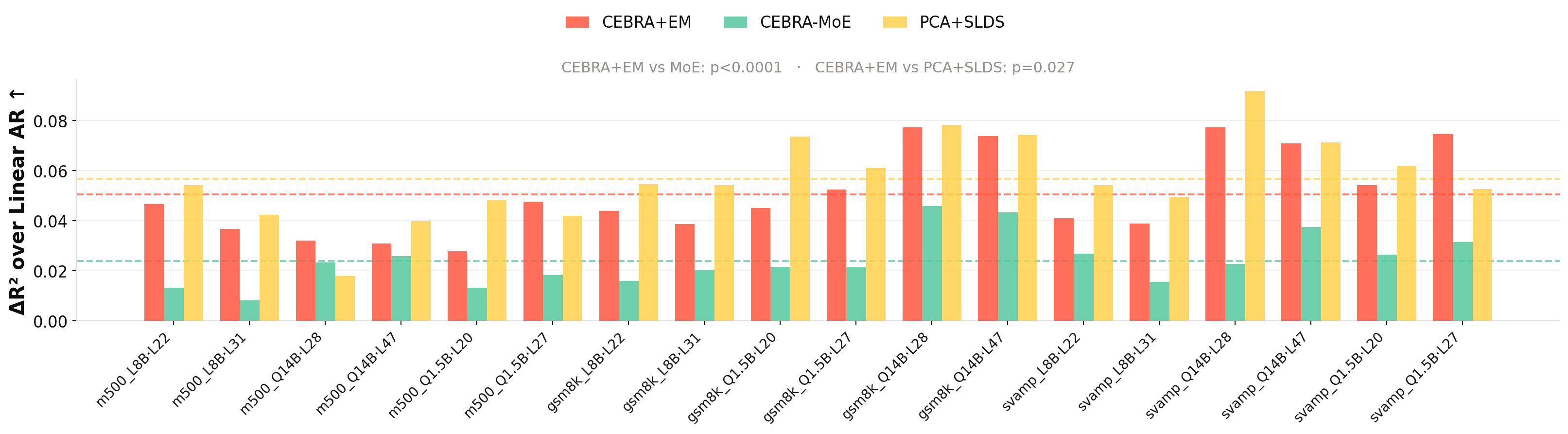}
    \caption{$\Delta R^2$ over linear autoregressive baselines for alternative latent-state discovery pipelines. CEBRA+EM consistently improves predictive fit over CEBRA-MoE, while PCA+SLDS achieves competitive fit on some settings but does so with substantially different structural tradeoffs.}
    \label{fig:method_delta_r2}
\end{figure*}

\begin{figure*}[p]
    \centering
    \includegraphics[width=\textwidth]{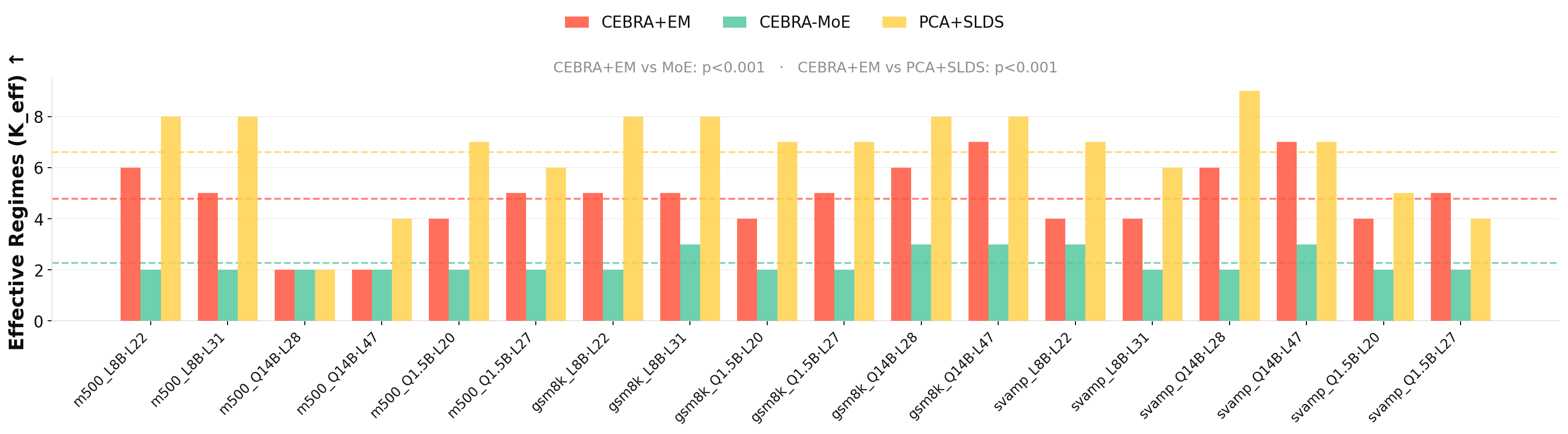}
    \caption{Effective number of occupied regimes for alternative pipelines. CEBRA-MoE tends to use too few states, whereas PCA+SLDS often uses a larger number of effective states. CEBRA+EM occupies an intermediate regime that is compatible with the small-cardinality policy picture developed in the main text.}
    \label{fig:method_k_eff}
\end{figure*}

\begin{figure*}[p]
    \centering
    \includegraphics[width=\textwidth]{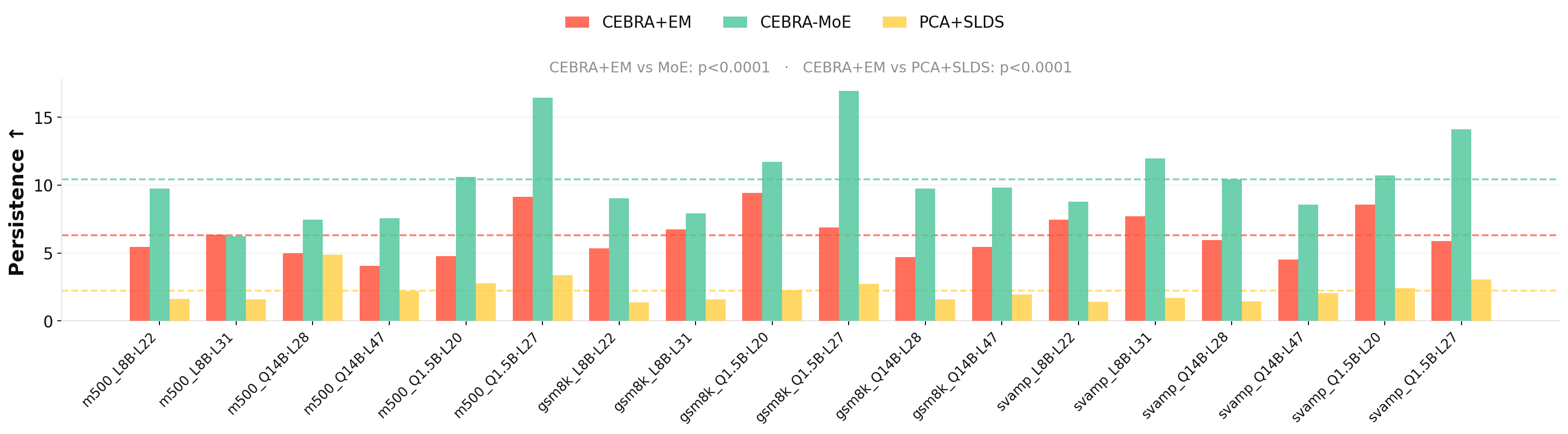}
    \caption{Persistence under alternative latent-state discovery pipelines. CEBRA-MoE attains the largest raw persistence but does so together with low state usage, while PCA+SLDS exhibits markedly lower persistence. CEBRA+EM balances persistence with broader regime utilization and stronger predictive adequacy.}
    \label{fig:method_persist}
\end{figure*}

\begin{figure*}[p]
    \centering
    \includegraphics[width=\textwidth]{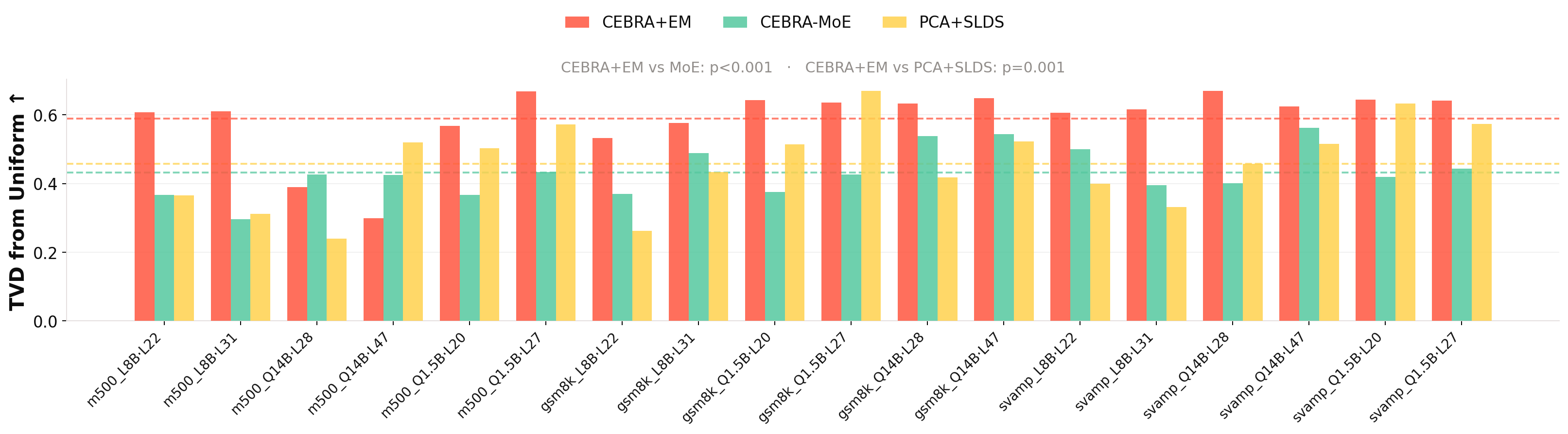}
    \caption{Transition-structure TVD relative to a reference distribution. CEBRA+EM yields the most consistently structured switching dynamics, whereas CEBRA-MoE and PCA+SLDS produce lower or less stable transition structure across settings.}
    \label{fig:method_tvd}
\end{figure*}

\begin{figure*}[p]
    \centering
    \includegraphics[width=\textwidth]{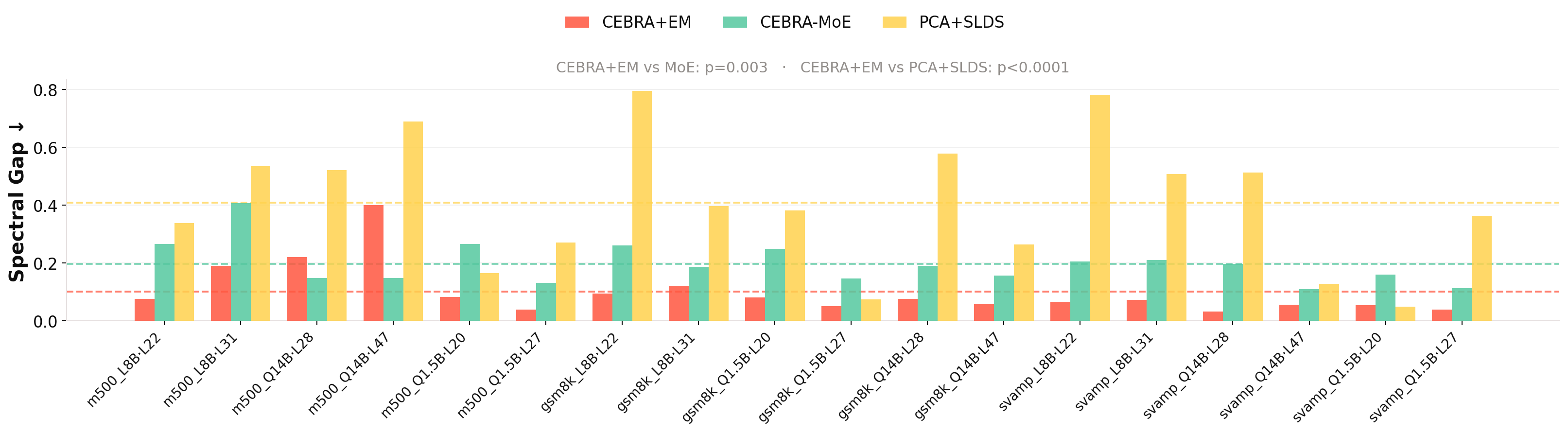}
    \caption{Spectral-gap proxy across alternative pipelines. Lower values correspond to slower mixing and more persistent latent regimes. CEBRA+EM generally attains smaller gaps than the alternatives, consistent with the emergence of longer-lived policy states.}
    \label{fig:method_spectral_gap}
\end{figure*}

\subsection{State-Swap Ablation}
\label{sec:state_swap_details}

Figures~\ref{fig:state_swap_early_layers} and~\ref{fig:state_swap_late_layers} report state-swap ablation results across early and late layers 
for all model families and datasets. In both cases the identity assignment substantially 
outperforms random permutation, with the gap growing in later layers.

\begin{figure*}[h]
    \begin{center}
    \includegraphics[width=\textwidth]{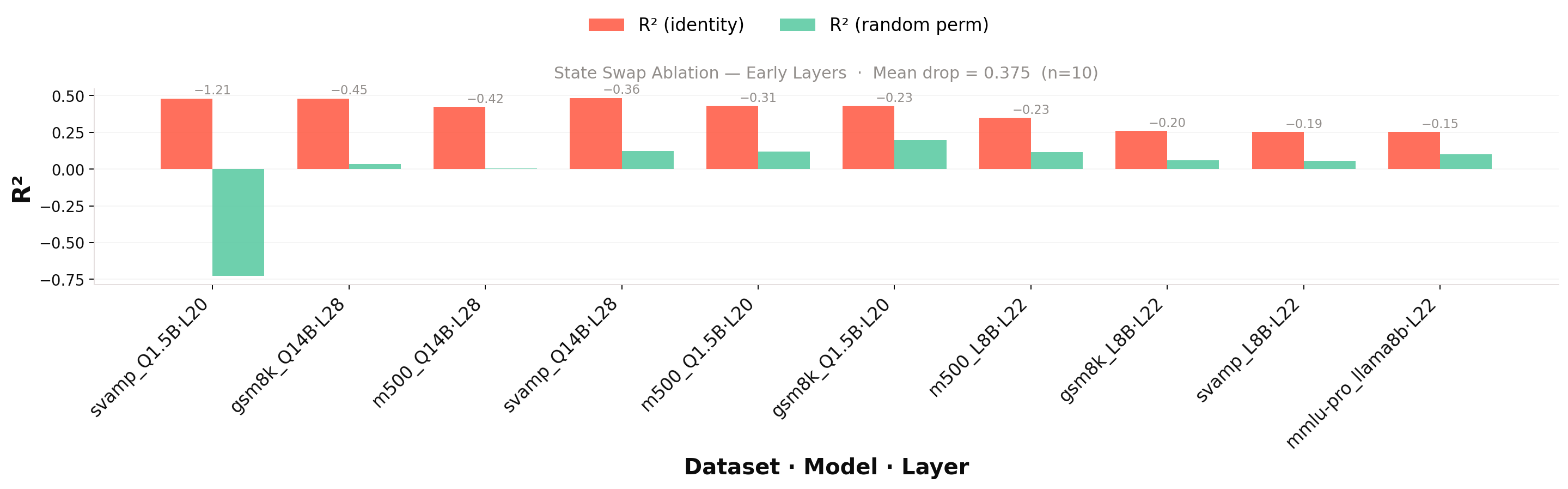}
    \end{center}
    \caption{State-swap ablation in early layers. The identity assignment preserves 
    substantially higher predictive fit than a random permutation across all 
    model-dataset-layer combinations.}
    \label{fig:state_swap_early_layers}
\end{figure*}

\begin{figure*}[h]
    \begin{center}
    \includegraphics[width=\textwidth]{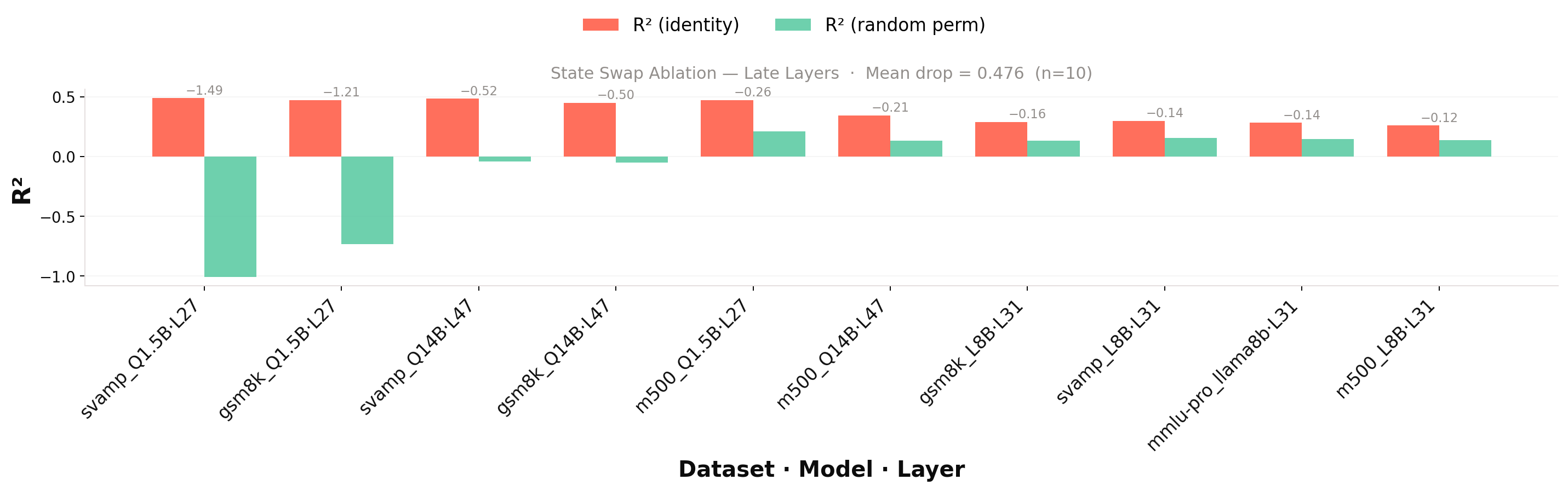}
    \end{center}
    \caption{State-swap ablation in late layers. The drop in $R^2$ under random 
    permutation is larger than in early layers, consistent with progressive 
    consolidation of latent policy structure with depth.}
    \label{fig:state_swap_late_layers}
\end{figure*}

\subsection{Additional State-Swap Ablation Details}
\label{app:state_swap_appendix}

The state-swap intervention isolates whether regime identity matters beyond the mere existence of multiple experts. After fitting the model and estimating per-regime linear dynamics $\{A_k, b_k\}_{k=1}^K$, we hold these coefficients fixed and randomly permute the mapping between inferred states and regime-specific dynamics. Let $R^2_{\mathrm{id}}$ denote the predictive fit under the original assignments and $R^2_{\sigma}$ the fit under a permutation $\sigma$ of the regime labels. Averaging over random permutations yields $\bar{R}^2_{\mathrm{rand}}$, and the gap
\[
\Delta = R^2_{\mathrm{id}} - \bar{R}^2_{\mathrm{rand}}
\]
measures how much predictive adequacy depends on applying the correct dynamical map at the correct time step.

The early- and late-layer results in the main text show that this gap is already positive in intermediate layers and grows in later layers. The interpretation is that the discovered states are not interchangeable clusters: they index distinct dynamical operators whose temporal deployment matters for next-step prediction. The stronger late-layer drop is consistent with a progressive consolidation of latent policy structure as the residual stream approaches the output distribution.

\subsection{Cross-Dataset Consistency of Latent Policy States}
\label{sec:cross_dataset}

A key question is whether the latent policy states recovered by our framework are 
specific to the dataset used for training or reflect a more general property of the 
model's reasoning behavior. To test this, we train a single CEBRA encoder on GSM8K 
and freeze its weights, then embed SVAMP, MATH-500, and MMLU-Pro into the same 
representation space without any retraining. We fit a separate SDS ($K=4$) on each 
dataset's embeddings using EM and measure pairwise consistency across three 
complementary metrics.

\textbf{Transition matrix similarity.} For two fitted transition matrices $T_1, T_2 
\in \mathbb{R}^{K \times K}$, we solve the Hungarian assignment problem to find the 
optimal state relabeling $\sigma^*$ and report
\begin{equation}
S_T(T_1, T_2) = 1 - \frac{\|T_1 - T_2^{\sigma^*}\|_F}{\sqrt{2K}},
\end{equation}
where $T_2^{\sigma^*}$ is $T_2$ permuted by $\sigma^*$. Values close to 1 indicate 
that the two datasets induce similar switching dynamics.

\textbf{Centroid cosine similarity.} Let $C_1, C_2 \in \mathbb{R}^{K \times d}$ be 
the per-regime centroids in the shared CEBRA space. We align them via Procrustes 
rotation $R^* = \arg\min_R \|C_2 R - C_1\|_F$ and report the mean cosine similarity 
under the optimal assignment:
\begin{equation}
S_C(C_1, C_2) = \frac{1}{K} \sum_{k=1}^K 
\frac{C_1[k]^\top (C_2 R^*)[{\sigma^*(k)}]}
{\|C_1[k]\| \cdot \|(C_2 R^*)[\sigma^*(k)]\|}.
\end{equation}

\textbf{Cross-fit $\Delta R^2$.} To test whether the SDS fitted on one dataset 
retains predictive utility on another, we apply the dynamics parameters 
$\theta_A = \{\pi_A, T_A, \{A_k, b_k, \Sigma_k\}_A\}$ fitted on dataset $A$ to 
embed and decode trajectories from dataset $B$, and report
\begin{equation}
\Delta R^2_{A \to B} = R^2_{\theta_A}(B) - R^2_{\mathrm{AR}}(B),
\end{equation}
where $R^2_{\mathrm{AR}}(B)$ is the linear autoregressive baseline on $B$. A positive 
value indicates that the cross-dataset SDS retains predictive structure beyond a 
single-mode baseline.

Figure~\ref{fig:cross_dataset} reports all three metrics across the six dataset pairs. 
GSM8K, SVAMP, and MMLU-Pro exhibit high mutual consistency ($S_T \geq 0.85$, $S_C 
\geq 0.76$) and positive cross-fit. 

\begin{figure}
    \centering
    \includegraphics[width=\linewidth,trim=0 0 0 30, clip]{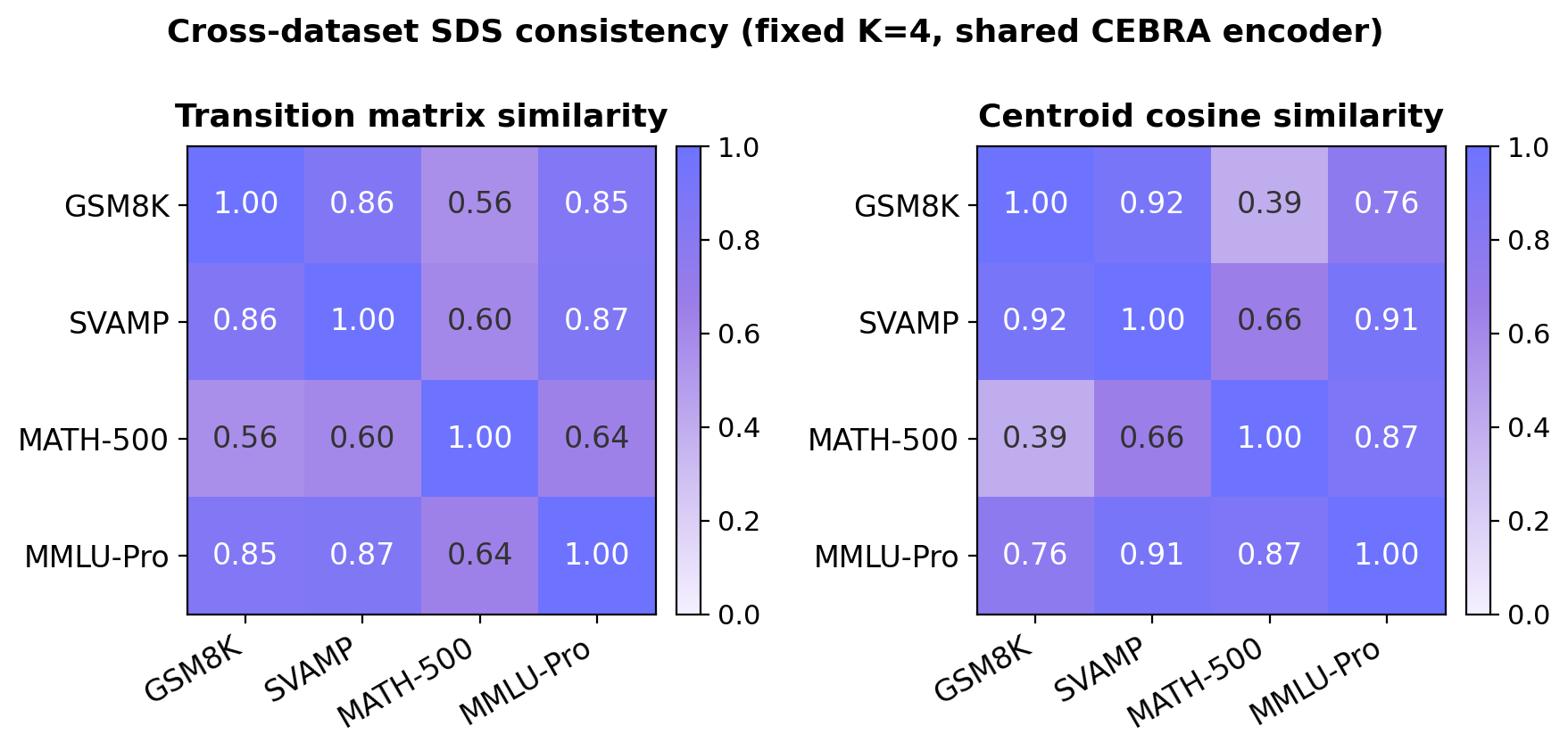}
    \caption{Cross-dataset consistency of recovered latent policy states under a shared 
frozen CEBRA encoder trained on GSM8K. Left: transition matrix similarity after 
Hungarian state alignment. Right: centroid cosine similarity after Procrustes 
rotation. GSM8K, SVAMP, and MMLU-Pro form a consistent cluster with high pairwise 
similarity ($S_T \geq 0.85$, $S_C \geq 0.76$), while MATH-500 shows lower alignment, 
consistent with its longer and more complex reasoning chains requiring richer latent 
structure than the other benchmarks.}
\label{fig:cross_dataset}
\end{figure}

\section{Functional State Specialization: Extended analysis}
\label{app:func_spec}



Figure~\ref{fig:func_spec} shows that across all three models, there is clear functional specialization.
A self-monitoring regime concentrating \textsc{Uncertainty\_Management} and \textsc{Self\_Checking} tokens emerges consistently: R1 in both Llama-8B and Qwen-1.5B, and split into two finer-grained states (R0 for uncertainty, R4 for verification) in Qwen-14B, suggesting that larger models develop more specialized introspective machinery. \textsc{Final\_Answer\_Emission} is likewise reliably isolated (R0 in Llama-8B, R3 in Qwen-1.5B, R2 in Qwen-14B), as is \textsc{Active\_Computation} (R3 in Qwen-14B, co-loading with \textsc{Fact\_Retrieval}). \textsc{Plan\_Generation} and \textsc{Problem\_Setup} stages tend to distribute more diffusely, often absorbed by lower-specificity background regimes, consistent with their role as transitional rather than focal reasoning behaviors. Notably, Qwen-1.5B exhibits the sharpest specialization overall, with R1 concentrating $78\%$ of \textsc{Uncertainty\_Management} and $68\%$ of \textsc{Self\_Checking} tokens, a more collapsed representation than the larger Qwen-14B, which distributes these behaviors across more states.


\section{Hard-Example Dataset Construction}
\label{sec:hard_examples}

We construct the hard-example subsets used for both policy transplantation
and \textsc{PrefixGuard} by generating eight responses from each base
checkpoint and retaining problems for which all eight responses are
incorrect, i.e., base-model pass@8 is zero. The same fixed subsets are used throughout both
actionable-policy evaluations.

The resulting collection contains 444 hard examples across three model
families and four benchmarks. Table~\ref{tab:hard_subset_sizes} reports
the per-setting sample sizes.

\begin{table}[t]
    \centering
    \small
    \setlength{\tabcolsep}{6pt}
    \begin{tabular}{lrrrrr}
        \toprule
        Model
        & GSM8K
        & MATH-500
        & MMLU-Pro
        & SVAMP
        & Total \\
        \midrule
        Llama-8B
        & 30 & 50 & 49 & 18 & \textbf{147} \\

        Qwen-1.5B
        & 50 & 49 & 50 & 27 & \textbf{176} \\

        Qwen-14B
        & 18 & 32 & 50 & 21 & \textbf{121} \\

        \midrule
        \textbf{Total}
        & \textbf{98}
        & \textbf{131}
        & \textbf{149}
        & \textbf{66}
        & \textbf{444} \\
        \bottomrule
    \end{tabular}
    \caption{
    Sizes of the hard-example subsets used for policy transplantation and
    \textsc{PrefixGuard}. The same 444-example collection spans all 12
    model--dataset settings across both actionable-policy evaluations.
    }
    \label{tab:hard_subset_sizes}
\end{table}

\section{Policy-Transplantation Implementation}
\label{app:transplant_details}

At inference time, we monitor the base model's residual stream through a
forward hook at layer 27 for Qwen-1.5B and layer 31 for Llama-8B. Every
$\tau=20$ generated tokens, the current hidden state $h_t$ is standardized
using the activation statistics of the corresponding
reasoning-fine-tuned model and projected through its trained CEBRA
encoder:
\[
z_t
=
f_{\mathrm{CEBRA}}
\left(
\frac{h_t-\mu_r}{\sigma_r}
\right),
\]
where $\mu_r$ and $\sigma_r$ are the reasoning model's activation mean
and standard deviation.

Because forward--backward posteriors are unavailable during online
generation, we infer the current regime using the nearest reasoning-model
centroid:
\[
s_t
=
\arg\min_k \|z_t-c_k\|_2,
\]
where $\{c_k\}_{k=1}^K$ are the centroids of the reasoning-model SDS. We
select the most likely successor state under its transition matrix:
\[
k^\star=\arg\max_j T_{s_tj}.
\]

Let $p_t(k)=T_{s_tk}$ be the natural next-state distribution and
$f_k(z_t)=A_kz_t+b_k$ the conditional mean dynamics associated with
regime $k$. The original predicted next-step mean is
\[
\mu_{\mathrm{orig}}(z_t)
=
\sum_{k=1}^{K}p_t(k)f_k(z_t).
\]
To favor the target successor $k^\star$, we apply a KL-regularized
exponential tilt:
\[
q_t^\star(k)
=
\frac{
p_t(k)\exp\!\left(\beta\,
\mathbb{1}\{k=k^\star\}\right)
}{
\sum_{j=1}^{K}
p_t(j)\exp\!\left(\beta\,
\mathbb{1}\{j=k^\star\}\right)
}.
\]
The corresponding steered prediction and latent intervention are
\[
\mu_{\mathrm{steered}}(z_t)
=
\sum_{k=1}^{K}q_t^\star(k)f_k(z_t),
\qquad
\Delta z_t
=
\mu_{\mathrm{steered}}(z_t)
-
\mu_{\mathrm{orig}}(z_t).
\]

We decode $\Delta z_t$ into activation space using the learned linear
decoder $W_{\mathrm{dec}}$ and bound the intervention relative to the
current hidden-state norm:
\[
\Delta h_t
=
\alpha
\frac{\Delta z_tW_{\mathrm{dec}}}
{\|\Delta z_tW_{\mathrm{dec}}\|_2}
\min\!\left(1,0.1\|h_t\|_2\right).
\]
The intervention $h_t\leftarrow h_t+\Delta h_t$ is then applied to the
residual stream. We use $\alpha=8$, $\beta=8$, generation temperature
$0.7$, and eight samples per problem for all reported experiments. No
model weights are updated.

\begin{algorithm*}[t]
\caption{Online transplantation of reasoning-model SDS dynamics}
\label{alg:em_state_steering}
\begin{algorithmic}[1]
\Require Hidden state $h_t$; encoder $f_{\mathrm{CEBRA}}$;
scaler $(\mu_r,\sigma_r)$; centroids $\{c_k\}$; SDS parameters
$\{A_k,b_k,T\}$; decoder $W_{\mathrm{dec}}$; strengths
$\alpha,\beta$
\State $z_t \gets
f_{\mathrm{CEBRA}}\bigl((h_t-\mu_r)/\sigma_r\bigr)$
\State $s_t \gets \arg\min_k\|z_t-c_k\|_2$
\State $k^\star \gets \arg\max_j T_{s_tj}$
\State $p_t(k) \gets T_{s_tk}$ for $k=1,\ldots,K$
\State $f_k(z_t)\gets A_kz_t+b_k$ for $k=1,\ldots,K$
\State $\mu_{\mathrm{orig}}
\gets\sum_k p_t(k)f_k(z_t)$
\State $\bar q_t(k)
\gets p_t(k)\exp\!\bigl(\beta\mathbb{1}\{k=k^\star\}\bigr)$
\State $q_t^\star(k)
\gets \bar q_t(k)/\sum_j\bar q_t(j)$
\State $\mu_{\mathrm{steered}}
\gets\sum_k q_t^\star(k)f_k(z_t)$
\State $\Delta z_t
\gets\mu_{\mathrm{steered}}-\mu_{\mathrm{orig}}$
\State $u_t\gets\Delta z_tW_{\mathrm{dec}}$
\State $\Delta h_t
\gets\alpha\,u_t/\|u_t\|_2
\cdot\min(1,0.1\|h_t\|_2)$
\State \Return $h_t+\Delta h_t$
\end{algorithmic}
\end{algorithm*}

\section{Cross-Model Transfer and Policy-Transplantation Results}
\label{app:transplant_results}

We evaluate transfer in two complementary ways. First, we test whether an
SDS fitted on one checkpoint can predict trajectories from its paired
checkpoint. This measures compatibility between the recovered dynamical
models without modifying activations. Second, we intervene on the base
model during generation and test whether importing reasoning-model
dynamics changes downstream behavior.

\subsection{Asymmetric Predictive Transfer of SDS Dynamics}
\label{app:asymmetric_sds_transfer}

Let $R^2_{x\to T}$ denote the predictive fit obtained by applying an SDS
trained on source model $x\in\{b,r\}$ to trajectories from a fixed target
model $T\in\{b,r\}$, where $b$ and $r$ denote the base and
reasoning-fine-tuned checkpoints. We define the target-normalized
degradation caused by replacing the target's native SDS with the SDS from
the paired checkpoint as
\begin{equation}
\label{eq:common-target-drop}
\Delta_{\to T}
=
\frac{
R^2_{\mathrm{nat}}(T)-R^2_{\mathrm{oth}}(T)
}{
R^2_{\mathrm{nat}}(T)
}.
\end{equation}

Table~\ref{tab:combined-sds-transplant-all-models} reports results at
representative middle layers: layer 20 for Qwen2.5-Math-1.5B, layer 22
for Llama-3.1-8B, and layer 28 for Qwen2.5-14B. Across all six
model--dataset settings, applying base-trained dynamics to reasoning
trajectories causes a larger relative degradation than applying
reasoning-trained dynamics to base trajectories:
\[
\Delta_{\to r}>\Delta_{\to b}.
\]
The mean degradation is $25.1\%$ in the former direction and $17.4\%$ in
the latter.

\begin{table*}[t]
\centering
\small
\setlength{\tabcolsep}{5pt}
\renewcommand{\arraystretch}{1.08}
\begin{tabular*}{\textwidth}{
@{\extracolsep{\fill}}llcccccc@{}
}
\toprule
Model & Dataset
& $R^2_{r\to r}$
& $R^2_{b\to r}$
& $R^2_{b\to b}$
& $R^2_{r\to b}$
& $\Delta_{\to r}$
& $\Delta_{\to b}$ \\
\midrule
Qwen2.5-Math-1.5B
& GSM8K
& $0.414$ & $0.301$
& $0.548$ & $0.456$
& $27.2\%$ & $16.7\%$ \\
& SVAMP
& $0.475$ & $0.372$
& $0.577$ & $0.495$
& $21.8\%$ & $14.1\%$ \\
\midrule
Llama-3.1-8B
& GSM8K
& $0.242$ & $0.182$
& $0.452$ & $0.355$
& $25.0\%$ & $21.4\%$ \\
& SVAMP
& $0.239$ & $0.188$
& $0.507$ & $0.438$
& $21.2\%$ & $13.6\%$ \\
\midrule
Qwen2.5-14B
& GSM8K
& $0.455$ & $0.333$
& $0.649$ & $0.491$
& $26.8\%$ & $24.4\%$ \\
& SVAMP
& $0.463$ & $0.331$
& $0.557$ & $0.479$
& $28.6\%$ & $14.0\%$ \\
\midrule
\multicolumn{6}{r}{\textbf{Mean relative degradation}}
& $\mathbf{25.1\%}$
& $\mathbf{17.4\%}$ \\
\bottomrule
\end{tabular*}
\caption{
Cross-model predictive transfer of SDS dynamics at representative middle
layers. Native evaluations apply an SDS to trajectories from the model on
which it was fitted; cross evaluations apply it to the paired checkpoint.
In every setting, base-trained dynamics degrade more when transferred to
reasoning trajectories than reasoning-trained dynamics do when transferred
to base trajectories.
}
\label{tab:combined-sds-transplant-all-models}
\end{table*}

This asymmetry suggests that reasoning-trained dynamics remain partially
compatible with base-model trajectories, whereas base-trained dynamics
do not capture part of the organization present in reasoning-model
trajectories. We interpret this as evidence that reasoning fine-tuning
reorganizes and differentiates a dynamical scaffold already present in
the base model, rather than creating all latent dynamics from scratch.

\subsection{Behavioral Policy Transplantation}
\label{app:behavioral_transplant}

We next evaluate whether transferring reasoning-model SDS dynamics can
improve the base model's behavior. Table~\ref{tab:transplant_full}
reports pass@8 on the hard subsets defined in
Appendix~\ref{sec:hard_examples}. The unsteered baseline is zero by
construction.

\begin{table}[t]
\centering
\small
\setlength{\tabcolsep}{5pt}
\renewcommand{\arraystretch}{1.08}
\begin{tabular}{llrcc}
\toprule
Model & Dataset
& $n_{\mathrm{hard}}$
& Base
& Steered \\
\midrule
\multirow{3}{*}{Qwen-1.5B}
& GSM8K
& 50 & $0.00$ & $\mathbf{0.60}$ \\
& MATH-500
& 49 & $0.00$ & $\mathbf{0.24}$ \\
& SVAMP
& 27 & $0.00$ & $\mathbf{0.22}$ \\
\midrule
\multirow{3}{*}{Llama-8B}
& GSM8K
& 30 & $0.00$ & $\mathbf{0.46}$ \\
& MATH-500
& 50 & $0.00$ & $\mathbf{0.33}$ \\
& MMLU-Pro
& 49 & $0.00$ & $\mathbf{0.28}$ \\
\bottomrule
\end{tabular}
\caption{
Pass@8 after transplanting reasoning-model SDS dynamics into the paired
base model. Evaluation uses the larger hard-example subsets from
Table~\ref{tab:hard_subset_sizes}, comprising problems with base-model
$\mathrm{pass@8}=0$ by construction.
}
\label{tab:transplant_full}
\end{table}
Policy transplantation improves pass@8 in every evaluated setting,
reaching $0.60$ for Qwen-1.5B on GSM8K and $0.46$ for Llama-8B on
GSM8K. These results show that steering the base model toward transitions
recovered from its reasoning-fine-tuned counterpart can causally alter
downstream reasoning without changing model weights.

We treat this intervention as causal validation rather than as a practical
deployment strategy: when the reasoning model is available, it would
normally be used directly. The SDS-guided prefix-pruning method in
Section~\ref{sec:prefixguard} instead operates directly on reasoning-model
trajectories and provides the practical application of the recovered
latent-state structure.

\section{CEBRA Trajectory Visualizations Gallery}
\label{app:cebra_traj_gallery}

To visualize the geometry of the recovered latent policies directly, we plot representative reasoning traces in the learned CEBRA embedding space together with their decoded regime sequences. Each panel shows a single trajectory overlaid on the full point cloud of embedded steps from the corresponding model, dataset, and layer, with the right-hand column displaying the aligned sequence of inferred regimes across reasoning steps. These visualizations in Figures~\ref{fig:traj_llama_gsm8k}, \ref{fig:traj_llama_math500}, \ref{fig:traj_qwen15_gsm8k}, \ref{fig:traj_qwen15_math500}, \ref{fig:traj_qwen14_gsm8k} and \ref{fig:traj_qwen14_math500} make the contrast in temporal organization concrete. Base-model trajectories frequently dwell in one dominant cluster for long stretches or alternate among nearby regions without a stable large-scale progression. Reasoning trajectories more often execute a structured path through separated clusters, with regime occupancy changing at a pace that is neither purely sticky nor rapidly flickering.

\begin{figure*}[p]
    \centering
    \includegraphics[width=0.92\textwidth]{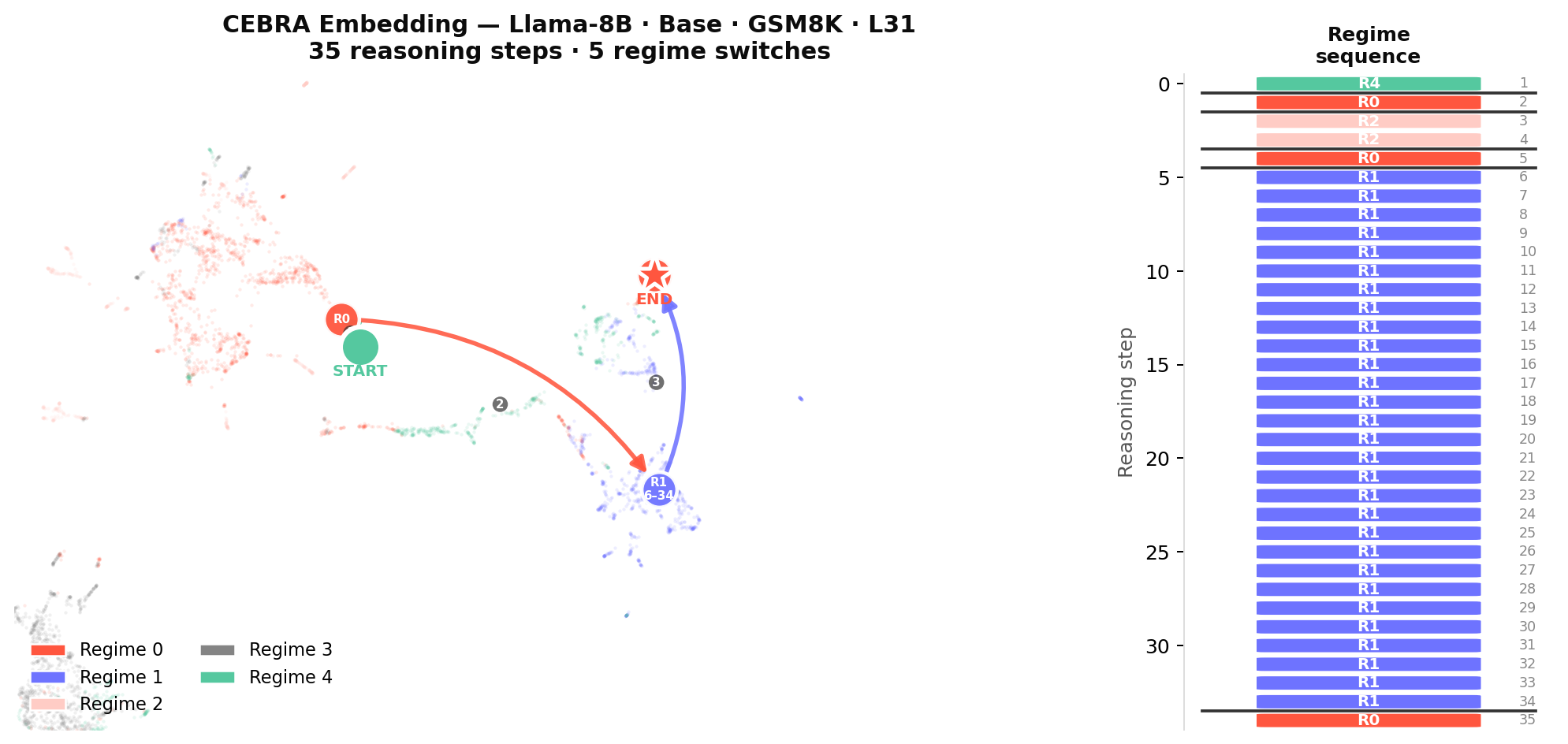}\\[0.6em]
    \includegraphics[width=0.92\textwidth]{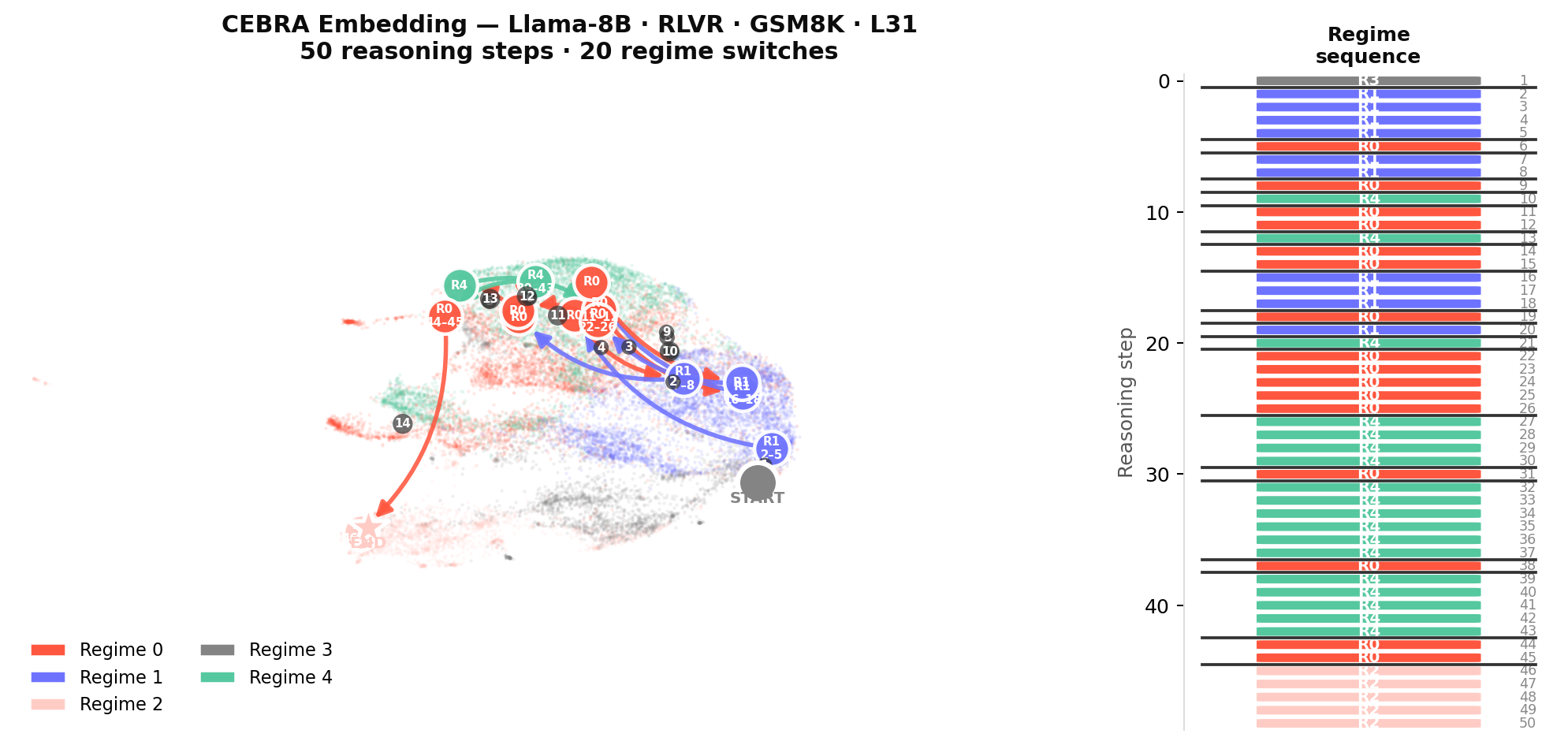}
    \caption{Representative CEBRA trajectories for Llama-8B on GSM8K. \emph{Top:} \texttt{Base}. \emph{Bottom:} \texttt{Reasoning}. The base trajectory quickly enters a prolonged regime-$R1$ plateau after a brief prefix, whereas the reasoning trajectory traverses multiple clusters and revisits distinct regimes before terminating in a separate end region.}
    \label{fig:traj_llama_gsm8k}
\end{figure*}

\begin{figure*}[p]
    \centering
    \includegraphics[width=\textwidth]{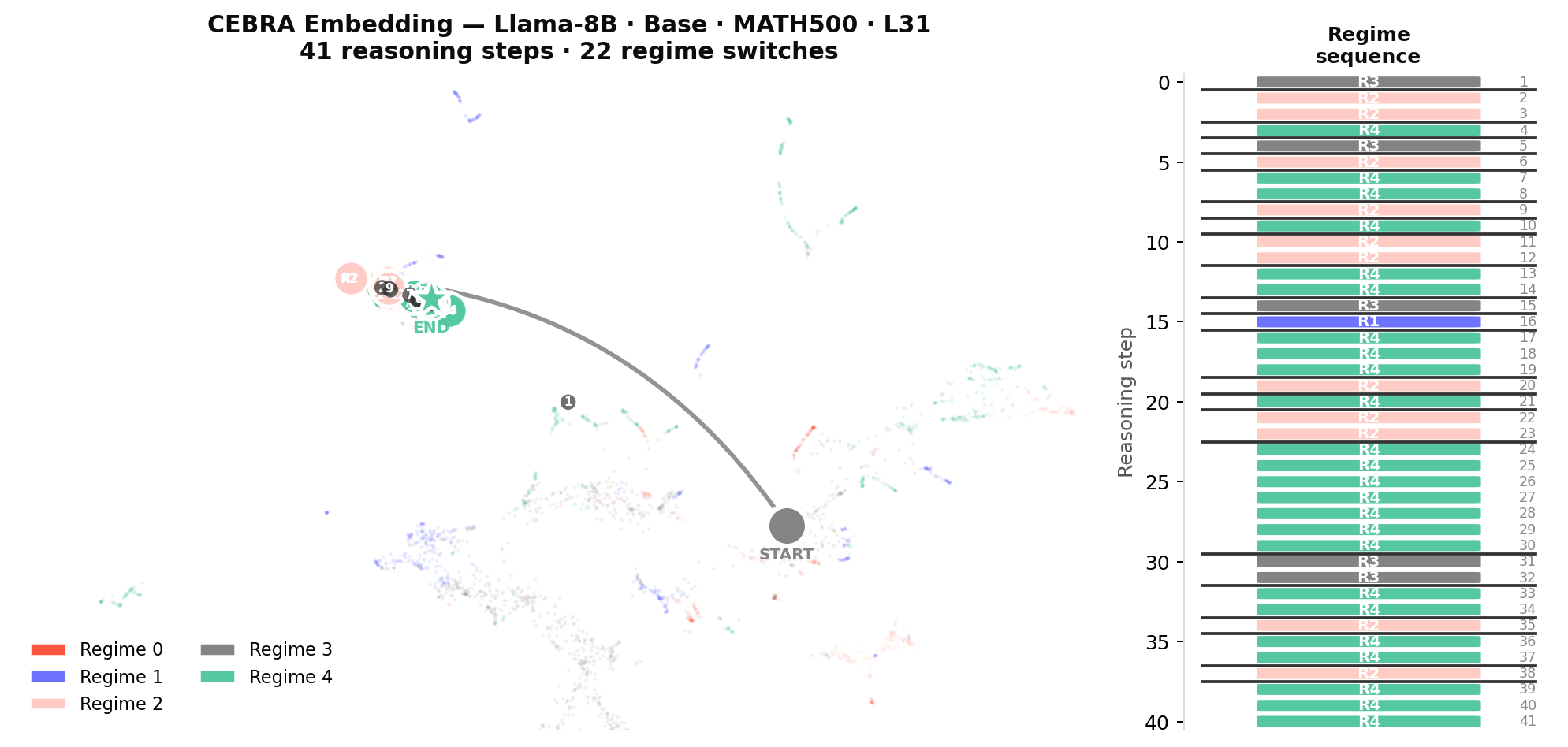}\\[0.6em]
    \includegraphics[width=\textwidth]{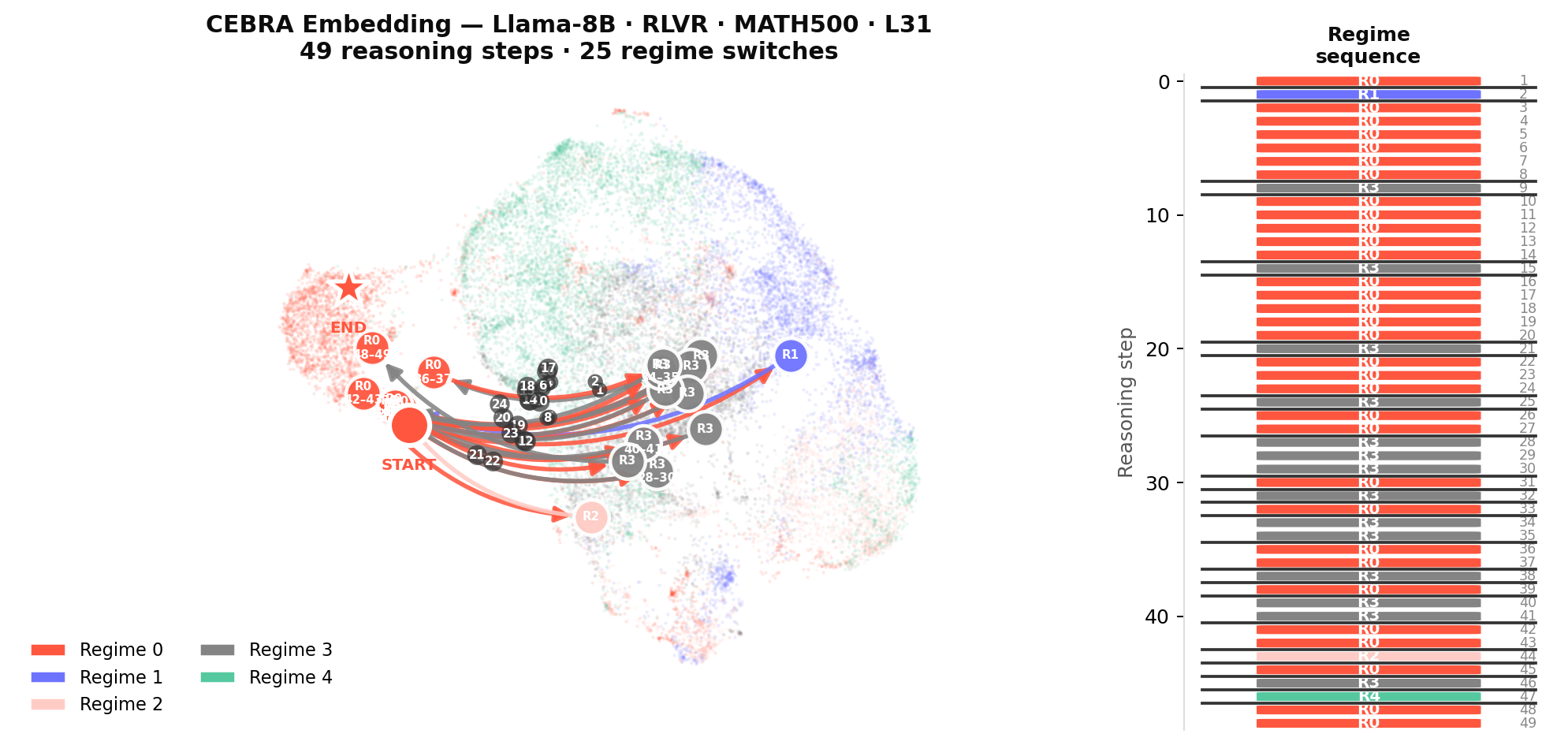}
    \caption{Representative CEBRA trajectories for Llama-8B on MATH500. \emph{Top:} \texttt{Base}. \emph{Bottom:} \texttt{Reasoning}. The base run shows rapid local switching near the terminal region without a clear global traversal, while the reasoning run follows a longer multi-cluster path with a sustained middle segment and a distinct terminal computation basin.}
    \label{fig:traj_llama_math500}
\end{figure*}

\begin{figure*}[p]
    \centering
    \includegraphics[width=0.92\textwidth]{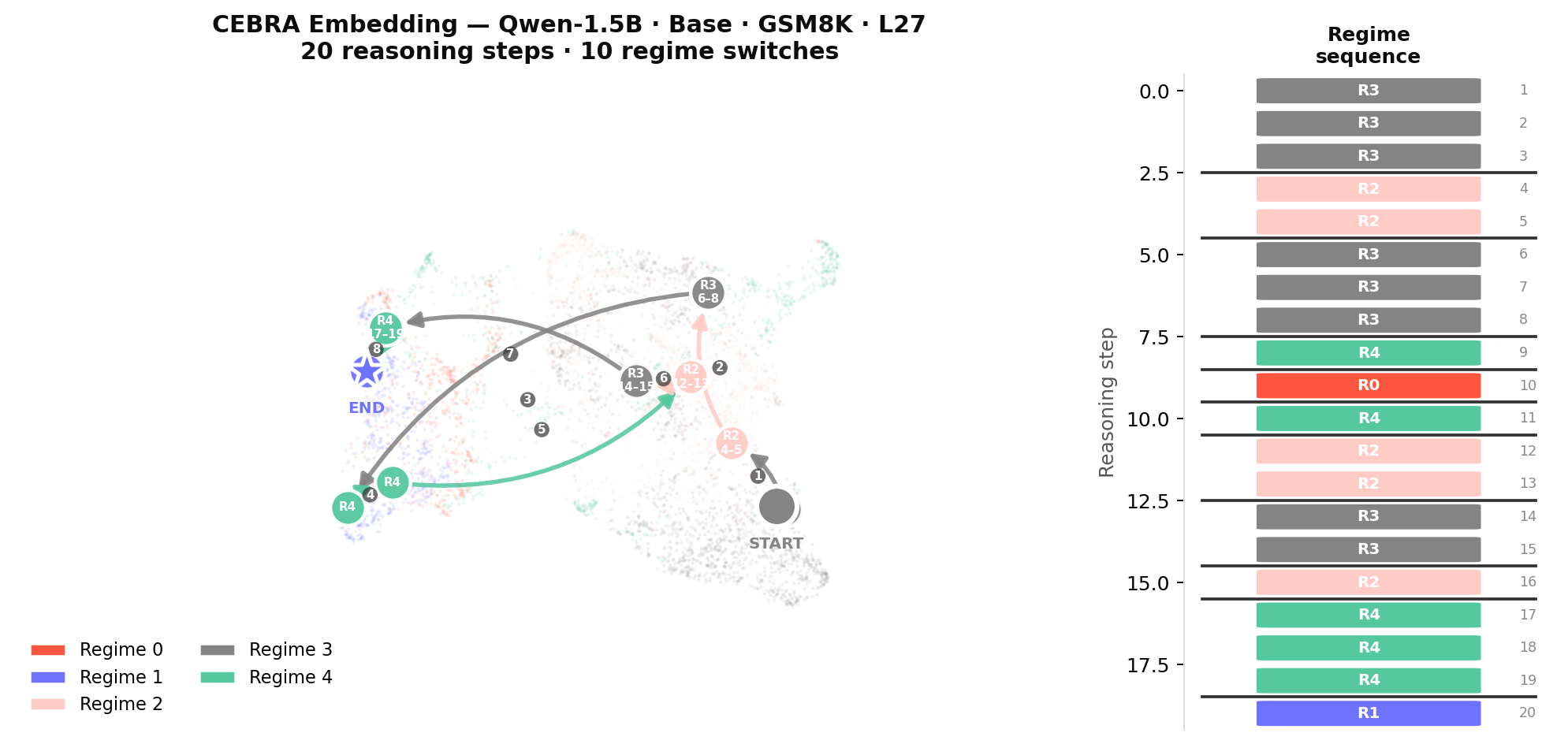}\\[0.6em]
    \includegraphics[width=0.92\textwidth]{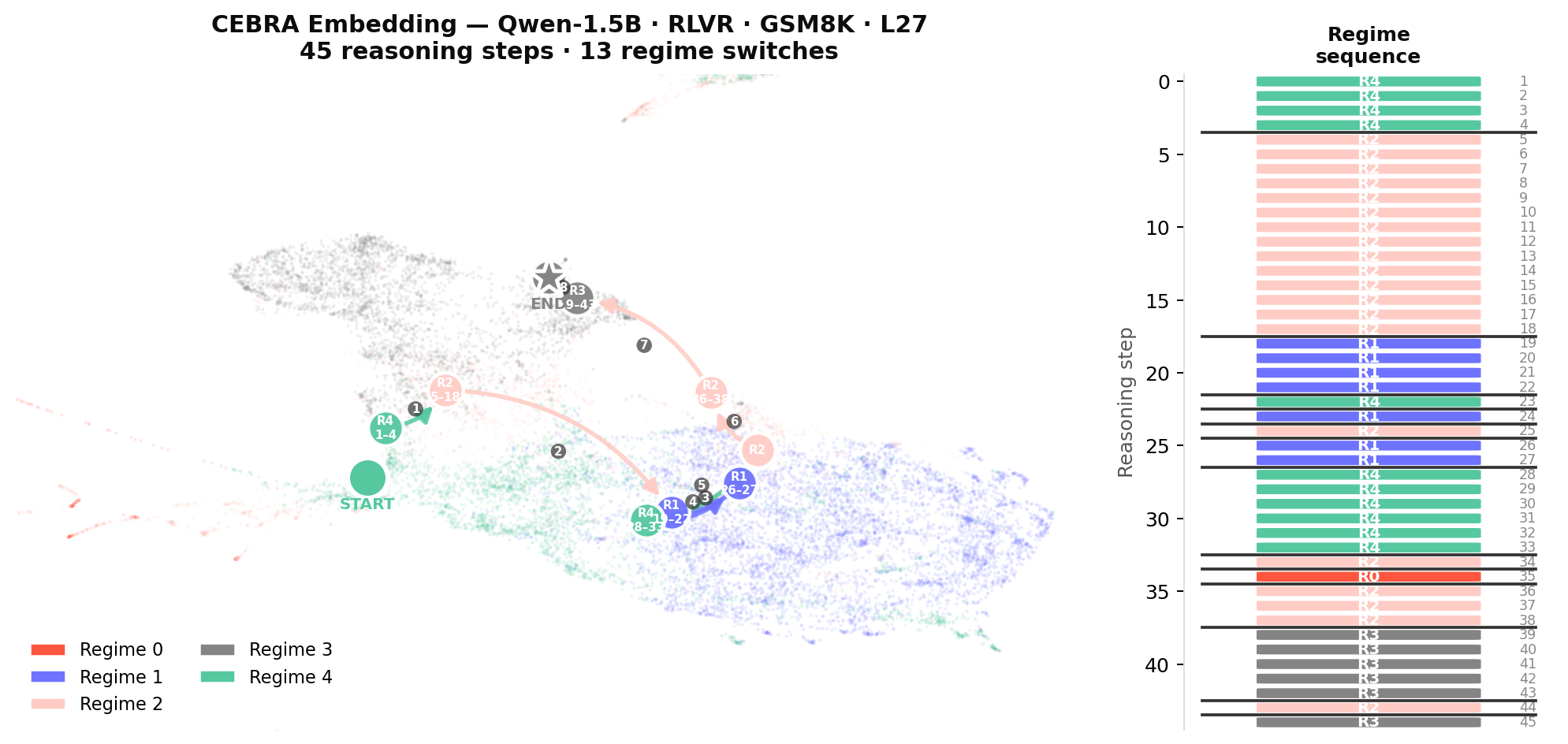}
    \caption{Representative CEBRA trajectories for Qwen-1.5B on GSM8K. \emph{Top:} \texttt{Base}. \emph{Bottom:} \texttt{Reasoning}. In the base model, the trajectory alternates among a small set of nearby clusters and ends with an isolated terminal jump. In the reasoning model, regime segments are longer and the path organizes into a clearer progression from an initial green cluster through blue and pink segments before reaching the terminal state.}
    \label{fig:traj_qwen15_gsm8k}
\end{figure*}

\begin{figure*}[p]
    \centering
    \includegraphics[width=0.92\textwidth]{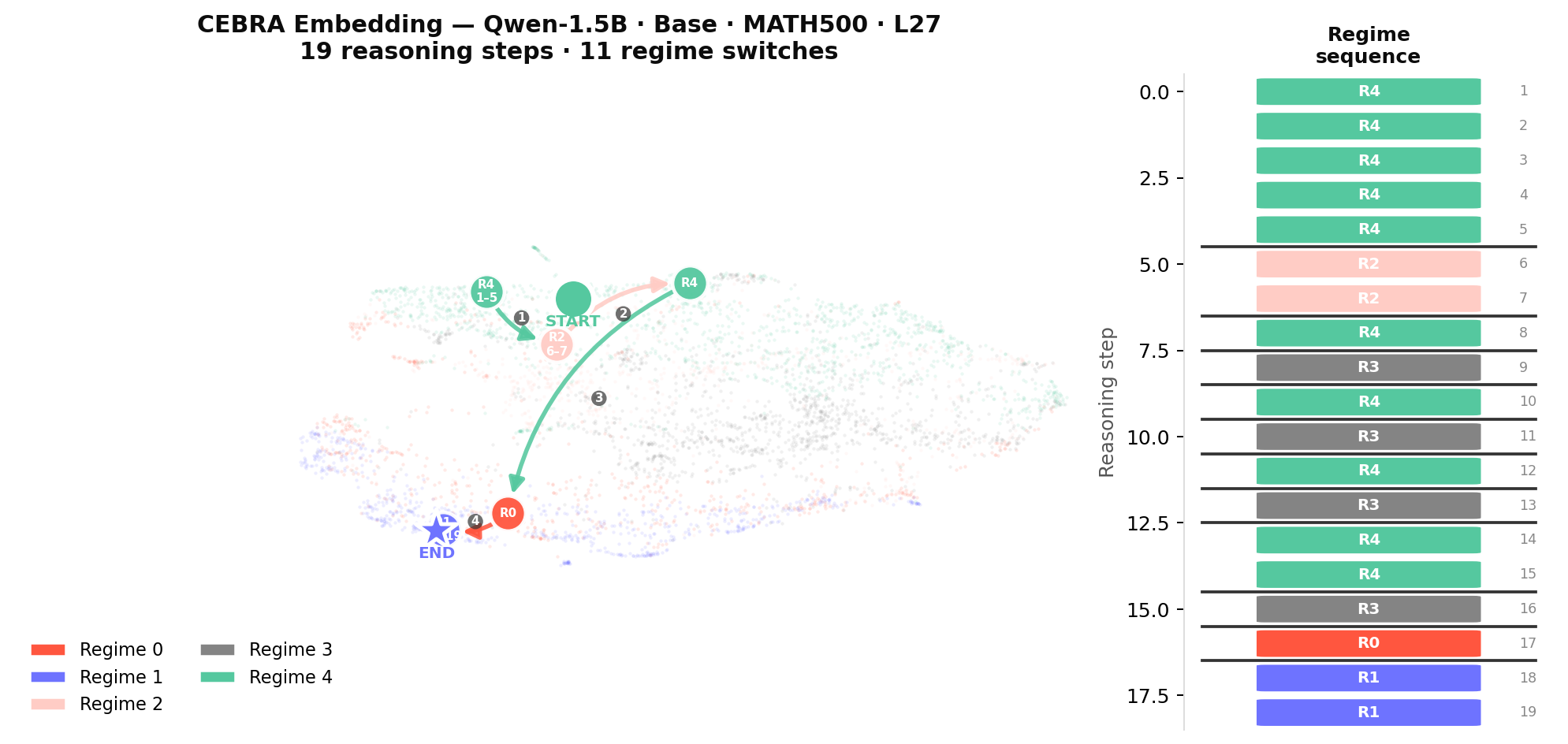}\\[0.6em]
    \includegraphics[width=0.92\textwidth]{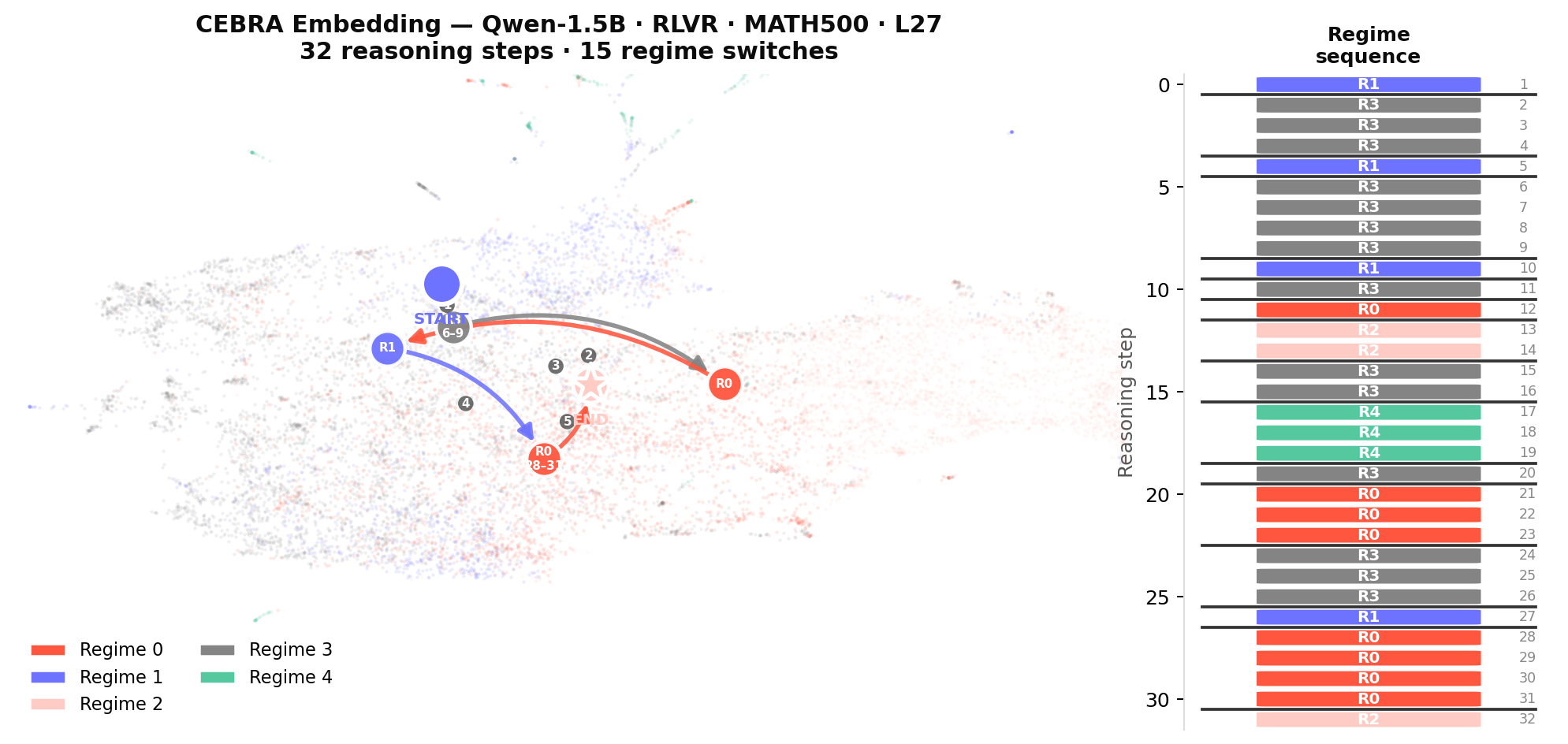}
    \caption{Representative CEBRA trajectories for Qwen-1.5B on MATH500. \emph{Top:} \texttt{Base}. \emph{Bottom:} \texttt{Reasoning}. The base trajectory remains concentrated in a narrow band with frequent short switches, whereas the reasoning trajectory separates into longer blocks centered on a small number of regimes and transitions through a more interpretable start-to-end route.}
    \label{fig:traj_qwen15_math500}
\end{figure*}

\begin{figure*}[p]
    \centering
    \includegraphics[width=0.92\textwidth]{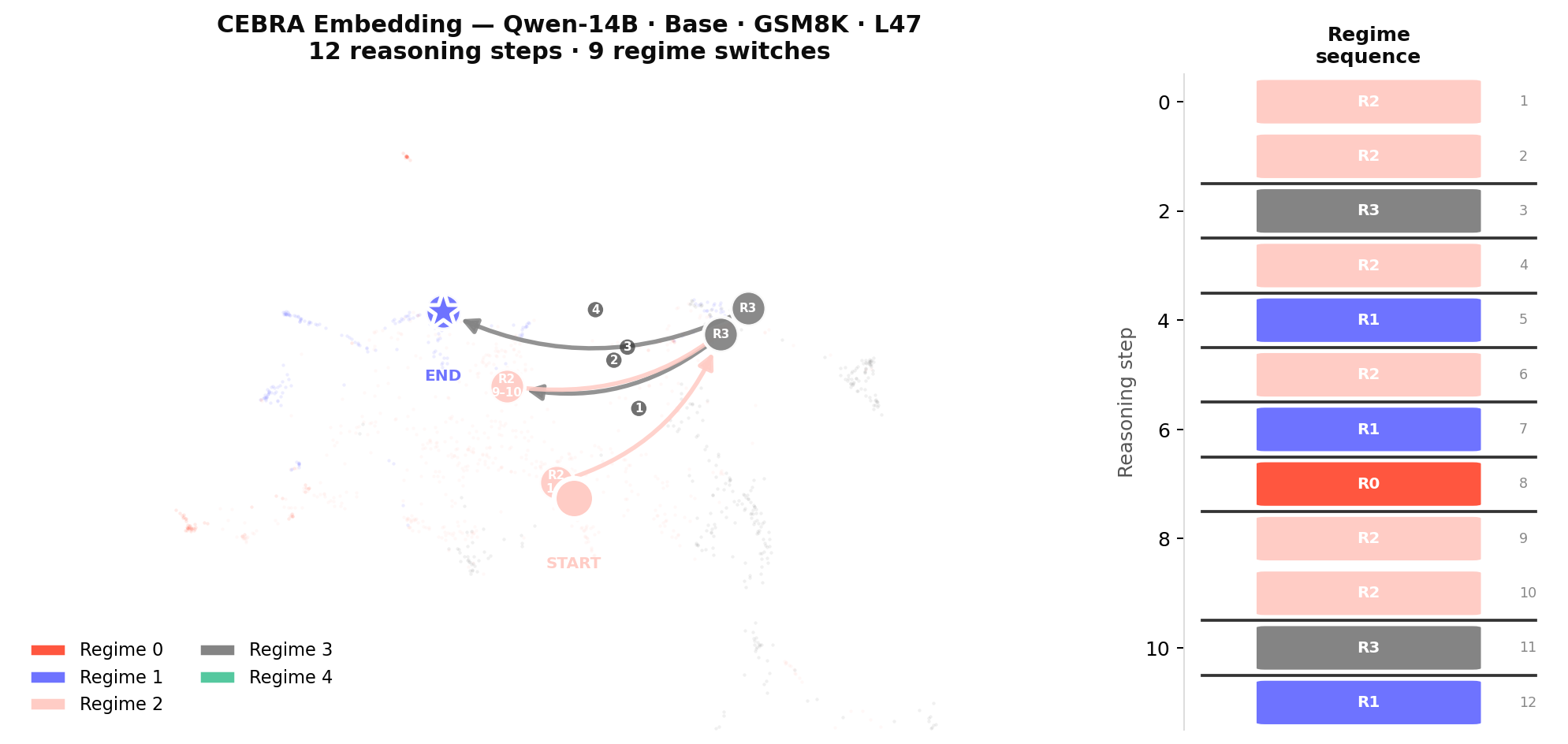}\\[0.6em]
    \includegraphics[width=0.92\textwidth]{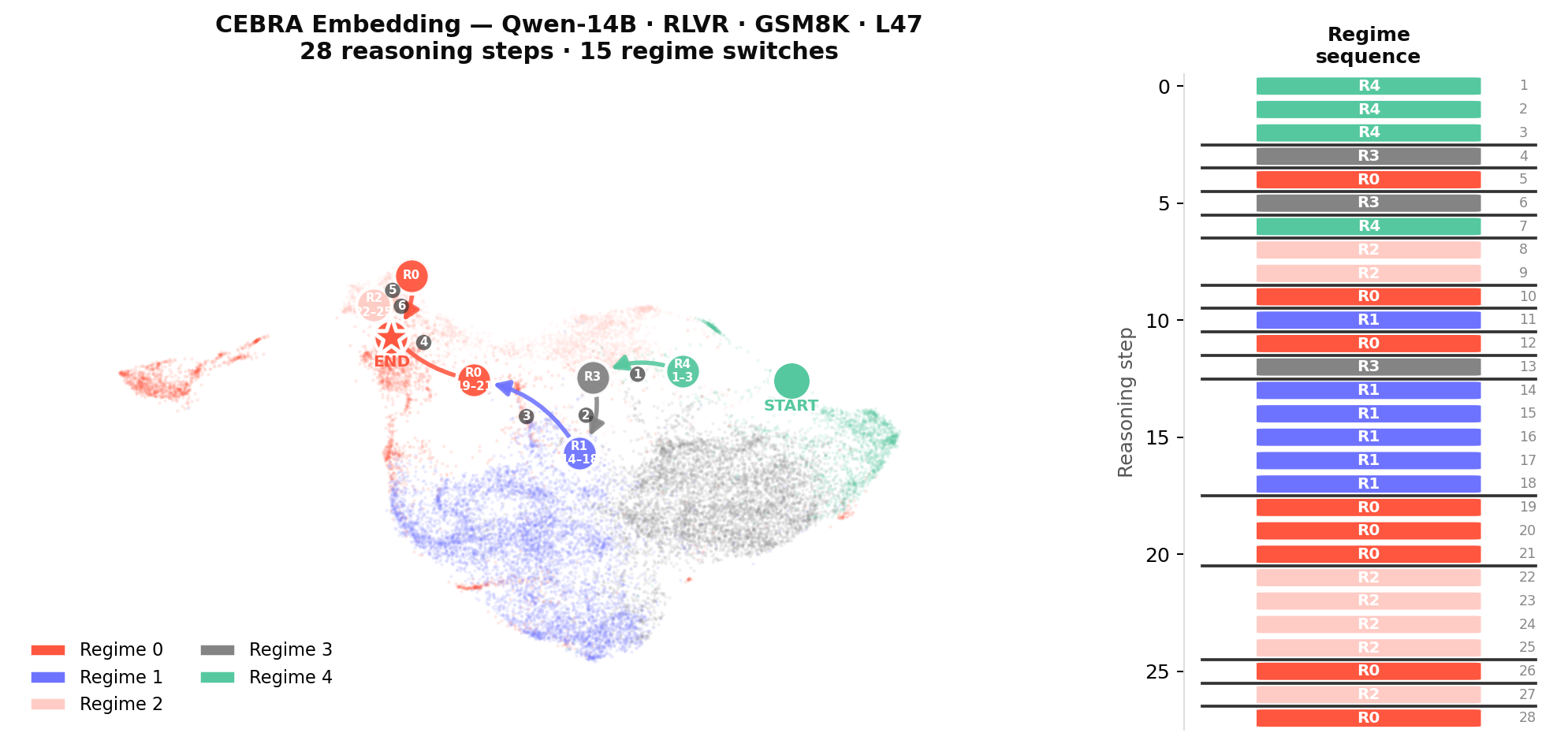}
    \caption{Representative CEBRA trajectories for Qwen-14B on GSM8K. \emph{Top:} \texttt{Base}. \emph{Bottom:} \texttt{Reasoning}. The base trajectory oscillates among nearby $R2$, $R3$, and $R1$ regions with limited geometric separation, while the reasoning trajectory begins in a distinct green cluster and proceeds through blue and red regimes toward a compact terminal region.}
    \label{fig:traj_qwen14_gsm8k}
\end{figure*}

\begin{figure*}[p]
    \centering
    \includegraphics[width=0.92\textwidth]{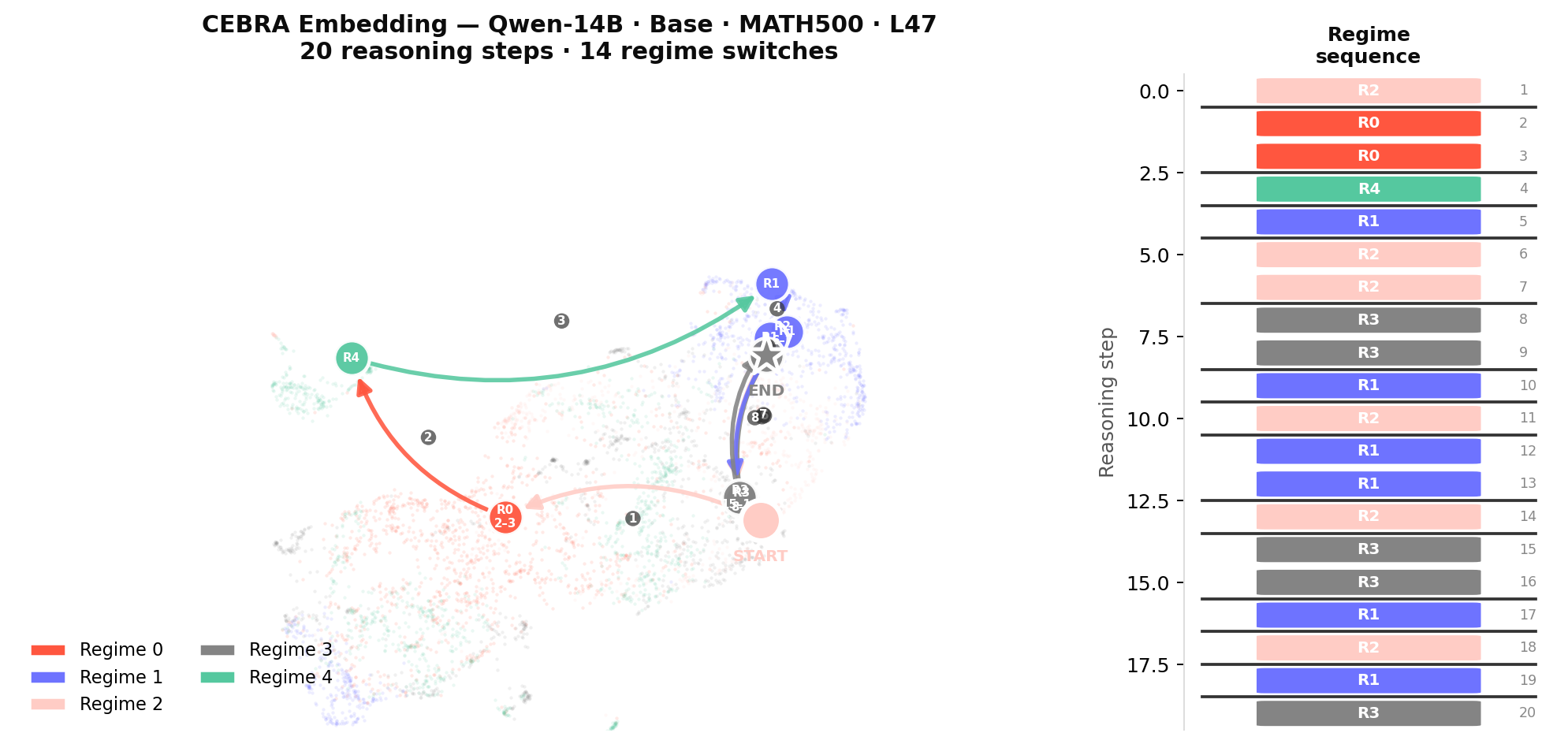}\\[0.6em]
    \includegraphics[width=0.92\textwidth]{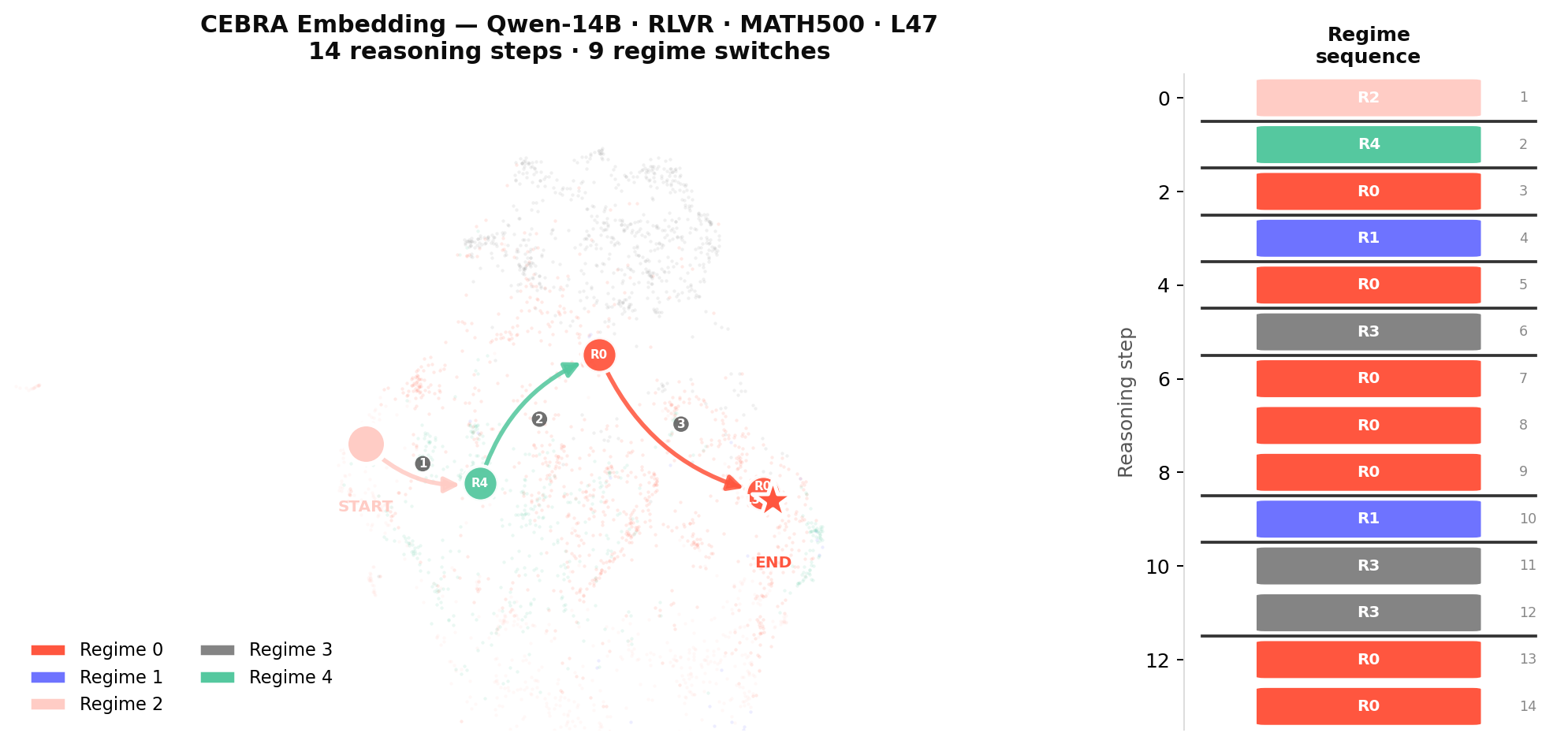}
    \caption{Representative CEBRA trajectories for Qwen-14B on MATH500. \emph{Top:} \texttt{Base}. \emph{Bottom:} \texttt{Reasoning}. The base model produces a diffuse trajectory with frequent short-lived switches near the terminal region, whereas the reasoning model concentrates its trajectory into a smaller number of spatially separated clusters and a more ordered regime sequence.}
    \label{fig:traj_qwen14_math500}
\end{figure*}
\definecolor{stdgray}{RGB}{120,120,120}

\begin{table*}[t]
\centering
\small
\setlength{\tabcolsep}{3.8pt}
\begin{tabular}{llcccc}
\toprule
Model & Dataset & $K_{\mathrm{eff}}\uparrow$ & $p_{\mathrm{stay}}\uparrow$ & TVD$\uparrow$ & Spec. Gap$\downarrow$ \\
\midrule

\multirow{4}{*}{Llama-8B Base}

& GSM8K    
& $4.50 {\color{stdgray}\scriptstyle \pm 0.75}$ 
& $0.69 {\color{stdgray}\scriptstyle \pm 0.04}$ 
& $0.48 {\color{stdgray}\scriptstyle \pm 0.05}$ 
& $0.20 {\color{stdgray}\scriptstyle \pm 0.06}$ \\

& SVAMP    
& $5.70 {\color{stdgray}\scriptstyle \pm 0.96}$ 
& $0.62 {\color{stdgray}\scriptstyle \pm 0.07}$ 
& $0.47 {\color{stdgray}\scriptstyle \pm 0.03}$ 
& $0.21 {\color{stdgray}\scriptstyle \pm 0.05}$ \\

& MATH500  
& $3.10 {\color{stdgray}\scriptstyle \pm 0.51}$ 
& $0.77 {\color{stdgray}\scriptstyle \pm 0.03}$ 
& $0.44 {\color{stdgray}\scriptstyle \pm 0.06}$ 
& $0.23 {\color{stdgray}\scriptstyle \pm 0.07}$ \\

& MMLU-Pro 
& $5.90 {\color{stdgray}\scriptstyle \pm 0.83}$ 
& $0.67 {\color{stdgray}\scriptstyle \pm 0.04}$ 
& $0.53 {\color{stdgray}\scriptstyle \pm 0.04}$ 
& $0.14 {\color{stdgray}\scriptstyle \pm 0.05}$ \\

\midrule

\multirow{4}{*}{Llama-8B Reasoning}

& GSM8K    
& $\mathbf{4.50 {\color{stdgray}\scriptstyle \pm 0.49}}$ 
& $\mathbf{0.80 {\color{stdgray}\scriptstyle \pm 0.03}}$ 
& $\mathbf{0.57 {\color{stdgray}\scriptstyle \pm 0.02}}$ 
& $\mathbf{0.12 {\color{stdgray}\scriptstyle \pm 0.03}}$ \\

& SVAMP    
& $\mathbf{4.40 {\color{stdgray}\scriptstyle \pm 0.49}}$ 
& $\mathbf{0.83 {\color{stdgray}\scriptstyle \pm 0.03}}$ 
& $\mathbf{0.60 {\color{stdgray}\scriptstyle \pm 0.02}}$ 
& $\mathbf{0.08 {\color{stdgray}\scriptstyle \pm 0.02}}$ \\

& MATH500  
& $\mathbf{4.90 {\color{stdgray}\scriptstyle \pm 0.73}}$ 
& $\mathbf{0.81 {\color{stdgray}\scriptstyle \pm 0.03}}$ 
& $\mathbf{0.60 {\color{stdgray}\scriptstyle \pm 0.03}}$ 
& $\mathbf{0.14 {\color{stdgray}\scriptstyle \pm 0.04}}$ \\

& MMLU-Pro 
& $\mathbf{4.40 {\color{stdgray}\scriptstyle \pm 0.49}}$ 
& $\mathbf{0.85 {\color{stdgray}\scriptstyle \pm 0.03}}$ 
& $\mathbf{0.62 {\color{stdgray}\scriptstyle \pm 0.02}}$ 
& $\mathbf{0.03 {\color{stdgray}\scriptstyle \pm 0.00}}$ \\

\midrule
\midrule

\multirow{4}{*}{Qwen-1.5B Base}

& GSM8K    
& $3.30 {\color{stdgray}\scriptstyle \pm 0.60}$ 
& $0.68 {\color{stdgray}\scriptstyle \pm 0.08}$ 
& $0.43 {\color{stdgray}\scriptstyle \pm 0.05}$ 
& $0.24 {\color{stdgray}\scriptstyle \pm 0.06}$ \\

& SVAMP    
& $2.90 {\color{stdgray}\scriptstyle \pm 0.20}$ 
& $0.73 {\color{stdgray}\scriptstyle \pm 0.03}$ 
& $0.39 {\color{stdgray}\scriptstyle \pm 0.02}$ 
& $0.27 {\color{stdgray}\scriptstyle \pm 0.04}$ \\

& MATH500  
& $2.00 {\color{stdgray}\scriptstyle \pm 0.00}$ 
& $0.88 {\color{stdgray}\scriptstyle \pm 0.01}$ 
& $0.38 {\color{stdgray}\scriptstyle \pm 0.01}$ 
& $0.23 {\color{stdgray}\scriptstyle \pm 0.01}$ \\

& MMLU-Pro 
& $1.60 {\color{stdgray}\scriptstyle \pm 0.40}$ 
& $0.94 {\color{stdgray}\scriptstyle \pm 0.10}$ 
& $0.24 {\color{stdgray}\scriptstyle \pm 0.19}$ 
& $0.12 {\color{stdgray}\scriptstyle \pm 0.19}$ \\

\midrule

\multirow{4}{*}{Qwen-1.5B Reasoning}

& GSM8K    
& $\mathbf{5.10 {\color{stdgray}\scriptstyle \pm 0.62}}$ 
& $\mathbf{0.81 {\color{stdgray}\scriptstyle \pm 0.05}}$ 
& $\mathbf{0.62 {\color{stdgray}\scriptstyle \pm 0.03}}$ 
& $\mathbf{0.07 {\color{stdgray}\scriptstyle \pm 0.01}}$ \\

& SVAMP    
& $\mathbf{4.20 {\color{stdgray}\scriptstyle \pm 0.40}}$ 
& $\mathbf{0.86 {\color{stdgray}\scriptstyle \pm 0.01}}$ 
& $\mathbf{0.62 {\color{stdgray}\scriptstyle \pm 0.02}}$ 
& $\mathbf{0.06 {\color{stdgray}\scriptstyle \pm 0.01}}$ \\

& MATH500  
& $\mathbf{5.00 {\color{stdgray}\scriptstyle \pm 0.57}}$ 
& $\mathbf{0.78 {\color{stdgray}\scriptstyle \pm 0.05}}$ 
& $\mathbf{0.59 {\color{stdgray}\scriptstyle \pm 0.03}}$ 
& $\mathbf{0.09 {\color{stdgray}\scriptstyle \pm 0.02}}$ \\

& MMLU-Pro 
& $\mathbf{7.40 {\color{stdgray}\scriptstyle \pm 1.11}}$ 
& $\mathbf{0.78 {\color{stdgray}\scriptstyle \pm 0.04}}$ 
& $\mathbf{0.66 {\color{stdgray}\scriptstyle \pm 0.02}}$ 
& $\mathbf{0.05 {\color{stdgray}\scriptstyle \pm 0.01}}$ \\

\midrule
\midrule

\multirow{4}{*}{Qwen-14B Base}

& GSM8K    
& $3.10 {\color{stdgray}\scriptstyle \pm 0.76}$ 
& $0.85 {\color{stdgray}\scriptstyle \pm 0.08}$ 
& $0.50 {\color{stdgray}\scriptstyle \pm 0.03}$ 
& $0.05 {\color{stdgray}\scriptstyle \pm 0.04}$ \\

& SVAMP    
& $2.30 {\color{stdgray}\scriptstyle \pm 0.45}$ 
& $0.79 {\color{stdgray}\scriptstyle \pm 0.08}$ 
& $0.34 {\color{stdgray}\scriptstyle \pm 0.07}$ 
& $0.34 {\color{stdgray}\scriptstyle \pm 0.11}$ \\

& MATH500  
& $3.10 {\color{stdgray}\scriptstyle \pm 0.20}$ 
& $0.75 {\color{stdgray}\scriptstyle \pm 0.05}$ 
& $0.43 {\color{stdgray}\scriptstyle \pm 0.05}$ 
& $0.23 {\color{stdgray}\scriptstyle \pm 0.05}$ \\

& MMLU-Pro 
& $4.75 {\color{stdgray}\scriptstyle \pm 1.01}$ 
& $0.80 {\color{stdgray}\scriptstyle \pm 0.06}$ 
& $0.57 {\color{stdgray}\scriptstyle \pm 0.05}$ 
& $0.05 {\color{stdgray}\scriptstyle \pm 0.04}$ \\

\midrule

\multirow{4}{*}{Qwen-14B Reasoning}

& GSM8K    
& $\mathbf{6.50 {\color{stdgray}\scriptstyle \pm 1.02}}$ 
& $\mathbf{0.78 {\color{stdgray}\scriptstyle \pm 0.04}}$ 
& $\mathbf{0.64 {\color{stdgray}\scriptstyle \pm 0.03}}$ 
& $\mathbf{0.06 {\color{stdgray}\scriptstyle \pm 0.02}}$ \\

& SVAMP    
& $\mathbf{6.90 {\color{stdgray}\scriptstyle \pm 0.65}}$ 
& $\mathbf{0.77 {\color{stdgray}\scriptstyle \pm 0.05}}$ 
& $\mathbf{0.65 {\color{stdgray}\scriptstyle \pm 0.02}}$ 
& $\mathbf{0.07 {\color{stdgray}\scriptstyle \pm 0.02}}$ \\

& MATH500  
& $\mathbf{3.00 {\color{stdgray}\scriptstyle \pm 0.00}}$ 
& $\mathbf{0.78 {\color{stdgray}\scriptstyle \pm 0.05}}$ 
& $\mathbf{0.45 {\color{stdgray}\scriptstyle \pm 0.05}}$ 
& $\mathbf{0.23 {\color{stdgray}\scriptstyle \pm 0.06}}$ \\

& MMLU-Pro 
& $\mathbf{6.30 {\color{stdgray}\scriptstyle \pm 0.98}}$ 
& $\mathbf{0.83 {\color{stdgray}\scriptstyle \pm 0.03}}$ 
& $\mathbf{0.67 {\color{stdgray}\scriptstyle \pm 0.01}}$ 
& $\mathbf{0.06 {\color{stdgray}\scriptstyle \pm 0.01}}$ \\

\bottomrule
\end{tabular}

\caption{
Dynamics statistics averaged across analyzed layers and 5 random seeds.
Reasoning-tuned models consistently exhibit richer latent policy structure,
characterized by higher effective state usage ($K_{\mathrm{eff}}$),
larger transition asymmetry (TVD),
and slower-mixing dynamics (lower spectral gap).
}

\label{tab:dataset_dynamics}
\end{table*}

\begin{table*}[t]
\centering
\small
\setlength{\tabcolsep}{3.8pt}
\begin{tabular}{llcccc}
\toprule
Model & Dataset & $K_{\mathrm{eff}}\uparrow$ & $p_{\mathrm{stay}}\uparrow$ & TVD$\uparrow$ & Spec. Gap$\downarrow$ \\
\midrule
\multirow{4}{*}{QwQ-32B May-base} & GSM8K & $5.62 {\color{stdgray}\scriptstyle \pm 0.86}$ & $0.71 {\color{stdgray}\scriptstyle \pm 0.08}$ & $0.54 {\color{stdgray}\scriptstyle \pm 0.10}$ & $0.14 {\color{stdgray}\scriptstyle \pm 0.02}$ \\
 & MATH500 & $3.50 {\color{stdgray}\scriptstyle \pm 0.50}$ & $0.78 {\color{stdgray}\scriptstyle \pm 0.09}$ & $0.50 {\color{stdgray}\scriptstyle \pm 0.05}$ & $0.17 {\color{stdgray}\scriptstyle \pm 0.05}$ \\
 & MMLU-Pro & $6.25 {\color{stdgray}\scriptstyle \pm 0.83}$ & $0.74 {\color{stdgray}\scriptstyle \pm 0.05}$ & $0.60 {\color{stdgray}\scriptstyle \pm 0.05}$ & $0.10 {\color{stdgray}\scriptstyle \pm 0.07}$ \\
 & SVAMP & $5.62 {\color{stdgray}\scriptstyle \pm 0.48}$ & $0.72 {\color{stdgray}\scriptstyle \pm 0.08}$ & $0.55 {\color{stdgray}\scriptstyle \pm 0.09}$ & $0.14 {\color{stdgray}\scriptstyle \pm 0.06}$ \\
\midrule
\multirow{4}{*}{QwQ-32B May-think} & GSM8K & $6.00 {\color{stdgray}\scriptstyle \pm 1.00}$ & $0.75 {\color{stdgray}\scriptstyle \pm 0.04}$ & $0.59 {\color{stdgray}\scriptstyle \pm 0.05}$ & $0.10 {\color{stdgray}\scriptstyle \pm 0.02}$ \\
 & MATH500 & $4.25 {\color{stdgray}\scriptstyle \pm 0.66}$ & $0.80 {\color{stdgray}\scriptstyle \pm 0.02}$ & $0.56 {\color{stdgray}\scriptstyle \pm 0.04}$ & $0.14 {\color{stdgray}\scriptstyle \pm 0.04}$ \\
 & MMLU-Pro & $6.38 {\color{stdgray}\scriptstyle \pm 0.48}$ & $0.80 {\color{stdgray}\scriptstyle \pm 0.05}$ & $0.65 {\color{stdgray}\scriptstyle \pm 0.04}$ & $0.09 {\color{stdgray}\scriptstyle \pm 0.05}$ \\
 & SVAMP & $5.75 {\color{stdgray}\scriptstyle \pm 0.83}$ & $0.79 {\color{stdgray}\scriptstyle \pm 0.04}$ & $0.62 {\color{stdgray}\scriptstyle \pm 0.04}$ & $0.09 {\color{stdgray}\scriptstyle \pm 0.02}$ \\
\bottomrule
\end{tabular}
\caption{QwQ-32B SDS metrics using BIC-selected $K$ for each seed/layer run. Values are mean $\pm$ std across available seed-layer runs.}
\label{tab:qwq32_sds_bic_complete}
\end{table*}




\end{document}